\documentclass[twoside,11pt]{article}

\usepackage{jmlr2e}
\usepackage{caption}
\usepackage{subcaption}
\usepackage{color}
\usepackage{amsmath,amsfonts,mathrsfs}
\usepackage{algorithm, algorithmic}

\DeclareGraphicsExtensions{.eps,.pdf}
\graphicspath{{figs/}}

\newcommand{\der}[1]{\dot #1}
\newcommand{\dder}[1]{\ddot #1}
\newcommand{\e}[1]{\mathrm{e}^{#1}}
\newcommand{\norm}[1]{\|#1\|}
\newcommand{\R}{\mathbb{R}}
\newcommand{\inverse}[1]{\frac{1}{#1}}
\newcommand{\goto}{\rightarrow}
\newcommand{\xsr}{X^{\mathrm{sr}}}
\newcommand{\tsr}{t_{\mathrm{sr}}}
\renewcommand{\d}{{\mathrm{d}}}
\def\half{\frac{1}{2}}
\def\N{Nesterov's scheme }
\def\D{\Delta}
\def\iid{\textrm{i.i.d.~}}

\definecolor{wjs}{RGB}{250,0,0}

\jmlrheading{}{}{}{}{}{Weijie Su, Stephen Boyd and Emmanuel J.~Cand\`es}

\ShortHeadings{An ODE for Modeling Nesterov's Scheme}{Su, Boyd and Cand\`es}
\firstpageno{1}

\begin{document}

\title{A Differential Equation for Modeling Nesterov's Accelerated Gradient Method: Theory and Insights}

\author{\name Weijie Su \email wjsu@stanford.edu \\
       \addr Department of Statistics\\
       Stanford University\\
       Stanford, CA 94305, USA
       \AND
       \name Stephen Boyd \email boyd@stanford.edu \\
       \addr Department of Electrical Engineering\\
        Stanford University\\
        Stanford, CA 94305, USA 
        \AND
        \name Emmanuel J.~Cand\`es \email candes@stanford.edu\\
        \addr Departments of Statistics and Mathematics\\
       Stanford University\\
       Stanford, CA 94305, USA
}
\editor{}
\maketitle

\begin{abstract}
We derive a second-order ordinary differential equation (ODE) which is the limit of Nesterov's accelerated gradient method. This ODE exhibits approximate equivalence to Nesterov's scheme and thus can serve as a tool for analysis. We show that the continuous time ODE allows for a better understanding of Nesterov's scheme. As a byproduct, we obtain a family of schemes with similar convergence rates. The ODE interpretation also suggests restarting Nesterov's scheme leading to an algorithm, which can be rigorously proven to converge at a linear rate whenever the objective is strongly convex.
\end{abstract}


\section{Introduction}
In many fields of machine learning, minimizing a convex function is at the core of efficient model estimation. In the simplest and most standard form,  we are interested in solving 
\[\mbox{minimize}\quad f(x),\] where $f$ is a convex function, smooth or non-smooth, and $x \in \R^n$ is the variable. Since Newton, numerous algorithms and methods have been proposed to solve the minimization problem, notably gradient and subgradient descent, Newton's methods, trust region methods, conjugate gradient methods, and interior point methods \citep[see e.g.][for expositions]{polyak1987, boydconvex,nocedal2006numerical,ruszczynski2006,ADMM,shor2012,beck2014introduction}.

First-order methods have regained popularity as data sets and problems
are ever increasing in size and, consequently, there has been much
research on the theory and practice of accelerated first-order
schemes.  Perhaps the earliest first-order method for minimizing a
convex function $f$ is the gradient method, which dates back to Euler
and Lagrange. Thirty years ago, however, in a seminal paper Nesterov
proposed an accelerated gradient method \citep{nesterov}, which may
take the following form: starting with $x_0$ and $y_0 = x_0$,
inductively define
\begin{equation}\label{eq:nesterov-scheme}
\begin{aligned}  
& x_k = y_{k-1} - s\nabla f(y_{k-1})\\
& y_k = x_k + \frac{k-1}{k+2} (x_k - x_{k-1}).
\end{aligned}
\end{equation}
For any fixed step size $s \le 1/L$, where $L$ is the Lipschitz constant of $\nabla f$, this scheme exhibits the convergence rate
\begin{equation}\label{eq:sk_quad_err}
f(x_k) - f^\star \le O \left(\frac{\|x_0 - x^\star\|^2}{sk^2} \right).
\end{equation}
Above, $x^\star$ is any minimizer of $f$ and $f^\star = f(x^\star)$. It is well-known that this rate is optimal among all methods having only information about the gradient of $f$ at consecutive iterates~\citep{nesterov-book}. This is in contrast to vanilla gradient descent methods, which have the same computational complexity but can only achieve a rate of $O(1/k)$. This improvement relies on the introduction of the momentum term $x_k-x_{k-1}$ as well as the particularly tuned coefficient $(k-1)/(k+2)\approx 1 - 3/k$. Since the introduction of Nesterov's scheme, there has been much work on the development of first-order accelerated methods, see \citet{nesterov-book, nesterovsmooth, nesterov_compo} for theoretical developments, and \citet{tseng} for a unified analysis of these ideas. Notable applications can be found in sparse linear regression \citep{BeckTeboulle, qin2012}, compressed sensing \citep{becker} and, deep and recurrent neural networks \citep{sutskever2013}.

In a different direction, there is a long history relating ordinary differential equation (ODEs) to optimization, see \citet{dynamic}, \citet{schropp}, and \citet{fiori2005} for example. The connection between ODEs and numerical optimization is often established via taking step sizes to be very small so that the trajectory or solution path converges to a curve modeled by an ODE. The conciseness and well-established theory of ODEs provide deeper insights into optimization, which has led to many interesting findings. Notable examples include linear regression via solving differential equations induced by linearized Bregman iteration algorithm \citep{osher2014sparse}, a continuous-time Nesterov-like algorithm in the context of control design \citep{durr2012a, durr2012b}, and modeling design iterative optimization algorithms as nonlinear dynamical systems \citep{recht}.

In this work, we derive a second-order ODE which is the exact limit of
Nesterov's scheme by taking small step sizes in
\eqref{eq:nesterov-scheme}; to the best of our knowledge, this work is
the first to use ODEs to model Nesterov's scheme or its variants in
this limit.  One surprising fact in connection with this subject is
that a \textit{first-order} scheme is modeled by a
\textit{second-order} ODE.  This ODE takes the following form:
\begin{equation}\label{key}
\ddot{X} + \frac{3}{t}\dot{X} + \nabla f(X) = 0
\end{equation}
for $t > 0$, with initial conditions $X(0) = x_0, \der{X}(0) = 0$; here, $x_0$ is the starting point in Nesterov's scheme, $\dot X \equiv \d X/\d t$ denotes the time derivative or velocity and similarly $\ddot X \equiv \d^2 X/\d t^2$ denotes the acceleration. The time parameter in this ODE is related to the step size in \eqref{eq:nesterov-scheme} via $t \approx k\sqrt{s}$. Expectedly, it also enjoys inverse quadratic convergence rate as its discrete analog,
\[
f(X(t)) - f^\star \le O\left( \frac{\|x_0 - x^\star\|^2}{t^2}\right).
\]
Approximate equivalence between Nesterov's scheme and the ODE is established later in various perspectives, rigorous and intuitive. In the main body of this paper, examples and case studies are provided to demonstrate that the homogeneous and conceptually simpler ODE can serve as a tool for understanding, analyzing and generalizing Nesterov's scheme.

In the following, two insights of Nesterov's scheme are highlighted, the first one on oscillations in the trajectories of this scheme, and the second on the peculiar constant 3 appearing in the ODE.

\subsection{From Overdamping to Underdamping}
\label{sec:from-overd-underd}
In general, Nesterov's scheme is not monotone in the objective function value due to the introduction of the momentum term. Oscillations or overshoots along the trajectory of iterates approaching the minimizer are often observed when running Nesterov's scheme. Figure \ref{fig:intro_scaling} presents typical phenomena of this kind, where a two-dimensional convex function is minimized by Nesterov's scheme. Viewing the ODE as a damping system, we obtain interpretations as follows.
\par
{\noindent \bf Small $t$.} In the beginning, the damping ratio $3/t$ is large. This leads the ODE to be an overdamped system, returning to the equilibrium without oscillating;\\
{\noindent \bf Large $t$.} As $t$ increases, the ODE with a small $3/t$ behaves like an underdamped system, oscillating with the amplitude gradually decreasing to zero.
\par
As depicted in Figure \ref{fig:overshoot1}, in the beginning the ODE
curve moves smoothly towards the origin, the minimizer $x^\star$. The
second interpretation ``{\bf Large $t$'}' provides partial explanation for the oscillations
observed in Nesterov's scheme at later stage. Although our analysis
extends farther, it is similar in spirit to that carried in
\citet{restart}. In particular, the zoomed Figure \ref{fig:overshoot2}
presents some butterfly-like oscillations for both the scheme and
ODE. There, we see that the trajectory constantly moves away from the origin and returns back later. Each overshoot in Figure \ref{fig:overshoot2} causes a bump in the function values, as shown in Figure \ref{fig:bump_w_text}. We observe also
from Figure \ref{fig:bump_w_text} that the periodicity captured by the
bumps are very close to that of the ODE solution. In passing, it is
worth mentioning that the solution to the ODE in this case can be
expressed via Bessel functions, hence enabling quantitative
characterizations of these overshoots and bumps, which are given in
full detail in Section \ref{sec:interpret}.

\begin{figure}[!htp]
\centering
\begin{subfigure}[b]{0.32\textwidth}
\centering
\includegraphics[width = \textwidth]{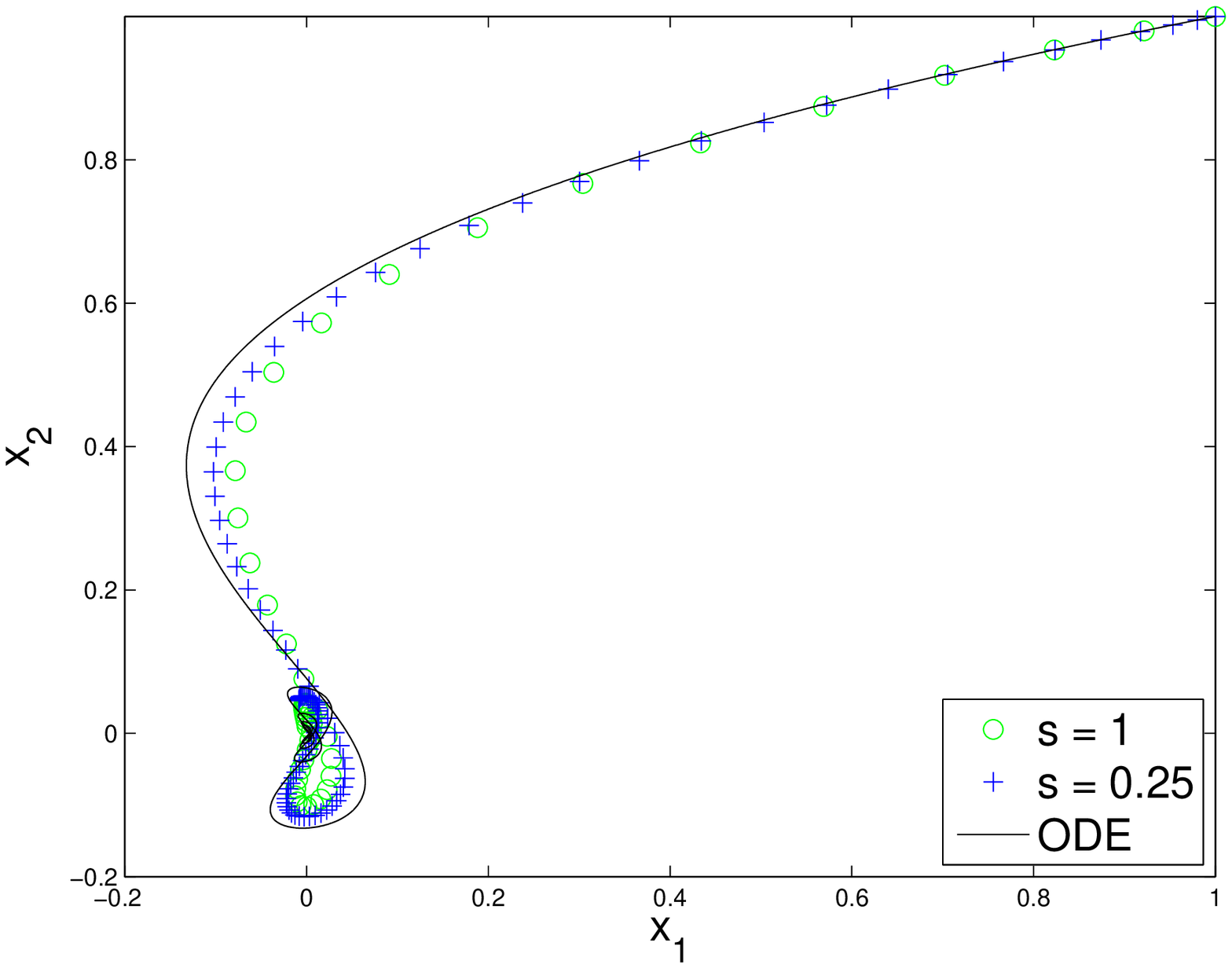}
\caption{Trajectories.}
\label{fig:overshoot1}
\end{subfigure}
\hfill
\begin{subfigure}[b]{0.32\textwidth}
\centering
\includegraphics[width = \textwidth]{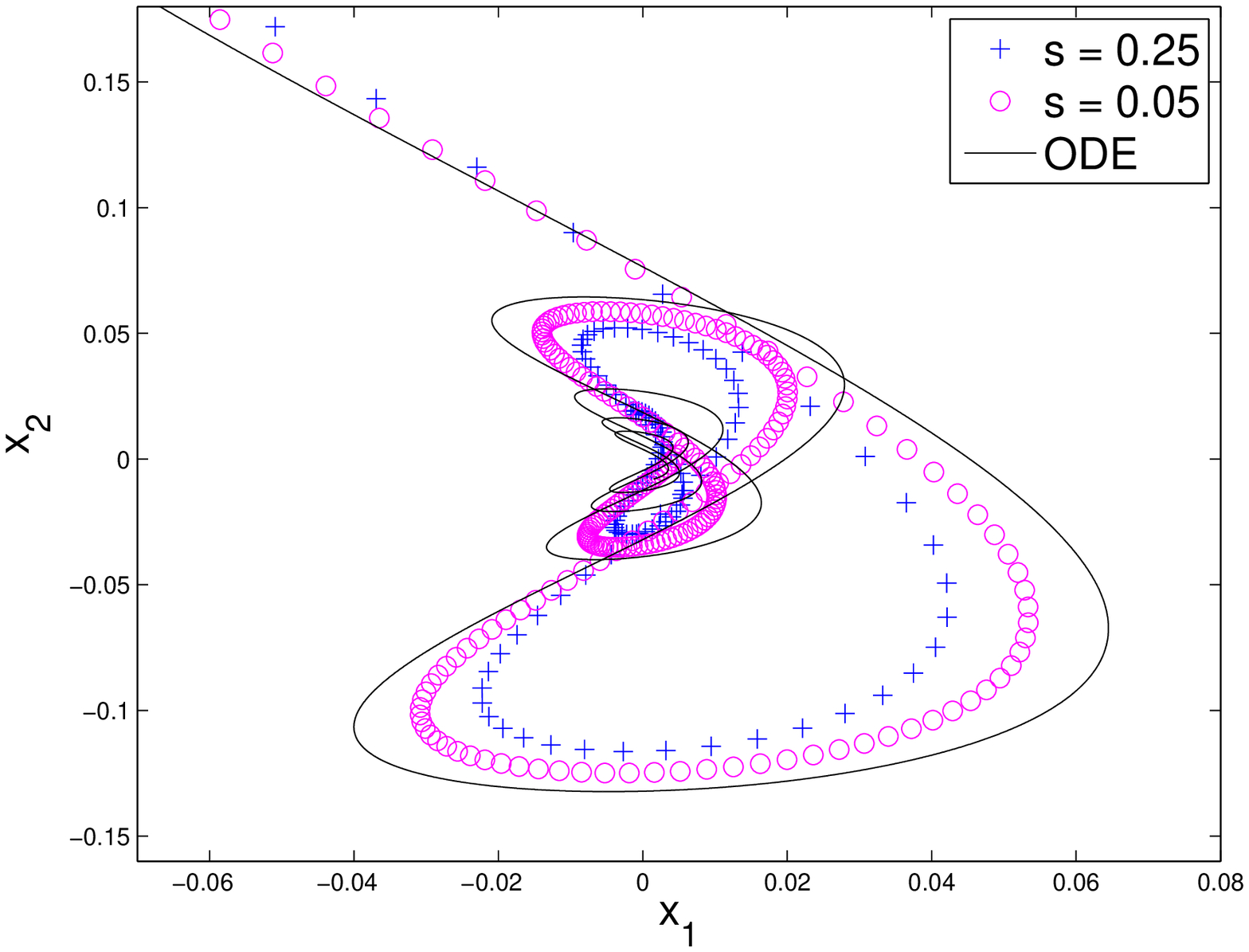}
\caption{Zoomed trajectories.}
\label{fig:overshoot2}
\end{subfigure}
\hfill
\begin{subfigure}[b]{0.32\textwidth}
\centering
\includegraphics[width = \textwidth]{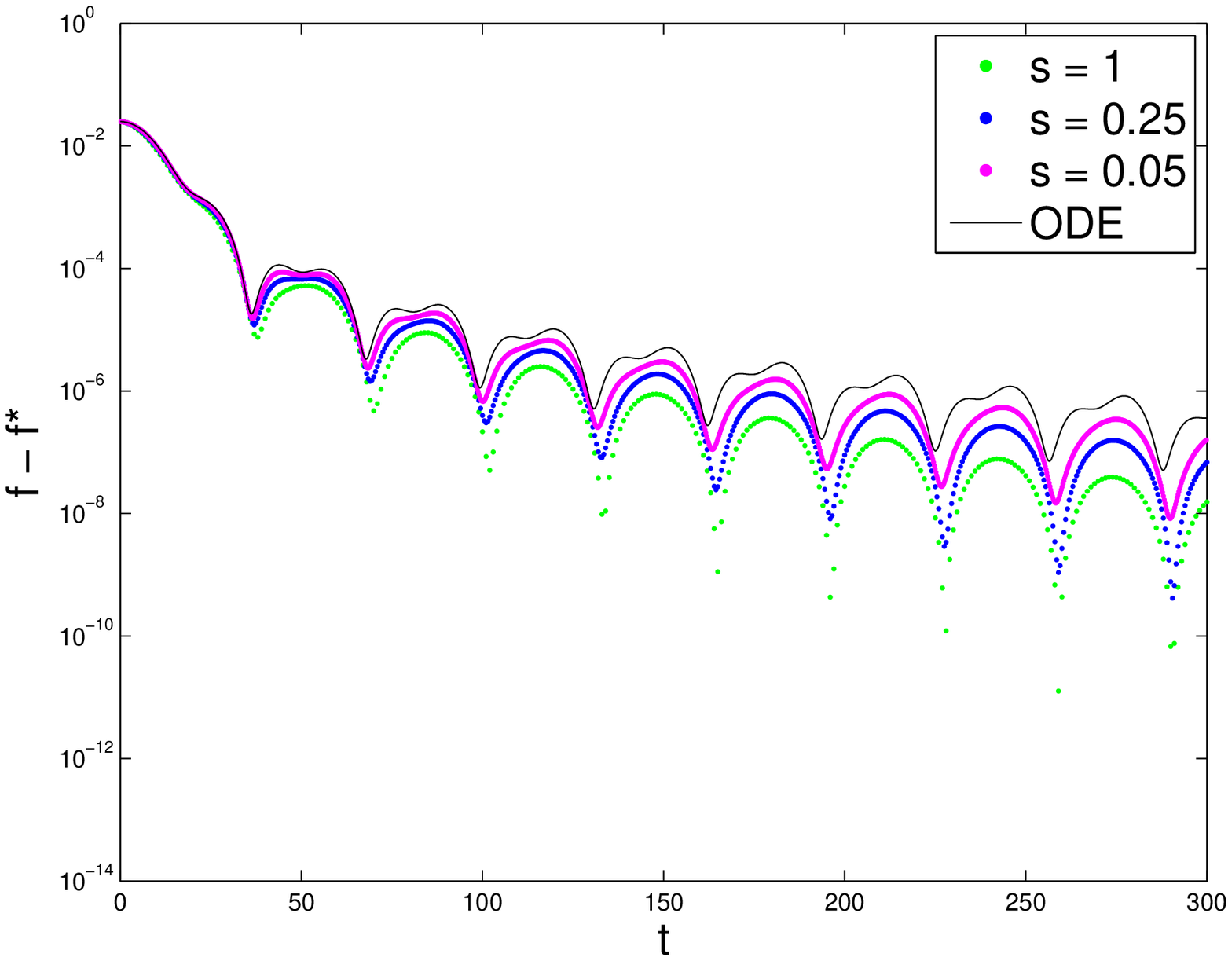}
\caption{Errors $f - f^\star$.}
\label{fig:bump_w_text}
\end{subfigure}
\caption{Minimizing $f = 2\times 10^{-2} x_1^2 + 5\times 10^{-3} x_2^2$, starting from $x_0 = (1, ~ 1)$. The black and solid curves correspond to the solution to the ODE. In (c), for the x-axis we use the identification between time and iterations, $t = k\sqrt{s}$.}
\label{fig:intro_scaling}
\end{figure}

\subsection{A Phase Transition}
\label{sec:phase-transition}
The constant 3, derived from $(k+2) - (k-1)$ in \eqref{key}, is not haphazard. In fact, it is the smallest constant that guarantees $O(1/t^2)$ convergence rate. Specifically, parameterized by a constant $r$, the generalized ODE
\[
\ddot X + \frac{r}{t}\dot X + \nabla f(X) = 0
\]
can be translated into a generalized Nesterov's scheme that is the same as the original \eqref{eq:nesterov-scheme} except for $(k-1)/(k+2)$ being replaced by $(k-1)/(k+r-1)$. Surprisingly, for both generalized ODEs and schemes, the inverse quadratic convergence is guaranteed if and only if $r \ge 3$. This phase transition suggests there might be deep causes for acceleration among first-order methods. In particular, for $r \ge 3$, the worst case constant in this inverse quadratic convergence rate is minimized at $r = 3$. 

Figure \ref{fig:small_r_fail} illustrates the growth of $t^2(f(X(t)) - f^\star)$ and $sk^2(f(x_k) - f^\star)$, respectively, for the generalized ODE and scheme with $r = 1$, where the objective function is simply $f(x) = \frac12 x^2$. Inverse quadratic convergence fails to be observed in both Figures \ref{fig:ode_abs_x} and \ref{fig:scheme_abs_x}, where the scaled errors grow with $t$ or iterations, for both the generalized ODE and scheme.

\begin{figure}[!htp]
\centering
\begin{subfigure}[b]{0.45\textwidth}
\centering
\includegraphics[width = \textwidth, height=1.6in]{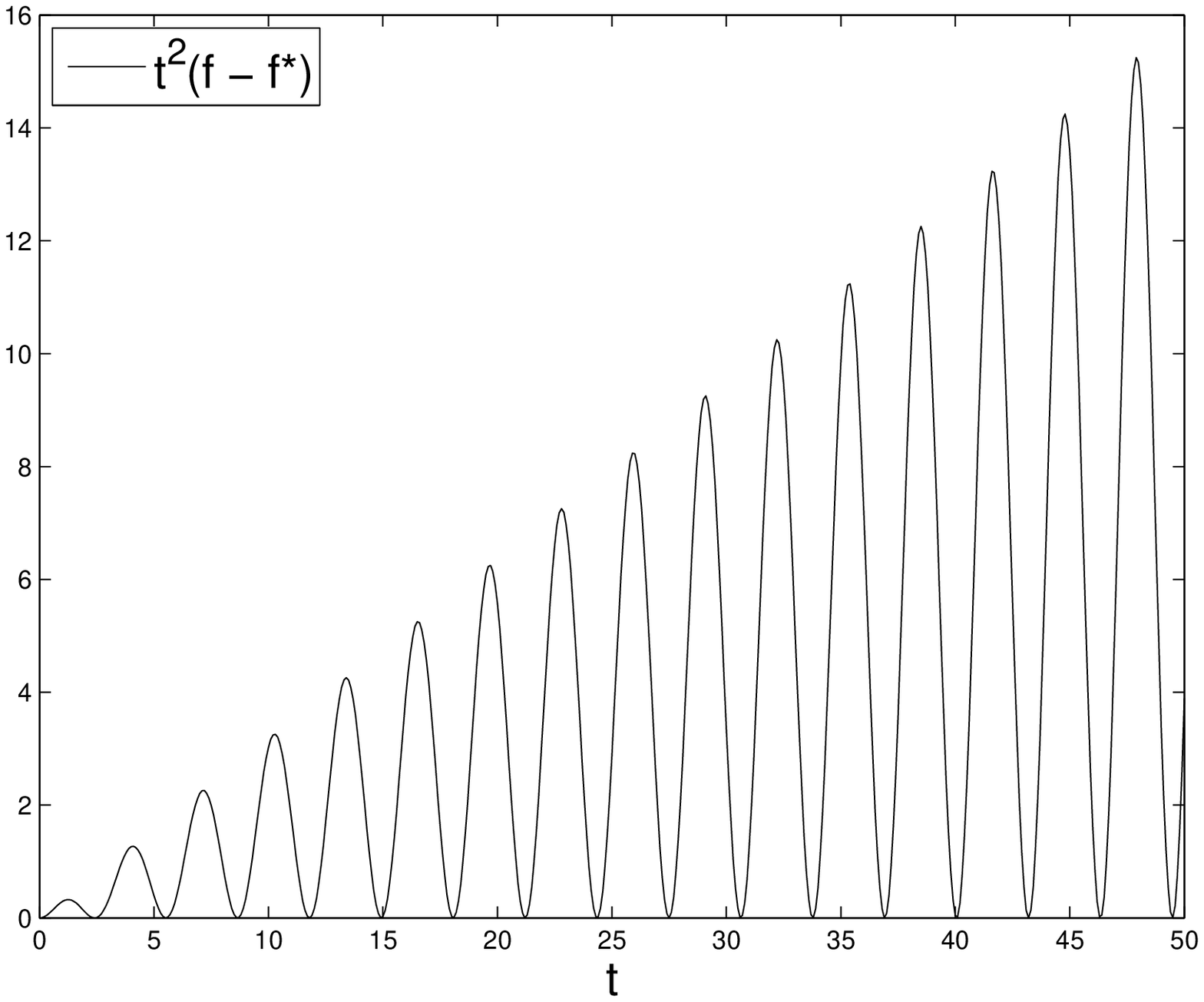}
\caption{Scaled errors $t^2(f(X(t)) - f^\star)$.}
\label{fig:ode_abs_x}
\end{subfigure}
\hfill
\begin{subfigure}[b]{0.45\textwidth}
\centering
\includegraphics[width = \textwidth, height=1.6in]{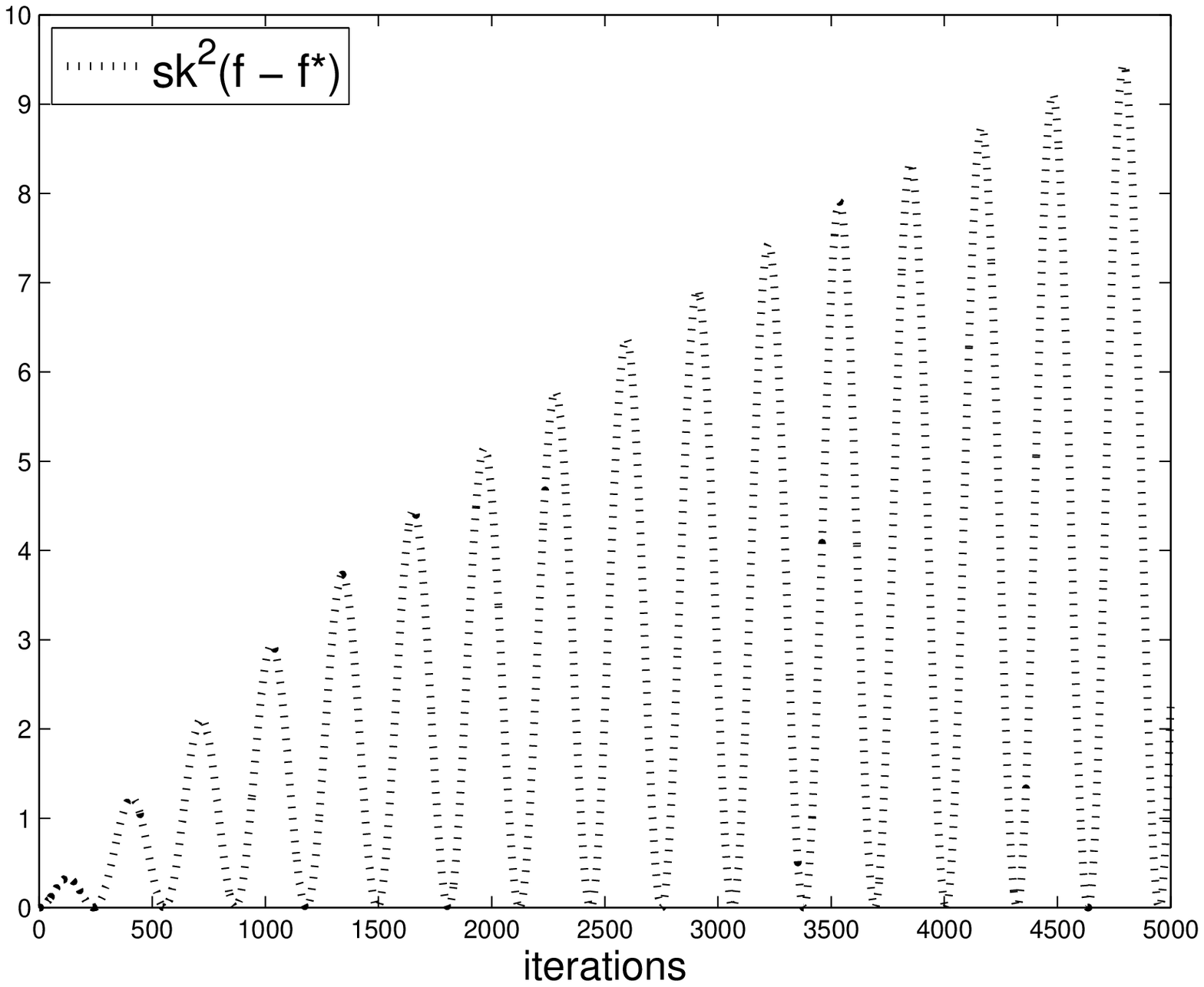}
\caption{Scaled errors $sk^2(f(x_k) - f^\star)$.}
\label{fig:scheme_abs_x}
\end{subfigure}
\caption{Minimizing $f = \frac12 x^2$ by the generalized ODE and scheme with $r = 1$, starting from $x_0 = 1$. In (b), the step size $s = 10^{-4}$. }
\label{fig:small_r_fail}
\end{figure}

\subsection{Outline and Notation}
\label{sec:outline-highlights}
The rest of the paper is organized as follows. In Section \ref{sec:derivation}, the ODE is rigorously derived from Nesterov's scheme, and a generalization to composite optimization, where $f$ may be non-smooth, is also obtained. Connections between the ODE and the scheme, in terms of trajectory behaviors and convergence rates, are summarized in Section \ref{sec:interpret}. In Section \ref{sec:varycoeff}, we discuss the effect of replacing the constant $3$ in \eqref{key} by an arbitrary constant on the convergence rate. A new restarting scheme is suggested in Section \ref{sec:accelerate}, with linear convergence rate established and empirically observed.

Some standard notations used throughout the paper are collected here. We denote by $\mathcal{F}_L$ the class of convex functions $f$ with $L$--Lipschitz continuous gradients defined on $\R^n$, i.e., $f$ is convex, continuously differentiable, and satisfies
\[
\norm{\nabla f(x) - \nabla f(y)} \le L\norm{x - y}
\]
for any $x, y\in \R^n$, where $\norm{\cdot}$ is the standard Euclidean norm and $L > 0$ is the Lipschitz constant. Next, $\mathcal{S}_{\mu}$ denotes the class of $\mu$--strongly convex functions $f$ on $\R^n$ with continuous gradients, i.e., $f$ is continuously differentiable and $f(x) - \mu\|x\|^2/2$ is convex. We set $\mathcal{S}_{\mu,L} = \mathcal{F}_{L}\cap \mathcal{S}_{\mu}$. 

\section{Derivation}
\label{sec:derivation}
First, we sketch an informal derivation of the ODE \eqref{key}. Assume
$f\in \mathcal{F}_L$ for $L > 0$. Combining the two equations of
\eqref{eq:nesterov-scheme} and applying a rescaling gives
\begin{equation}\label{eq:nesterovscheme_one}
\frac{x_{k+1} -x_k}{\sqrt{s}} =  \frac{k-1}{k+2}\frac{x_k -x_{k-1}}{\sqrt{s}} -  \sqrt{s}\nabla f(y_k).
\end{equation}
Introduce the \textit{Ansatz} $x_k \approx X(k\sqrt{s})$ for some
smooth curve $X(t)$ defined for $t \ge 0$.  Put $k = t/\sqrt{s}$. Then
as the step size $s$ goes to zero, $X(t) \approx x_{t/\sqrt{s}} = x_k$
and $X(t+\sqrt{s} ) \approx x_{(t+\sqrt{s})/\sqrt{s}} = x_{k+1}$, and
Taylor expansion gives
\[
(x_{k+1} -x_k)/ \sqrt{s}  = \dot{X}(t) + \frac{1}{2}\ddot{X}(t)\sqrt{s} + o(\sqrt{s}), \quad (x_{k} -x_{k-1})/ \sqrt{s}  = \dot{X}(t) - \frac{1}{2}\ddot{X}(t)\sqrt{s} + o(\sqrt{s})
\]
and $\sqrt{s}\nabla f(y_k) = \sqrt{s}\nabla f(X(t)) +
o(\sqrt{s})$. Thus \eqref{eq:nesterovscheme_one} can be written as
\begin{multline}\label{eq:nesterov_taylor}
\dot{X}(t) + \frac{1}{2}\ddot{X}(t)\sqrt{s} + o(\sqrt{s})\\
 = \Big{(}1 - \frac{3\sqrt{s}}{t} \Big{)}  \Big{(}\dot{X}(t) - \frac{1}{2}\ddot{X}(t)\sqrt{s} + o(\sqrt{s}) \Big{)} - \sqrt{s}\nabla f(X(t))  + o(\sqrt{s}).
\end{multline}
By comparing the coefficients of $\sqrt{s}$ in \eqref{eq:nesterov_taylor}, we obtain
\begin{equation}\nonumber
\ddot{X} + \frac{3}{t}\dot{X} + \nabla f(X) = 0.
\end{equation}
The first initial condition is $X(0) = x_0$. Taking $k=1$ in \eqref{eq:nesterovscheme_one} yields 
\[
(x_2-x_1)/\sqrt{s} = -\sqrt{s}\nabla f(y_1) = o(1).
\]
Hence, the second initial condition is simply $\der{X(0)} = 0$ (vanishing initial velocity).

One popular alternative momentum coefficient is $\theta_k(\theta_{k-1}^{-1} - 1)$, where $\theta_k$ are iteratively defined as $\theta_{k+1} = \left(\sqrt{\theta_k^4 + 4\theta_k^2} - \theta_k^2 \right)/2$, starting from $\theta_0 = 1$ \citep{nesterov, BeckTeboulle}. Simple analysis reveals that $\theta_k(\theta_{k-1}^{-1} - 1)$ asymptotically equals $1 - 3/k + O(1/k^2)$, thus leading to the same ODE as \eqref{eq:nesterov-scheme}.

Classical results in ODE theory do not directly imply the existence or
uniqueness of the solution to this ODE because the coefficient $3/t$
is singular at $t=0$. In addition, $\nabla f$ is typically not
analytic at $x_0$, which leads to the inapplicability of the power
series method for studying singular ODEs. Nevertheless, the ODE is
well posed: the strategy we employ for showing this constructs a
series of ODEs approximating \eqref{key}, and then chooses a convergent
subsequence by some compactness arguments such as the Arzel\'a-Ascoli
theorem. Below, $C^{2}((0,\infty);\R^n)$ denotes the class of twice continuously differentiable maps from $(0, \infty)$ to $\R^n$; similarly, $C^{1}([0,\infty);\R^n)$ denotes the class of continuously differentiable maps from $[0, \infty)$ to $\R^n$.
\begin{theorem}\label{thm:regularity}
For any $f\in\mathcal{F}_{\infty} :=  \cup_{L > 0}\mathcal{F}_L$ and any $x_0\in\R^n$, the ODE \eqref{key} with initial conditions $X(0) = x_0, \der{X}(0) = 0$ has a \textit{unique global} solution $X\in C^{2}((0,\infty);\R^n)\cap C^{1}([0,\infty);\R^n)$.
\end{theorem}
The next theorem, in a rigorous way, guarantees the validity of the derivation of this ODE. The proofs of both theorems are deferred to the appendices.

\begin{theorem}\label{thm:nesterov_limit}
For any $f \in \mathcal{F}_{\infty}$, as the step size $s \goto 0$, Nesterov's scheme \eqref{eq:nesterov-scheme} converges to the ODE \eqref{key} in the sense that for all fixed $T > 0$,
\[
\lim_{s\goto 0}\max_{0 \le k \le \frac{T}{\sqrt{s}}}\left\| x_k - X\left( k\sqrt{s} \right) \right\| = 0.
\] 

\end{theorem}

\subsection{Simple Properties}
We collect some elementary properties that are helpful in understanding the ODE.
\\
\noindent{\bf{Time Invariance.}}~ If we adopt a linear time transformation, $\tilde t = ct$ for some $c > 0$, by the chain rule it follows that
\begin{equation*}
\frac{\d X}{\d \tilde t} = \frac{1}{c}\frac{\d X}{\d t}, ~ \frac{\d^2 X}{\d \tilde t^2} = \frac{1}{c^2}\frac{\d^2 X}{\d t^2}.
\end{equation*}
This yields the ODE parameterized by $\tilde t$,
\begin{equation}\nonumber
\frac{\d^2 X}{\d \tilde t^2} + \frac{3}{\tilde t}\frac{\d X}{\d \tilde t} + \nabla f(X)/c^2 = 0.
\end{equation}
Also note that minimizing $f/c^2$ is equivalent to minimizing $f$. Hence, the ODE is invariant under the time change. In fact, it is easy to see that time invariance holds if and only if the coefficient of $\dot X$ has the form $C/t$ for some constant $C$.
\\
\noindent{\bf{Rotational Invariance.}}~ Nesterov's scheme and other gradient-based schemes are invariant under rotations. As expected, the ODE is also invariant under orthogonal transformation. To see this, let $Y = QX$ for some orthogonal matrix $Q$. This leads to $\dot Y = Q\dot X, \ddot Y = Q\ddot X$ and $\nabla_Y f = Q\nabla_X f$. Hence, denoting by $Q^T$ the transpose of $Q$, the ODE in the new coordinate system reads $Q^T\ddot Y + \frac{3}{t}Q^T\dot Y + Q^T\nabla_Yf = 0$, which is of the same form as \eqref{key} once multiplying $Q$ on both sides.
\\
\noindent{\bf{Initial Asymptotic.}}~ Assume sufficient smoothness of $X$ such that $\lim_{t\goto 0} \ddot X(t)$ exists. The mean value theorem guarantees the existence of some $\xi\in(0, t)$ that satisfies $\dot X(t)/t = (\dot X(t) - \dot X(0))/t = \ddot X(\xi)$. Hence, from the ODE we deduce $\ddot X(t) + 3\dder{X}(\xi) + \nabla f(X(t)) = 0$. Taking the limit $t\goto 0$ gives $\ddot X(0) = -\nabla f(x_0)/4$. Hence, for small $t$ we have the asymptotic form:
\begin{equation}\nonumber
X(t) =  - \frac{\nabla f(x_0)t^2}{8} + x_0 + o(t^2).
\end{equation}
This asymptotic expansion is consistent with the empirical observation that Nesterov's scheme moves slowly in the beginning.

\subsection{ODE for Composite Optimization}
\label{sec:extension-non-smooth}
It is interesting and important to generalize the ODE to minimizing $f$ in the composite form $f(x) = g(x) + h(x)$, where the smooth part $g\in \mathcal{F}_{L}$ and the non-smooth part $h:\R^n\goto(-\infty,\infty]$ is a structured general convex function. Both \cite{nesterov_compo} and \cite{BeckTeboulle} obtain $O(1/k^2)$ convergence rate by employing the proximal structure of $h$. In analogy to the smooth case, an ODE for composite $f$ is derived in the appendix.

\section{Connections and Interpretations}
\label{sec:interpret}
In this section, we explore the approximate equivalence between the ODE and Nesterov's scheme, and provide evidence that the ODE can serve as an amenable tool for interpreting and analyzing Nesterov's scheme. The first subsection exhibits inverse quadratic convergence rate for the ODE solution, the next two address the oscillation phenomenon discussed in Section \ref{sec:from-overd-underd}, and the last subsection is devoted to comparing Nesterov's scheme with gradient descent from a numerical perspective.

\subsection{Analogous Convergence Rate}
\label{sec:complex}
The original result from \citet{nesterov} states that, for any $f\in\mathcal{F}_L$, the sequence $\{x_k\}$ given by \eqref{eq:nesterov-scheme} with step size $s \le 1/L$ satisfies
\begin{equation}\label{eq:nesterov_ineq}
f(x_k) - f^\star \le \frac{2\|x_0 - x^\star\|^2}{s(k+1)^2}.
\end{equation}
Our next result indicates that the trajectory of \eqref{key} closely resembles the sequence $\{x_k\}$ in terms of the convergence rate to a minimizer $x^\star$. Compared with the discrete case, this proof is shorter and simpler.
\begin{theorem}\label{thm:ode_rate}
For any $f\in\mathcal{F}_{\infty}$, let $X(t)$ be the unique global solution to \eqref{key} with initial conditions $X(0) = x_0, \dot X(0) = 0$. Then, for any $t > 0$,
\begin{equation}\label{eq:t_2_rate}
f(X(t)) - f^\star \le \frac{2\norm{x_0-x^\star}^2}{t^2}.
\end{equation}
\end{theorem}

\begin{proof}
Consider the energy functional\footnote{We may also view this functional as the negative entropy. Similarly, for the gradient flow $\dot X + \nabla f(X) = 0$, an energy function of form $\mathcal{E}_{\mathrm{gradient}}(t) = t(f(X(t))-f^\star) + \|X(t) - x^\star\|^2/2$ can be used to derive the bound $f(X(t)) - f^\star \le \frac{\|x_0-x^\star\|^2}{2t}$.} defined as $\mathcal{E}(t) = t^2(f(X(t)) - f^\star) + 2\|X + t\dot{X}/2 - x^\star \|^2$, whose time derivative is 
\[
\dot{\mathcal{E}} = 2t(f(X) - f^\star) + t^2\langle \nabla f, \dot{X} \rangle + 4\left\langle X + \frac{t}{2}\dot{X} - x^\star, \frac{3}{2}\dot{X} + \frac{t}{2}\ddot{X} \right\rangle.
\]
Substituting $3\dot{X}/2 + t\ddot{X}/2$ with $-t\nabla f(X)/2$, the above equation gives
\[
\dot{\mathcal{E}} = 2t(f(X) - f^\star) + 4\langle X - x^\star, - t\nabla f(X)/2 \rangle  = 2t(f(X) - f^\star) -2t\langle X - x^\star, \nabla f(X) \rangle \le 0,
\]
where the inequality follows from the convexity of $f$. Hence by monotonicity of $\mathcal{E}$ and non-negativity of $2\norm{X + t\dot{X}/2 - x^\star}^2$, the gap satisfies
\begin{equation*}
f(X(t)) - f^\star \le \frac{\mathcal{E}(t)}{t^2} \le \frac{\mathcal{E}(0)}{t^2} = \frac{2\norm{x_0 - x^\star}^2}{t^2}.
\end{equation*}
\end{proof}
Making use of the approximation $t \approx k\sqrt{s}$, we observe that the convergence rate in \eqref{eq:nesterov_ineq} is essentially a discrete version of that in \eqref{eq:t_2_rate}, providing yet another piece of evidence for the approximate equivalence between the ODE and the scheme.

We finish this subsection by showing that the number 2 appearing in the numerator of the error bound in \eqref{eq:t_2_rate} is optimal. Consider an arbitrary $f \in \mathcal{F}_{\infty}(\R)$ such that $f(x) = x$ for $x \ge 0$. Starting from some $x_0 > 0$, the solution to \eqref{key} is $X(t) = x_0 - t^2/8$ before hitting the origin. Hence, $t^2(f(X(t)) - f^\star) = t^2(x_0 - t^2/8)$ has a maximum $2x_0^2 = 2|x_0 - 0|^2$ achieved at $t = 2\sqrt{x_0}$. Therefore, we cannot replace 2 by any smaller number, and we can expect that this tightness also applies to the discrete analog \eqref{eq:nesterov_ineq}.

\subsection{Quadratic $f$ and Bessel Functions}
\label{sec:bessel}
For quadratic $f$, the ODE \eqref{key} admits a solution in closed form. This closed form solution turns out to be very useful in understanding the issues raised in the introduction. 

Let $f(x) = \frac12\langle x, Ax \rangle + \langle b, x \rangle$, where $A \in \R^{n \times n}$ is a positive semidefinite matrix and $b$ is in the column space of $A$ because otherwise this function can attain $-\infty$. Then a simple translation in $x$ can absorb the linear term $\langle b, x\rangle$ into the quadratic term. Since both the ODE and the scheme move within the affine space perpendicular to the kernel of $A$, without loss of generality, we assume that $A$ is positive definite, admitting a spectral decomposition $A = Q^T\Lambda Q$, where $\Lambda$ is a diagonal matrix formed by the eigenvalues. Replacing $x$ with $Qx$, we assume $f = \frac12 \langle x, \Lambda x \rangle$ from now on. Now, the ODE for this function admits a simple decomposition of form
\begin{equation}\nonumber
\ddot X_i + \frac{3}{t}\dot X_i + \lambda_i X_i = 0, \quad i=1,\ldots, n
\end{equation}
with $X_i(0) = x_{0,i}, \dot X_i(0) = 0$. Introduce $Y_i(u) = uX_i(u/\sqrt{\lambda_i})$, which satisfies
\begin{equation}\nonumber
u^2\ddot{Y_i} + u\dot{Y_i} + (u^2 - 1)Y_i=0.
\end{equation}
This is Bessel's differential equation of order one. Since $Y_i$ vanishes at $u=0$, we see that $Y_i$ is a constant multiple of $J_1$, the Bessel function of the first kind of order one.\footnote{Up to a constant multiplier, $J_1$ is the unique solution to the Bessel's differential equation $u^2 \ddot J_1 + u\dot J_1 + (u^2-1)J_1 = 0$ that is finite at the origin. In the analytic expansion of $J_1$, $m !!$ denotes the double factorial defined as $m !! = m \times (m-2) \times \cdots \times 2$ for even $m$, or $m !! = m \times (m-2) \times \cdots \times 1$ for odd $m$.}  It has an analytic expansion:
\[
J_1(u) =\sum_{m=0}^{\infty} \frac{(-1)^m}{(2m)!!(2m+2)!!}u^{2m+1},
\]
which gives the asymptotic expansion 
\[
J_1(u) = (1 + o(1))\frac{u}2
\]
when $u\goto 0$. Requiring $X_i(0) = x_{0,i}$, hence, we obtain
\begin{equation}\label{eq:bessel_present}
X_i(t) = \frac{2x_{0,i}}{t\sqrt{\lambda_i}}J_1(t\sqrt{\lambda_i}).
\end{equation}
For large $t$, the Bessel function has the following asymptotic form \citep[see e.g.][]{besselbook}:
\begin{equation}\label{eq:bessel-large}
J_1(t) = \sqrt{\frac{2}{\pi t}} \Big{(} \cos(t - 3\pi/4) + O(1/t)\Big{)}.
\end{equation}
This asymptotic expansion yields (note that $f^\star = 0$)
\begin{equation}\label{eq:bessel-three-power}
f(X(t))-f^\star = f(X(t)) = \sum_{i=1}^n\frac{2x_{0,i}^2}{t^2}J_1\left( t\sqrt{\lambda_i} \right)^2 = O\left(\frac{\|x_0 - x^\star\|^2}{t^3\sqrt{\min\lambda_i}}\right).
\end{equation}
On the other hand, \eqref{eq:bessel-large} and \eqref{eq:bessel-three-power} give a lower bound:
\begin{equation}\label{eq:bessel-three-lower}
\begin{aligned}
\limsup_{t\rightarrow\infty}t^3(f(X(t)) - f^\star) &\ge \lim_{t\goto\infty}\frac{1}{t} \int_0^t u^3(f(X(u)) - f^\star) \d u\\ &=\lim_{t\goto\infty}\frac{1}{t}\int_{0}^t\sum_{i=1}^n2x_{0,i}^2uJ_1(u\sqrt{\lambda_i})^2 \d u \\
&= \sum_{i=1}^n\frac{2x_{0,i}^2}{\pi\sqrt{\lambda_i}} \geq \frac{2\norm{x_0 - x^\star}^2}{\pi\sqrt{L}},
\end{aligned}
\end{equation}
where $L = \|A\|_2$ is the spectral norm of $A$. The first inequality follows by interpreting $\lim_{t\goto\infty}\frac{1}{t} \int_0^t u^3(f(X(u)) - f^\star) \d u$ as the mean of $u^3(f(X(u)) - f^\star)$ on $(0, \infty)$ in certain sense.

In view of \eqref{eq:bessel-three-power}, \N might possibly exhibit $O(1/k^3)$ convergence rate for strongly convex functions. This convergence rate is consistent with the second inequality in Theorem \ref{thm:large_r_dis}. In Section \ref{sec:bett-conv-rate}, we prove the $O(1/t^3)$ rate for a generalized version of \eqref{key}. However, \eqref{eq:bessel-three-lower} rules out the possibility of a higher order convergence rate.


Recall that the function considered in Figure \ref{fig:intro_scaling} is $f(x) = 0.02x_1^2 + 0.005x_2^2$, starting from $x_0 = (1, ~1)$. As the step size $s$ becomes smaller, the trajectory of \N converges to the solid curve represented via the Bessel function. While approaching the minimizer $x^\star$, each trajectory displays the oscillation pattern, as well-captured by the zoomed Figure \ref{fig:overshoot2}. This prevents Nesterov's scheme from achieving better convergence rate. The representation \eqref{eq:bessel_present} offers excellent explanation as follows. Denote by $T_1, T_2$, respectively, the approximate periodicities of the first component $|X_1|$ in absolute value and the second $|X_2|$. By \eqref{eq:bessel-large}, we get $T_1 = \pi/\sqrt{\lambda_1} = 5\pi$ and $T_2 = \pi/\sqrt{\lambda_2} = 10\pi$. Hence, as the amplitude gradually decreases to zero, the function $f = 2x_{0,1}^2J_1(\sqrt{\lambda_1}t)^2/t^2 + 2x_{0,2}^2J_1(\sqrt{\lambda_2}t)^2/t^2$ has a major cycle of $10\pi$, the least common multiple of $T_1$ and $T_2$. A careful look at Figure \ref{fig:bump_w_text} reveals that within each major bump, roughly, there are $10\pi/T_1 = 2 $ minor peaks. 

\subsection{Fluctuations of Strongly Convex $f$}
\label{sec:fluc}
The analysis carried out in the previous subsection only applies to convex quadratic functions. In this subsection, we extend the discussion to one-dimensional strongly convex functions. The Sturm-Picone theory \citep[see e.g.][]{sturm} is extensively used all along the analysis.

Let $f \in \mathcal{S}_{\mu,L}(\R)$. Without loss of generality, assume $f$ attains minimum at $x^\star = 0$. Then, by definition $\mu \le f'(x)/x \le L $ for any $x \ne 0$. Denoting by $X$ the solution to the ODE \eqref{key}, we consider the self-adjoint equation,
\begin{equation}\label{eq:adjoint}
(t^3Y')' + \frac{t^3f'(X(t))}{X(t)}Y = 0,
\end{equation}
which, apparently, admits a solution $Y(t) = X(t)$. To apply the Sturm-Picone comparison theorem, consider
\begin{equation}\nonumber
(t^3Y')' + \mu t^3Y = 0
\end{equation}
for a comparison. This equation admits a solution $\widetilde Y(t) = J_1(\sqrt{\mu}t)/t$. Denote by $\tilde t_1 < \tilde t_2 < \cdots$ all the positive roots of $J_1(t)$, which satisfy \citep[see e .g.][]{besselbook}
\begin{equation}\nonumber
3.8317 = \tilde t_1 - \tilde t_0 > \tilde t_2 - \tilde t_3 > \tilde t_3 - \tilde t_4 > \cdots > \pi,
\end{equation}
where $\tilde t_0 = 0$. Then, it follows that the positive roots of $\widetilde Y$ are $\tilde t_1/\sqrt{\mu},\, \tilde t_2/\sqrt{\mu}, \ldots$. Since $t^3f'(X(t))/X(t) \geq \mu t^3$, the Sturm-Picone comparison theorem asserts that $X(t)$ has a root in each interval $[\tilde t_i/\sqrt{\mu}, \tilde t_{i+1}/\sqrt{\mu}]$. 

To obtain a similar result in the opposite direction, consider
\begin{equation}\label{eq:adjoint-l}
(t^3Y')' + L t^3Y = 0.
\end{equation}
Applying the Sturm-Picone comparison theorem to \eqref{eq:adjoint} and \eqref{eq:adjoint-l}, we ensure that between any two consecutive positive roots of $X$, there is at least one $\tilde t_i/\sqrt{L}$. Now, we summarize our findings in the following. Roughly speaking, this result concludes that the oscillation frequency of the ODE solution is between $O(\sqrt{\mu})$ and $O(\sqrt{L})$.

\begin{theorem}
Denote by $0 < t_1 < t_2 < \cdots $ all the roots of $X(t) - x^\star$. Then these roots satisfy, for all $i \ge 1$,
\begin{equation}\nonumber
t_1 < \frac{7.6635}{\sqrt{\mu}}, ~ t_{i+1} - t_{i} < \frac{7.6635}{\sqrt{\mu}}, ~ t_{i+2} - t_i > \frac{\pi}{\sqrt{L}}.
\end{equation}
\end{theorem}

\subsection{Nesterov's Scheme Compared with Gradient Descent}
\label{sec:nest-scheme-comp}
The ansatz $t \approx k\sqrt{s}$ in relating the ODE and Nesterov's scheme is formally confirmed in Theorem \ref{thm:nesterov_limit}. Consequently, for any constant $t_c > 0$, this implies that $x_k$ does not change much for a range of step sizes $s$ if $k \approx t_c/\sqrt{s}$. To empirically support this claim, we present an example in Figure \ref{fig:lasso_scaling}, where the scheme minimizes $f(x) = \|y - Ax\|^2/2 + \|x\|_1$ with $y = (4, ~2, ~0)$ and $A(:, 1) = (0, ~2, ~4), ~ A(:, 2) = (1, ~1, ~1)$ starting from $x_0 = (2, ~0)$ (here $A(:, j)$ is the $j$th column of $A$). From this figure, we are delight to observe that $x_k$ with the same $t_c$ are very close to each other.

This interesting square-root scaling has the potential to shed light on the superiority of Nesterov's scheme over gradient descent. Roughly speaking, each iteration in Nesterov's scheme amounts to traveling $\sqrt{s}$ in time along the integral curve of \eqref{key}, whereas it is known that the simple gradient descent $x_{k+1} = x_k - s \nabla f(x_k)$ moves $s$ along the integral curve of $\dot X + \nabla f(X) = 0$. We expect that for small $s$ Nesterov's scheme moves more in each iteration since $\sqrt{s}$ is much larger than $s$. Figure \ref{fig:comp1_time} illustrates and supports this claim, where the function minimized is $f = |x_1|^3 + 5|x_2|^3 + 0.001(x_1 + x_2)^2$ with step size $s = 0.05$ (The coordinates are appropriately rotated to allow $x_0$ and $x^\star$ lie on the same horizontal line). The circles are the iterates for $k = 1, 10, 20, 30, 45, 60, 90, 120, 150, 190, 250, 300$. For Nesterov's scheme, the seventh circle has already passed $t = 15$, while for gradient descent the last point has merely arrived at $t = 15$.

\begin{figure}[!htp]
\centering
\begin{subfigure}[b]{0.48\textwidth}
\includegraphics[width = \textwidth]{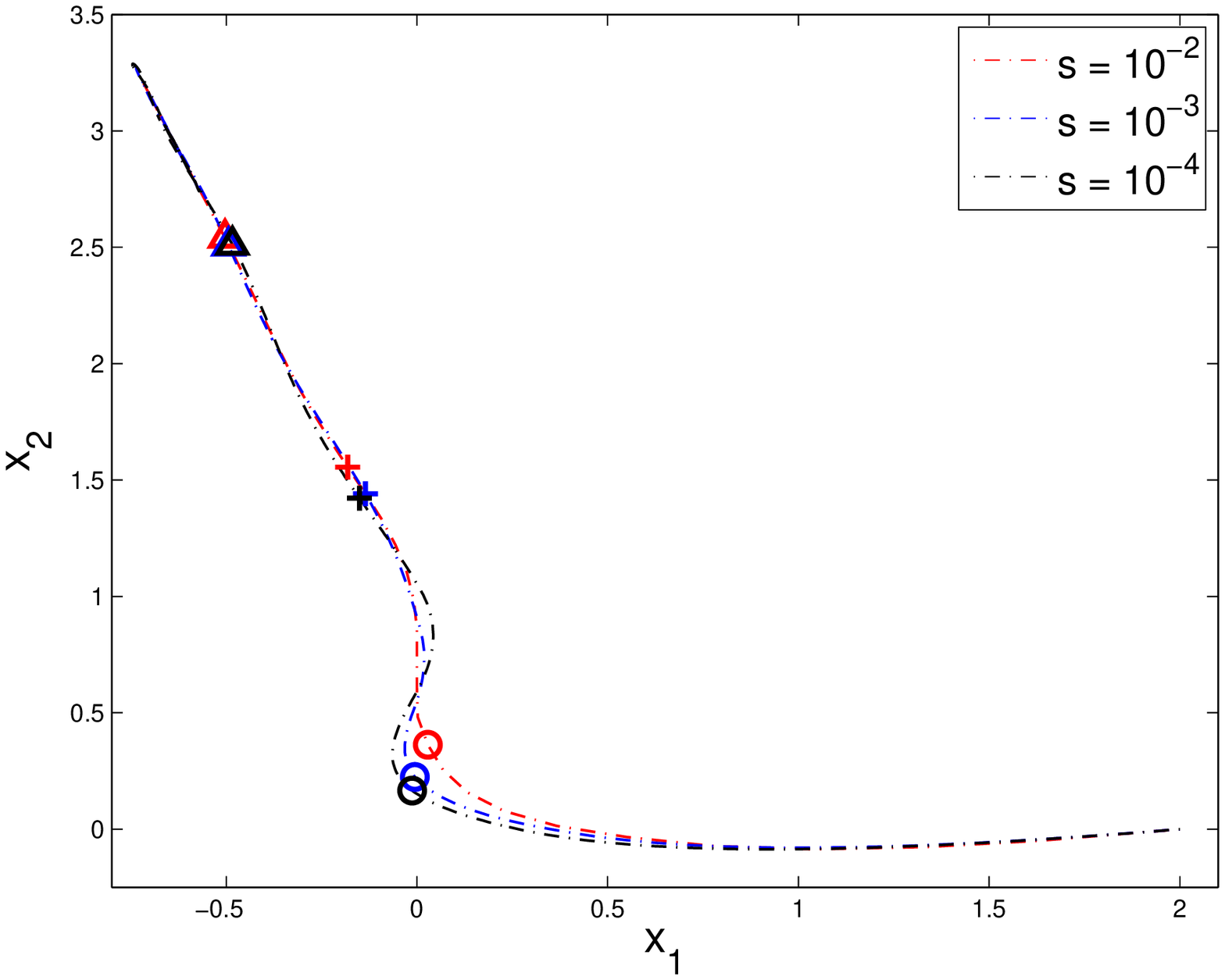}
\caption{Square-root scaling of $s$.}
\label{fig:lasso_scaling}
\end{subfigure}
\hfill
\begin{subfigure}[b]{0.48\textwidth}
\includegraphics[width = \textwidth]{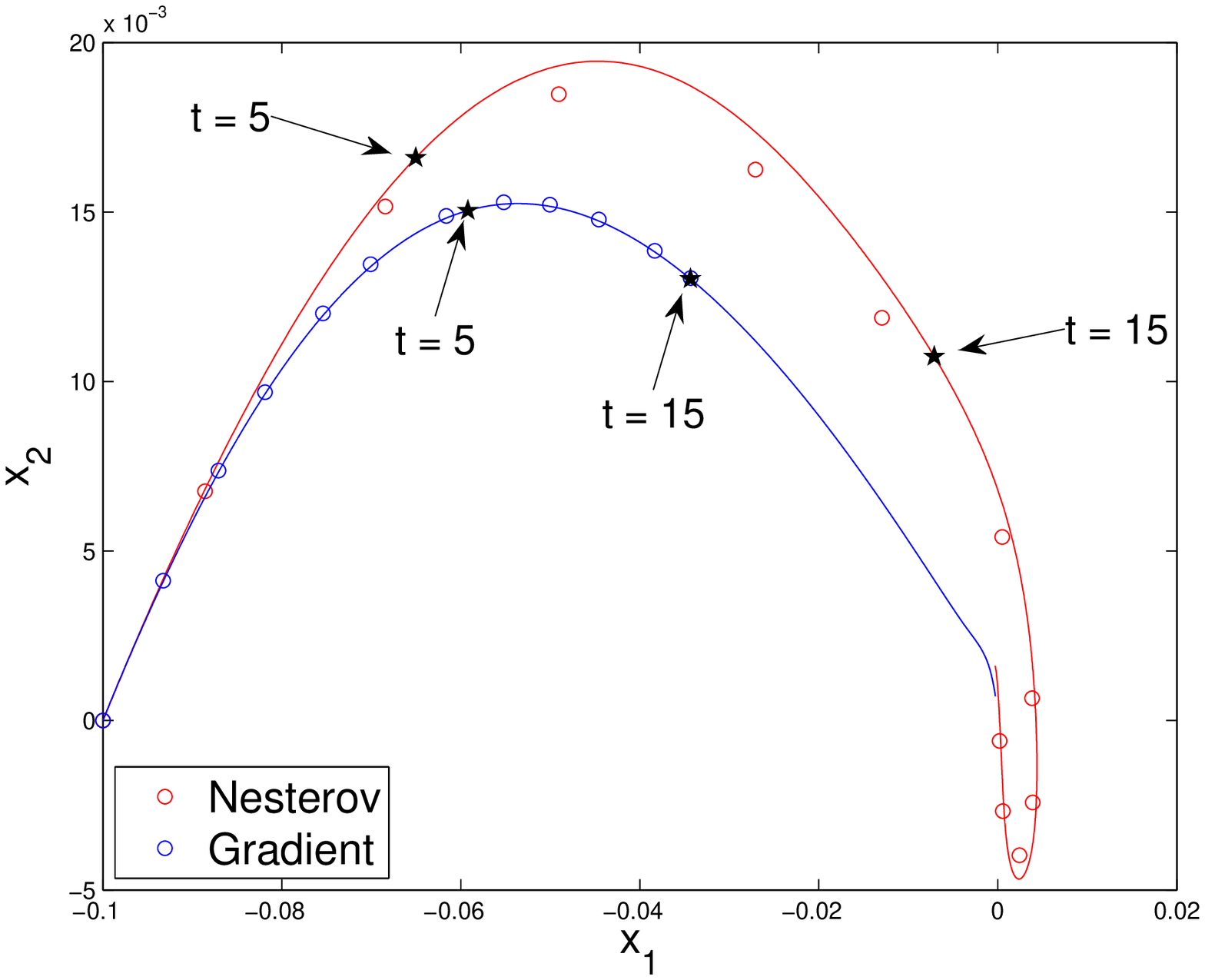}
\caption{Race between Nesterov's and gradient.}
\label{fig:comp1_time}
\end{subfigure}
\caption{In (a), the circles, crosses and triangles are $x_k$ evaluated at $k = \left\lceil 1/\sqrt{s} \right\rceil, \left\lceil 2/\sqrt{s} \right\rceil$ and $\left\lceil 3/\sqrt{s} \right\rceil$, respectively. In (b), the circles are iterations given by Nesterov's scheme or gradient descent, depending on the color, and the stars are $X(t)$ on the integral curves for $t = 5, 15$.}
\end{figure}

A second look at Figure \ref{fig:comp1_time} suggests that Nesterov's scheme allows a large deviation from its limit curve, as compared with gradient descent. This raises the question of the stable step size allowed for numerically solving the ODE \eqref{key} in the presence of accumulated errors. The finite difference approximation by the forward Euler method is
\begin{equation}\label{eq:ode_euler}
\frac{X(t+\D t) -2X(t) + X(t-\D t)}{\Delta t^2} + \frac{3}{t}\frac{X(t) - X(t-\D t)}{\D t} + \nabla f(X(t)) = 0,
\end{equation}
which is equivalent to
\begin{equation}\label{eq:like_nesterov}
X(t+\D t) = \Big{(}2 - \frac{3\D t}{t}\Big{)}X(t) - \D t^2 \nabla f(X(t)) -  \Big{(}1 - \frac{3\D t}{t}\Big{)} X(t-\D t).
\end{equation}
Assuming $f$ is sufficiently smooth, we have $\nabla f(x + \delta x) \approx \nabla f(x) + \nabla^2f(x)
\delta x$ for small perturbations $\delta x$, where $\nabla^2f(x)$ is the Hessian of $f$ evaluated at
$x$. Identifying $k = t/\D t$, the characteristic equation of this finite difference scheme is approximately
\begin{equation}\label{eq:character}
\det \left(\lambda^2 - \left( 2 - \Delta t^2 \nabla^2f -\frac{3\D t}{t}\right) \lambda + 1 - \frac{3\D t}{t} \right)= 0.
\end{equation}
The numerical stability of \eqref{eq:ode_euler} with respect to accumulated errors is equivalent to this: all the roots of \eqref{eq:character} lie in the unit circle \citep[see e.g.][]{leadernumerical}. When $\nabla^2f \preceq L I_n$ (i.e.~$LI_n - \nabla^2 f$ is positive semidefinite), if $\Delta t/t$ small and $\Delta t <
2/\sqrt{L}$, we see that all the roots of \eqref{eq:character} lie in
the unit circle. On the other hand, if $\Delta t > 2/\sqrt{L}$,
\eqref{eq:character} can possibly have a root $\lambda$ outside the
unit circle, causing numerical instability. Under our identification
$s = \D t^2$, a step size of $s = 1/L$ in Nesterov's scheme
\eqref{eq:nesterov-scheme} is approximately equivalent to a step size
of $\D t = 1/\sqrt{L}$ in the forward Euler method, which is stable
for numerically integrating \eqref{eq:ode_euler}.

As a comparison, note that the finite difference scheme of the ODE $\dot{X}(t) + \nabla f(X(t)) = 0$, which models gradient descent
with updates $x_{k+1} = x_k - s\nabla f(x_k)$, has the characteristic equation
$\det(\lambda - (1 - \Delta t \nabla^2 f)) = 0$. Thus, to guarantee
$-I_n \preceq 1 - \Delta t \nabla^2 f \preceq I_n$ in worst case
analysis, one can only choose $\Delta t \le 2/L$ for a fixed step
size, which is much smaller than the step size $2/\sqrt{L}$ for \eqref{eq:ode_euler} when $\nabla f$ is very variable, i.e., $L$ is large.

\section{The Magic Constant 3}
\label{sec:varycoeff}
Recall that the constant 3 appearing in the coefficient of $\dot X$ in \eqref{key} originates from $(k+2) - (k-1) = 3$. This number leads to the momentum coefficient in \eqref{eq:nesterov-scheme} taking the form $(k-1)/(k+2) = 1 - 3/k + O(1/k^2)$. In this section, we demonstrate that 3 can be replaced by any larger number, while maintaining the $O(1/k^2)$ convergence rate. To begin with, let us consider the following ODE parameterized by a constant $r$:
\begin{equation}\label{eq:ode_r}
\ddot{X} + \frac{r}{t}\dot{X} + \nabla f(X) = 0
\end{equation}
with initial conditions $X(0) = x_0, \dot X(0) = 0$. The proof of Theorem \ref{thm:regularity}, which seamlessly applies here, guarantees the existence and uniqueness of the solution $X$ to this ODE. 

Interpreting the damping ratio $r/t$ as a measure of
friction\footnote{In physics and engineering, damping may be modeled
  as a force proportional to velocity but opposite in direction,
  i.e.~resisting motion; for instance, this force may be used as an
  approximation to the friction caused by drag. In our model, this
  force would be proportional to $-\frac{r}{t} \dot{X}$ where
  $\dot{X}$ is velocity and $\frac{r}{t}$ is the damping coefficient.}
in the damping system, our results say that more friction does not end
the $O(1/t^2)$ and $O(1/k^2)$ convergence rate.  On the other hand, in
the lower friction setting, where $r$ is smaller than 3, we can no
longer expect inverse quadratic convergence rate, unless some
additional structures of $f$ are imposed. We believe that this
striking phase transition at 3 deserves more attention as an
interesting research challenge.

\subsection{High Friction}
\label{sec:high-friction}
Here, we study the convergence rate of \eqref{eq:ode_r} with $r > 3$ and $f \in \mathcal{F}_{\infty}$. Compared with \eqref{key}, this new ODE as a damping suffers from higher friction. Following the strategy adopted in the proof of Theorem \ref{thm:ode_rate}, we consider a new energy functional defined as
\begin{equation}\nonumber
\mathcal{E}(t) = \frac{2t^2}{r-1}(f(X(t)) - f^\star) + (r-1)\left\|X(t) + \frac{t}{r-1}\dot{X(t)} - x^\star\right\|^2.
\end{equation}
By studying the derivative of this functional, we get the following result.

\begin{theorem}\label{thm:large_r}
The solution $X$ to \eqref{eq:ode_r} satisfies
\begin{equation}\nonumber
f(X(t)) - f^\star \le \frac{(r-1)^2\|x_0 - x^\star\|^2}{2t^2}, \quad \int^{\infty}_0 t(f(X(t))-f^\star)\d t \le \frac{(r-1)^2\|x_0 - x^\star\|^2}{2(r-3)}.
\end{equation}
\end{theorem}

\begin{proof}
Noting $r\dot{X} + t\ddot{X} = -t\nabla f(X)$, we get $\dot{\mathcal{E}}$ equal to
\begin{multline}\label{eq:r_e_der}
\frac{4t}{r-1}(f(X) - f^\star) + \frac{2t^2}{r-1}\langle \nabla f, \dot{X} \rangle 
+ 2\langle X + \frac{t}{r-1}\dot{X} - x^\star, r\dot{X} + t\ddot{X} \rangle\\
= \frac{4t}{r-1}(f(X) - f^\star) - 2t\langle X - x^\star, \nabla f(X) \rangle 
\le -\frac{2(r-3)t}{r-1}(f(X) - f^\star),
\end{multline}
where the inequality follows from the convexity of $f$. Since $f(X)\geq f^\star$, the last display implies that $\mathcal{E}$ is non-increasing. Hence
\begin{equation}\nonumber
\frac{2t^2}{r-1}(f(X(t)) - f^\star) \le \mathcal{E}(t) \le \mathcal{E}(0) = (r-1)\|x_0 - x^\star\|^2,
\end{equation}
yielding the first inequality of this theorem. To complete the proof, from \eqref{eq:r_e_der} it follows that
\begin{equation}\nonumber
\int^{\infty}_0\frac{2(r-3)t}{r-1}(f(X) - f^\star)\d t \le -\int^{\infty}_0 \frac{\d\mathcal{E}}{\d t}\d t = \mathcal{E}(0) - \mathcal{E}(\infty) \le  (r-1)\|x_0 - x^\star\|^2,
\end{equation}
as desired for establishing the second inequality.
\end{proof}
The first inequality is the same as \eqref{eq:t_2_rate} for the ODE \eqref{key}, except for a larger constant $(r-1)^2/2$. The second inequality measures the error $f(X(t)) - f^\star$ in an average sense, and cannot be deduced from the first inequality.

Now, it is tempting to obtain such analogs for the discrete Nesterov's scheme as well. Following the formulation of \cite{BeckTeboulle}, we wish to minimize $f$ in the composite form $f(x) = g(x) + h(x)$, where $g\in\mathcal{F}_L$ for some $L>0$ and $h$ is convex on $\R^n$ possibly assuming extended value $\infty$. Define the proximal subgradient
\begin{equation}\nonumber
G_s(x) \triangleq \frac{x - \mbox{argmin}_{z}\left( \|z - (x-s\nabla g(x))\|^2/(2s) + h(z) \right) }{s}.
\end{equation}
Parametrizing by a constant $r$, we propose the generalized Nesterov's scheme,
\begin{gather}\label{eq:nesterov_general}
\begin{aligned}
&x_k = y_{k-1} -  sG_s(y_{k-1})\\
&y_k = x_k + \frac{k-1}{k+r-1}(x_k -x_{k-1}),
\end{aligned}
\end{gather}
starting from $y_0 = x_0$. The discrete analog of Theorem \ref{thm:large_r} is below.
\begin{theorem}\label{thm:large_r_dis}
The sequence $\{x_k\}$ given by \eqref{eq:nesterov_general} with $0 < s \le 1/L$ satisfies
\begin{equation}\nonumber
f(x_k) - f^\star \le \frac{(r-1)^2\|x_0 - x^\star\|^2}{2s(k+r-2)^2}, \quad \sum_{k=1}^{\infty}(k+r-1)(f(x_k)-f^\star) \le \frac{(r-1)^2\|x_0-x^\star\|^2}{2s(r-3)}.
\end{equation}
\end{theorem}
The first inequality suggests that the generalized Nesterov's schemes still achieve $O(1/k^2)$ convergence rate. However, if the error bound satisfies $f(x_{k'}) - f^\star \geq c/k'^2$ for some arbitrarily small $c > 0$ and a dense subsequence $\{k'\}$, i.e., $|\{k'\}\cap \{1, \ldots, m\}| \geq \alpha m$ for all $m \ge 1$ and some $\alpha > 0$, then the second inequality of the theorem would be violated. To see this, note that if it were the case, we would have $(k'+r-1)(f(x_{k'})-f^\star) \gtrsim \frac1{k'}$; the sum of the harmonic series $\frac1{k'}$ over a dense subset of $\{1, 2, \ldots\}$ is infinite. Hence, the second inequality is not trivial because it implies the error bound is, in some sense, $O(1/k^2)$ suboptimal. 

Now we turn to the proof of this theorem. It is worth pointing out that, though based on the same idea, the proof below is much more complicated than that of Theorem \ref{thm:large_r}. 
\begin{proof}
Consider the discrete energy functional,
\begin{equation}\nonumber
\mathcal{E}(k) = \frac{2(k+r-2)^2s}{r-1}(f(x_k)-f^\star) + (r-1)\|z_k - x^\star\|^2,
\end{equation}
where $z_k = (k+r-1)y_k/(r-1) - kx_k/(r-1)$.
If we have
\begin{equation}\label{eq:recursive}
\mathcal{E}(k) + \frac{2s[(r-3)(k+r-2)+1]}{r-1}(f(x_{k-1})-f^\star) \le \mathcal{E}(k-1),
\end{equation}
then it would immediately yield the desired results by summing \eqref{eq:recursive} over $k$. That is, by recursively applying \eqref{eq:recursive}, we see
\begin{multline}\nonumber
\mathcal{E}(k) + \sum_{i=1}^k\frac{2s[(r-3)(i+r-2)+1]}{r-1}(f(x_{i-1})-f^\star) \\
\le \mathcal{E}(0) = \frac{2(r-2)^2s}{r-1}(f(x_0)-f^\star) + (r-1)\|x_0 - x^\star\|^2, 
\end{multline}
which is equivalent to
\begin{equation}\label{eq:energy_bound}
\mathcal{E}(k) + \sum_{i=1}^{k-1}\frac{2s[(r-3)(i+r-1)+1]}{r-1}(f(x_{i})-f^\star) \le (r-1)\|x_0 - x^\star\|^2.
\end{equation}
Noting that the left-hand side of \eqref{eq:energy_bound} is lower bounded by $2s(k+r-2)^2(f(x_k)-f^\star)/(r-1)$, we thus obtain the first inequality of the theorem. Since $\mathcal{E}(k)\geq 0$, the second inequality is verified via taking the limit $k\goto\infty$ in \eqref{eq:energy_bound} and replacing $(r-3)(i+r-1)+1$ by $(r-3)(i+r-1)$.

We now establish \eqref{eq:recursive}. For $s\le 1/L$, we have the basic inequality,
\begin{equation}\label{eq:prox_ineq}
f(y-sG_s(y)) \le f(x) + G_s(y)^T(y-x) - \frac{s}{2}\|G_s(y)\|^2,
\end{equation}
for any $x$ and $y$. Note that $y_{k-1}-sG_s(y_{k-1})$ actually coincides with $x_k$. Summing of $(k-1)/(k+r-2)\times\eqref{eq:prox_ineq}$ with $x=x_{k-1}, y = y_{k-1}$ and $(r-1)/(k+r-2)\times\eqref{eq:prox_ineq}$ with $x=x^\star, y=y_{k-1}$ gives
\begin{equation}\nonumber
\begin{aligned}
f(x_k) &\le \frac{k-1}{k+r-2}f(x_{k-1}) + \frac{r-1}{k+r-2}f^\star\\
& + \frac{r-1}{k+r-2}G_s(y_{k-1})^T\Big{(}\frac{k+r-2}{r-1}y_{k-1}-\frac{k-1}{r-1}x_{k-1} - x^\star\Big{)} - \frac{s}{2}\|G_s(y_{k-1})\|^2\\
&= \frac{k-1}{k+r-2}f(x_{k-1}) + \frac{r-1}{k+r-2}f^\star + \frac{(r-1)^2}{2s(k+r-2)^2}\Big{(}\|z_{k-1}-x^\star\|^2 - \|z_{k}-x^\star\|^2\Big{)}, \\
\end{aligned}
\end{equation}
where we use $z_{k-1}-s(k+r-2)G_s(y_{k-1})/(r-1)=z_{k}$. Rearranging the above inequality and multiplying by $2s(k+r-2)^2/(r-1)$ gives the desired \eqref{eq:recursive}.

\end{proof}

In closing, we would like to point out this new scheme is equivalent to setting $\theta_k = (r-1)/(k+r-1)$ and letting $\theta_k(\theta_{k-1}^{-1} - 1)$ replace the momentum coefficient $(k-1)/(k+r-1)$. Then, the equal sign $``="$ in the update $\theta_{k+1} = (\sqrt{\theta_k^4 + 4\theta_k^2} - \theta_k^2 )/2$ has to be replaced by an inequality sign $``\ge"$. In examining the proof of Theorem 1(b) in \cite{tsengapprox}, we can get an alternative proof of Theorem \ref{thm:large_r_dis}.

\subsection{Low Friction}
\label{sec:low-friction}
Now we turn to the case $r < 3$. Then, unfortunately, the energy functional approach for proving Theorem \ref{thm:large_r} is no longer valid, since the left-hand side of \eqref{eq:r_e_der} is positive in general. In fact, there are counterexamples that fail the desired $O(1/t^2)$ or $O(1/k^2)$ convergence rate. We present such examples in continuous time. Equally, these examples would also violate the $O(1/k^2)$ convergence rate in the discrete schemes, and we forego the details.

Let $f(x) = \frac12\|x\|^2$ and $X$ be the solution to \eqref{eq:ode_r}. Then, $Y = t^{\frac{r-1}{2}}X$ satisfies
\begin{equation}\nonumber
t^2\ddot Y + t\dot Y + (t^2 - (r-1)^2/4)Y = 0.
\end{equation}
With the initial condition $Y(t) \approx t^{\frac{r-1}{2}}x_0$ for small $t$, the solution to the above Bessel equation in a vector form of order $(r-1)/2$ is $Y(t) = 2^{\frac{r-1}{2}}\Gamma((r+1)/2)J_{(r-1)/2}(t)x_0$. Thus,
\begin{equation}\nonumber
X(t) = \frac{2^{\frac{r-1}{2}}\Gamma((r+1)/2)J_{(r-1)/2}(t)}{t^{\frac{r-1}{2}}}x_0.
\end{equation}
For large $t$, the Bessel function $J_{(r-1)/2}(t) = \sqrt{2/(\pi t)}\big( \cos(t-(r-1)\pi/4-\pi/4) + O(1/t) \big)$. Hence,
\begin{equation}\nonumber
f(X(t)) - f^\star = O\left( \|x_0 - x^\star\|^2/t^r \right),
\end{equation}
where the exponent $r$ is tight. This rules out the possibility of inverse quadratic convergence of the generalized ODE and scheme for all $f \in \mathcal{F}_L$ if $r < 2$. An example with $r = 1$ is plotted in Figure \ref{fig:small_r_fail}.

Next, we consider the case $2 \le r < 3$ and let $f(x) = |x|$ (this
also applies to multivariate $f = \|x\|$).\footnote{This function does
  not have a Lipschitz continuous gradient. However, a similar pattern
  as in Figure \ref{fig:small_r_fail} can be also observed if we
  smooth $|x|$ at an arbitrarily small vicinity of 0.} Starting from
$x_0 > 0$, we get $X(t) = x_0 - \frac{t^2}{2(1+r)}$ for $t \le
\sqrt{2(1+r)x_0}$. Requiring continuity of $X$ and $\dot X$ at the
change point 0, we get
\[
X(t) = \frac{t^2}{2(1+r)} + \frac{2(2(1+r)x_0)^{\frac{r+1}{2}}}{(r^2-1)t^{r-1}} - \frac{r+3}{r-1}x_0
\]
for $\sqrt{2(1+r)x_0} < t \le \sqrt{2c^\star(1+r)x_0}$, where $c^\star$ is the positive root other than 1 of $(r-1)c + 4c^{-\frac{r-1}{2}} = r+3$. Repeating this process solves for $X$. Note that $t^{1-r}$ is in the null space of $\ddot X + r\dot X/t$ and satisfies $t^2 \times t^{1-r} \goto \infty$ as $t \goto \infty$. For illustration, Figure \ref{fig:small_r_abs_x} plots $t^2(f(X(t)) - f^\star)$ and $sk^2(f(x_k) - f^\star)$ with $r = 2, 2.5$, and $r = 4$ for comparison\footnote{For Figures \ref{fig:r_2b}, \ref{fig:r_25b} and \ref{fig:r_4b}, if running generalized Nesterov's schemes with too many iterations (e.g.~$10^5$), the deviations from the ODE will grow. Taking a sufficiently small $s$ can solve this issue.}. It is clearly that inverse quadratic convergence does not hold for $r = 2, 2.5$, that is, \eqref{eq:sk_quad_err} does not hold for $r < 3$. Interestingly, in Figures \ref{fig:r_2} and \ref{fig:r_2b}, the scaled errors at peaks grow linearly, whereas for $r = 2.5$, the growth rate, though positive as well, seems sublinear.

\begin{figure}[!htp]
\centering
\begin{subfigure}[b]{0.32\textwidth}
\includegraphics[width = \textwidth, height=1.2in]{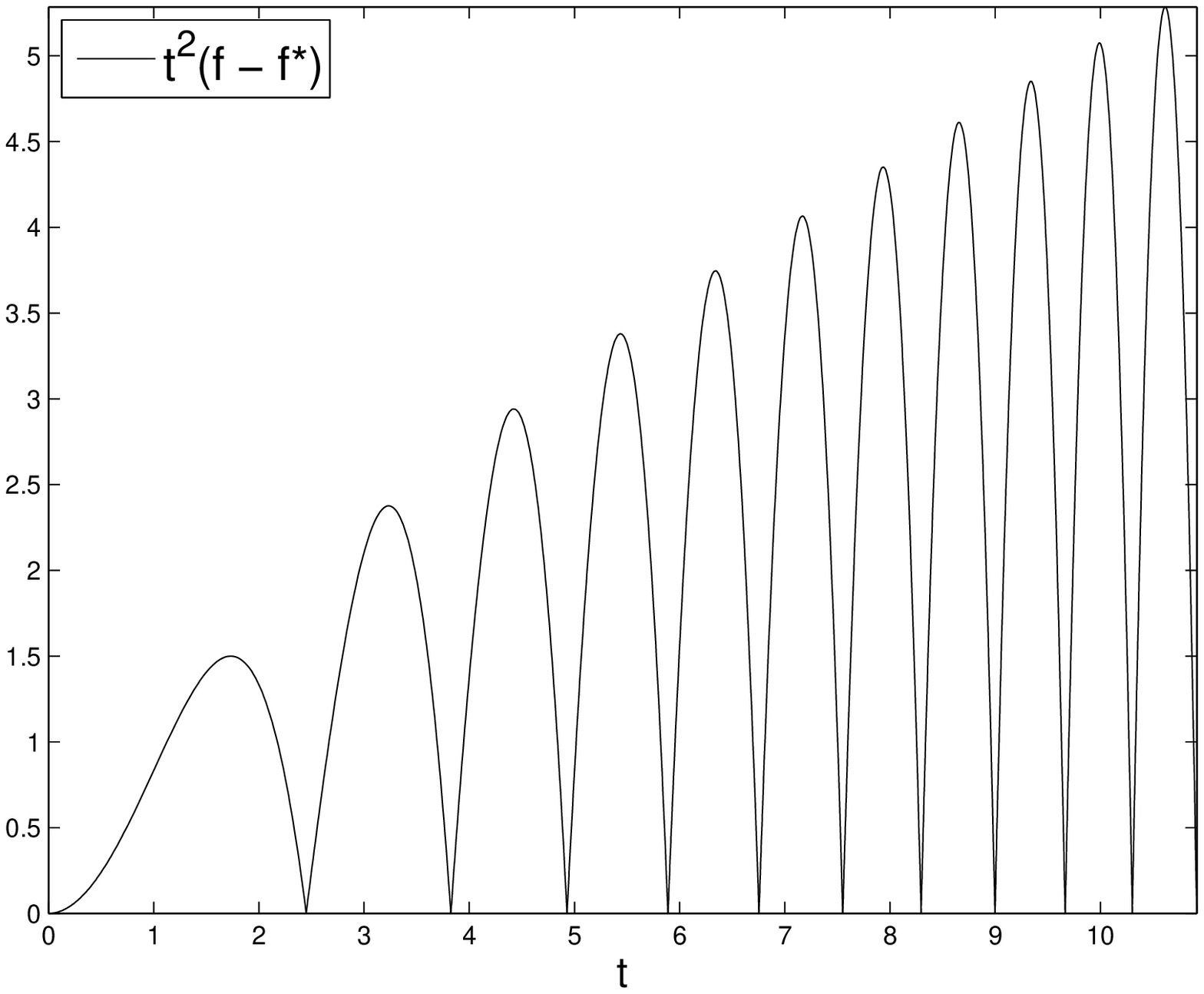}
\caption{ODE \eqref{eq:ode_r} with $r = 2$.}
\label{fig:r_2}
\end{subfigure}
\hfill
\begin{subfigure}[b]{0.32\textwidth}
\includegraphics[width = \textwidth, height=1.2in]{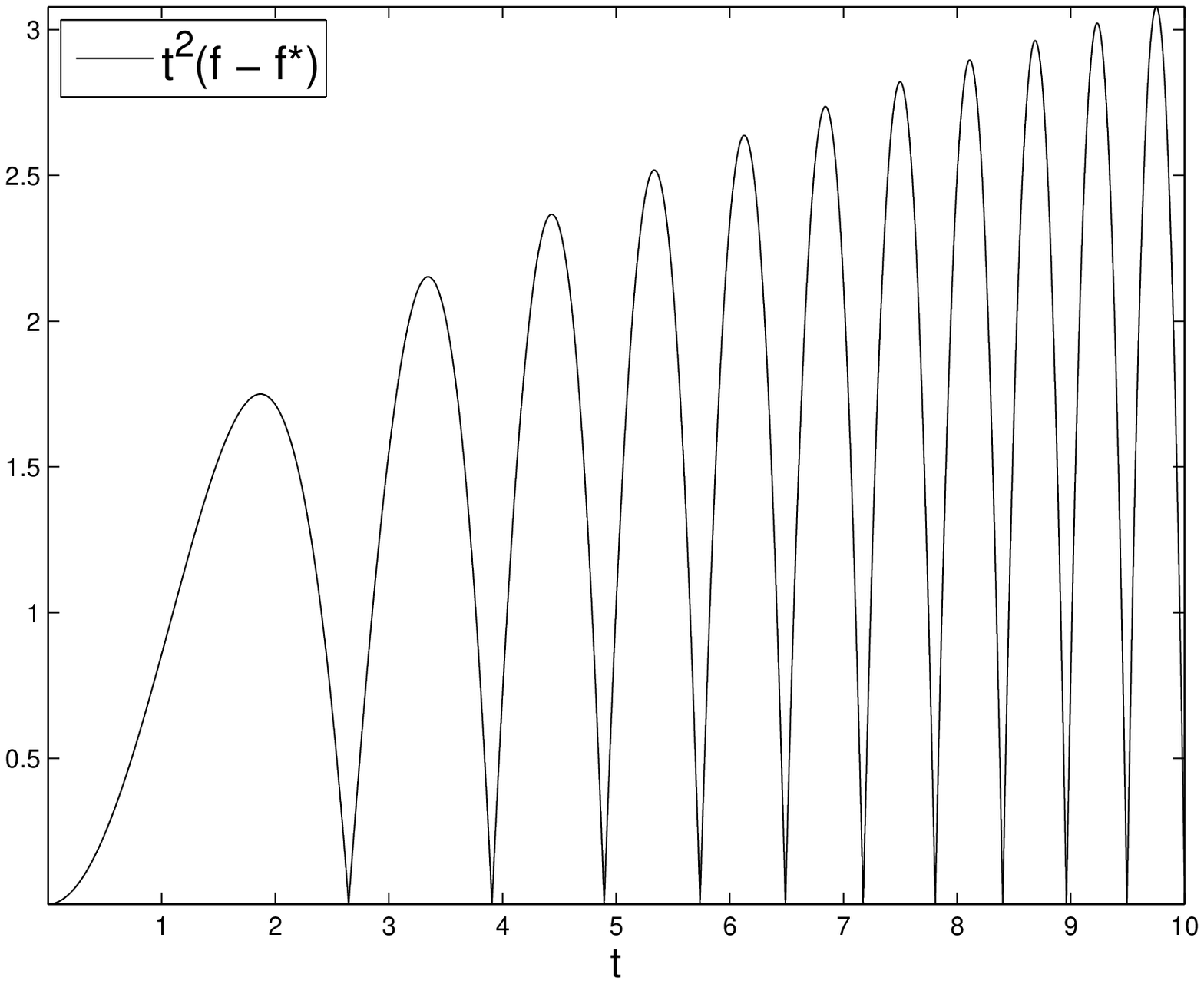}
\caption{ODE \eqref{eq:ode_r} with $r = 2.5$.}
\label{fig:r_25}
\end{subfigure}
\hfill
\begin{subfigure}[b]{0.32\textwidth}
\includegraphics[width = \textwidth, height=1.2in]{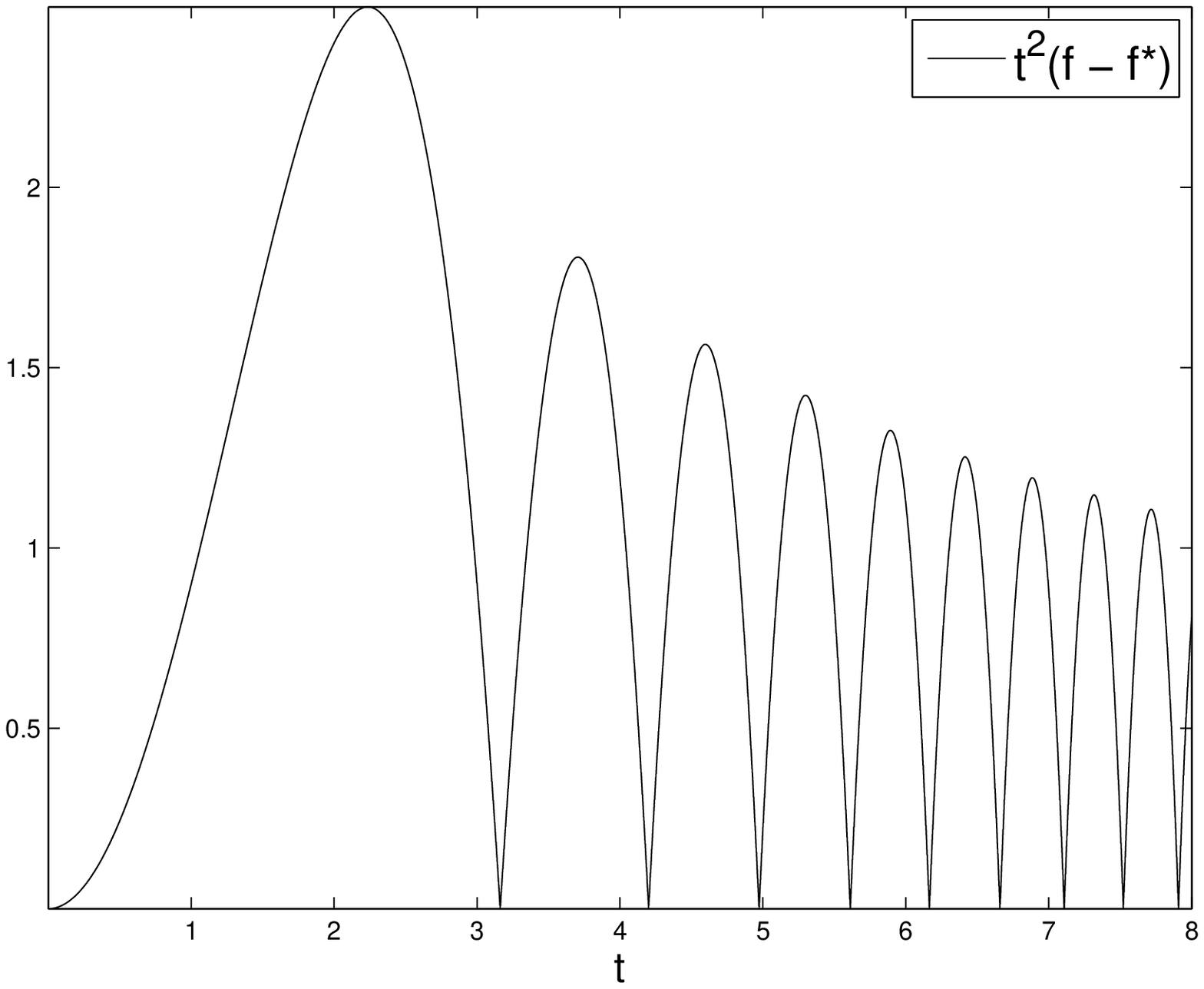}
\caption{ODE \eqref{eq:ode_r} with $r = 4$.}
\label{fig:r_4}
\end{subfigure}
\begin{subfigure}[b]{0.32\textwidth}
\includegraphics[width = \textwidth, height=1.2in]{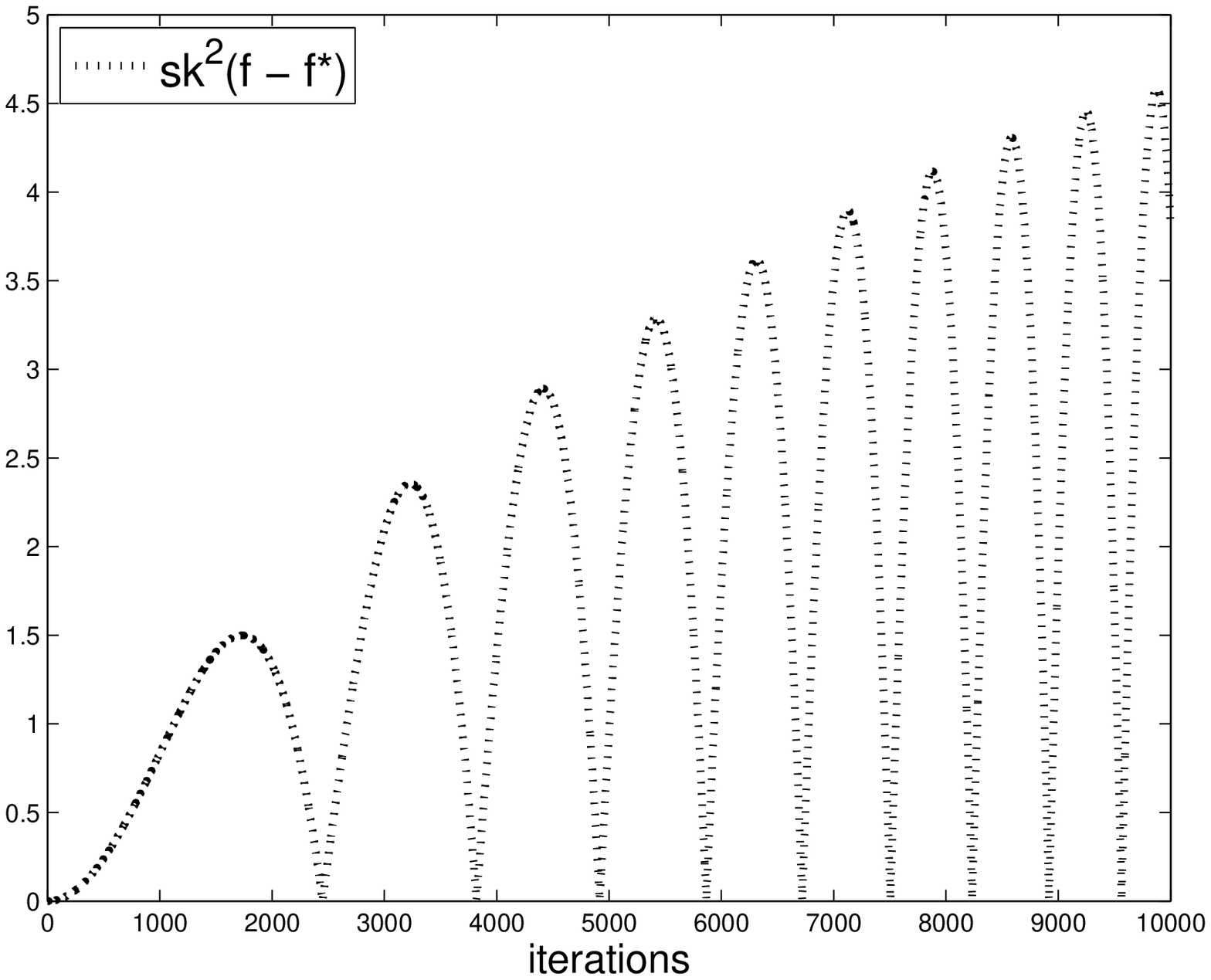}
\caption{Scheme \eqref{eq:nesterov_general} with $r = 2$.}
\label{fig:r_2b}
\end{subfigure}
\hfill
\begin{subfigure}[b]{0.32\textwidth}
\includegraphics[width = \textwidth, height=1.2in]{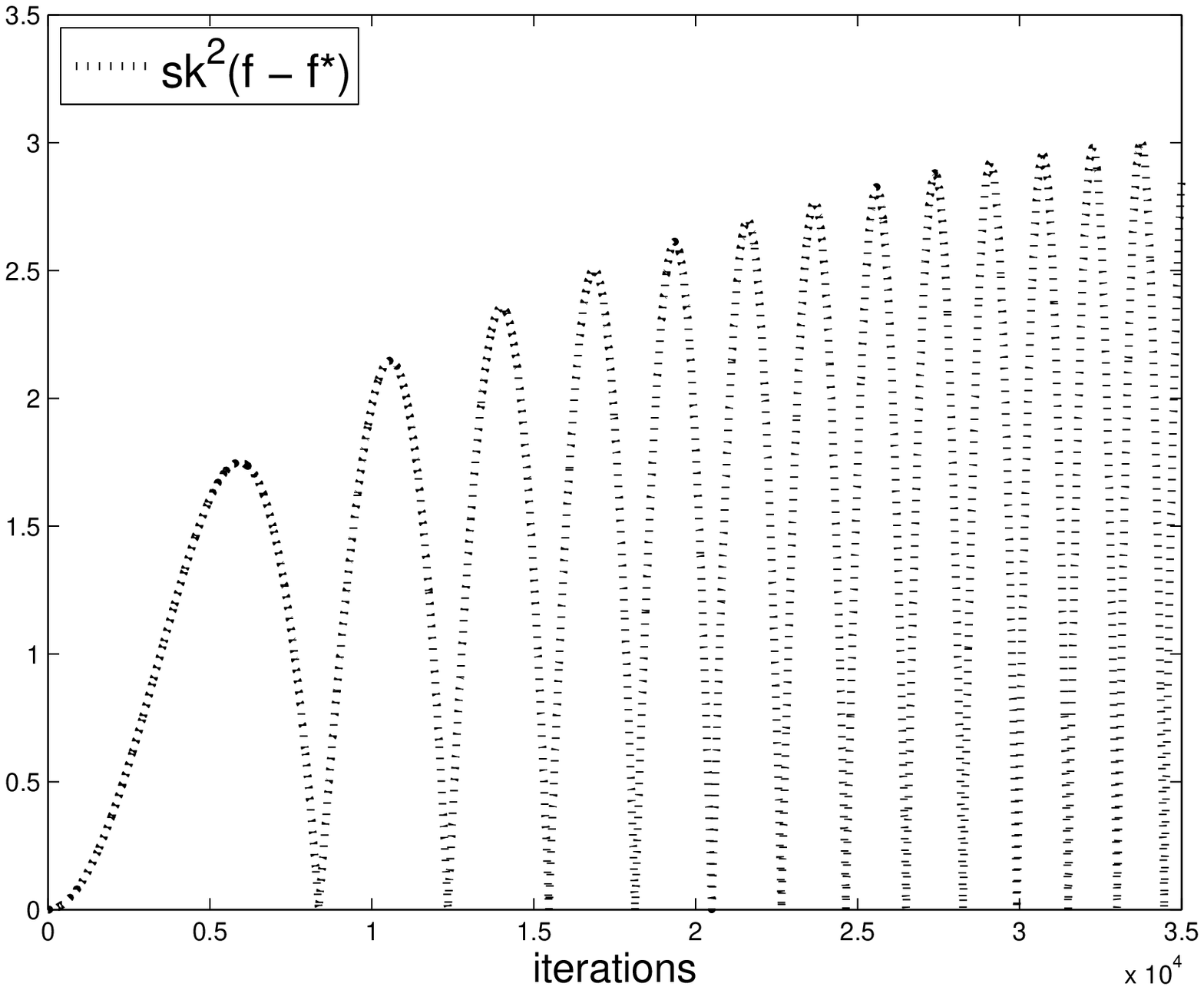}
\caption{Scheme \eqref{eq:nesterov_general} with $r = 2.5$.}
\label{fig:r_25b}
\end{subfigure}
\hfill
\begin{subfigure}[b]{0.32\textwidth}
\includegraphics[width = \textwidth, height=1.2in]{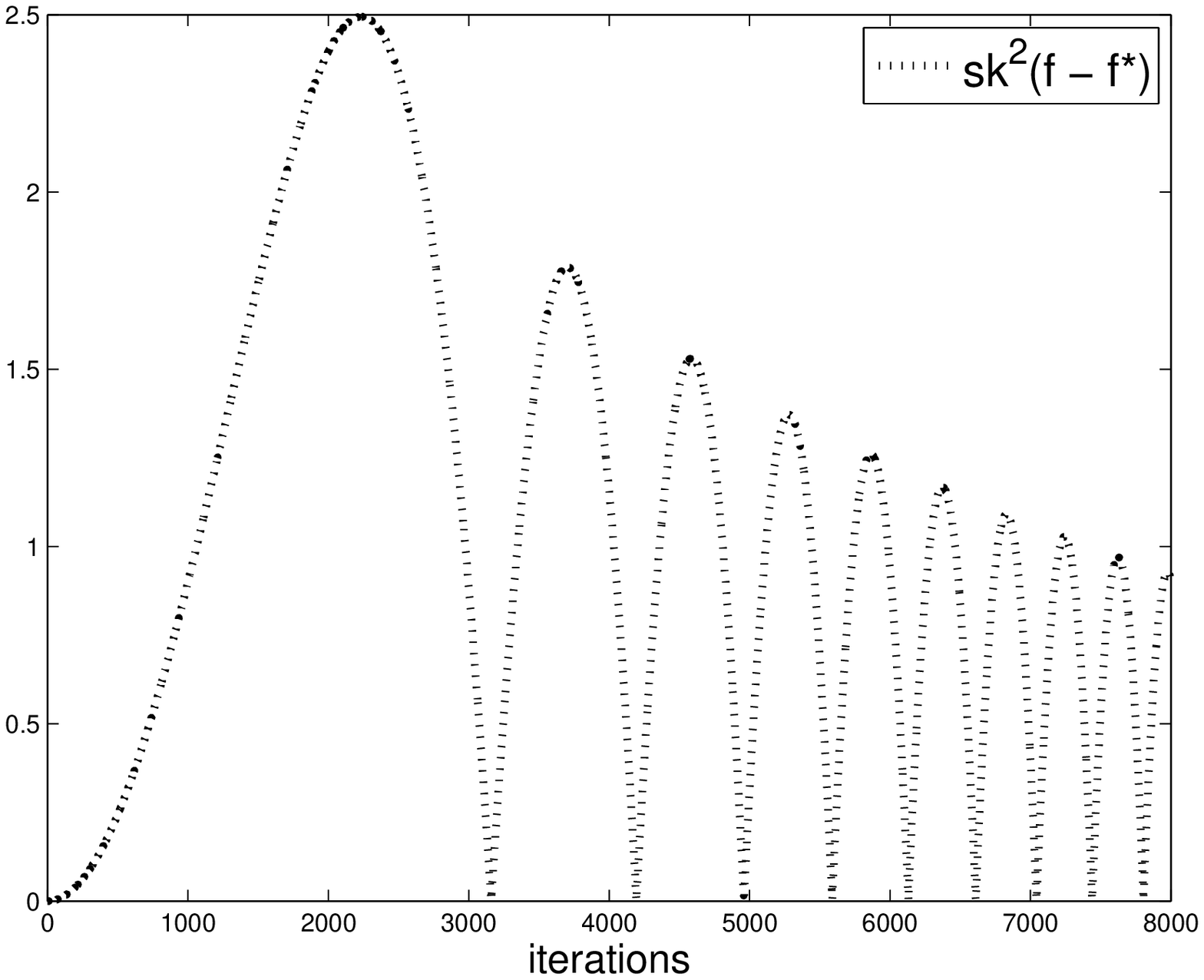}
\caption{Scheme \eqref{eq:nesterov_general} with $r = 4$.}
\label{fig:r_4b}
\end{subfigure}
\caption{Scaled errors $t^2(f(X(t))-f^\star)$ and $sk^2(f(x_k)-f^\star)$ of generalized ODEs and schemes for minimizing $f = |x|$. In (d), the step size $s = 10^{-6}$, in (e), $s = 10^{-7}$, and in (f), $s = 10^{-6}$.}
\label{fig:small_r_abs_x}
\end{figure}

However, if $f$ possesses some additional property, inverse quadratic convergence is still guaranteed, as stated below. In that theorem, $f$ is assumed to be a continuously differentiable convex function.

\begin{theorem}\label{thm:small_r}
Suppose $1 < r < 3$ and let $X$ be a solution to the ODE \eqref{eq:ode_r}. If $(f - f^\star)^{\frac{r-1}{2}}$ is also convex, then
\begin{equation}\nonumber
f(X(t)) - f^\star \le \frac{(r-1)^2\|x_0 - x^\star\|^2}{2t^2}.
\end{equation}
\end{theorem}

\begin{proof}
Since $(f - f^\star)^{\frac{r-1}{2}}$ is convex, we obtain
\[
\left( f(X(t)) - f^\star \right)^{\frac{r-1}{2}} \le \langle X - x^\star, \nabla(f(X) - f^\star)^{\frac{r-1}{2}} \rangle = \frac{r-1}{2}(f(X) - f^\star)^{\frac{r-3}{2}}\langle X - x^\star, \nabla f(X)\rangle,
\]
which can be simplified to $\frac{2}{r-1}(f(X) - f^\star) \le \langle X - x^\star, \nabla f(X)\rangle$. This inequality combined with \eqref{eq:r_e_der} leads to the monotonically decreasing of $\mathcal{E}(t)$ defined for Theorem \ref{thm:large_r}. This completes the proof by noting $f(X)-f^\star\le (r-1)\mathcal{E}(t)/(2t^2)\le (r-1)\mathcal{E}(0)/(2t^2) = (r-1)^2\|x_0 - x^\star\|^2/(2t^2)$.
\end{proof}


\subsection{Strongly Convex $f$}
\label{sec:bett-conv-rate}
Strong convexity is a desirable property for optimization. Making use of this property carefully suggests a generalized Nesterov's scheme that achieves optimal linear convergence \citep{nesterov-book}. In that case, even vanilla gradient descent has a linear convergence rate. Unfortunately, the example given in the previous
subsection simply rules out such possibility for \eqref{eq:nesterov-scheme} and
its generalizations \eqref{eq:nesterov_general}. However, from a
different perspective, this example suggests that $O(t^{-r})$
convergence rate can be expected for \eqref{eq:ode_r}. In the next
theorem, we prove a slightly weaker statement of this kind, that is, a
provable $O(t^{-\frac{2r}{3}})$ convergence rate is established for
strongly convex functions. Bridging this gap may require new tools and
more careful analysis.

Let $f\in\mathcal{S}_{\mu,L}(\R^n)$ and consider a new energy functional for $\alpha > 2$ defined as
\begin{equation}\nonumber
\mathcal{E}(t; \alpha) = t^{\alpha}(f(X(t)) - f^\star) + \frac{(2r-\alpha)^2t^{\alpha-2}}{8}\Big{\|}X(t) + \frac{2t}{2r-\alpha}\dot X - x^\star\Big{\|}^2.
\end{equation}
When clear from the context, $\mathcal{E}(t; \alpha)$ is simply denoted as $\mathcal{E}(t)$. For $r > 3$, taking $\alpha = 2r/3$ in the theorem stated below gives $f(X(t)) - f^\star \lesssim \|x_0 -x^\star\|^2/t^{\frac{2r}{3}}$. 

\begin{theorem}\label{thm:cube_rate}
For any $f\in\mathcal{S}_{\mu,L}(\R^n)$, if $2 \le \alpha \le 2r/3$ we get
\begin{equation}\nonumber
f(X(t)) - f^\star \le \frac{C\|x_0-x^\star\|^2}{\mu^{\frac{\alpha-2}{2}}t^{\alpha}}
\end{equation}
for any $t > 0$. Above, the constant $C$ only depends on $\alpha$ and $r$.
\end{theorem}

\begin{proof}
Note that $\dot{\mathcal{E}}(t; \alpha)$ equals
\begin{multline}\label{eq:alpha_energy}
\alpha t^{\alpha - 1}(f(X)-f^\star) - \frac{(2r-\alpha)t^{\alpha-1}}{2}\langle X- x^\star, \nabla f(X)\rangle + \frac{(\alpha-2)(2r-\alpha)^2t^{\alpha-3}}{8}\|X-x^\star\|^2 \\+ \frac{(\alpha-2)(2r-\alpha)t^{\alpha-2}}{4}\langle \dot X, X - x^\star\rangle.
\end{multline}
By the strong convexity of $f$, the second term of the right-hand side of \eqref{eq:alpha_energy} is bounded below as
\begin{equation}\nonumber
\begin{aligned}
\frac{(2r-\alpha)t^{\alpha-1}}{2}\langle X- x^\star, \nabla f(X)\rangle \geq \frac{(2r-\alpha)t^{\alpha-1}}{2}(f(X) -f^\star) + \frac{\mu(2r-\alpha)t^{\alpha-1}}{4}\|X-x^\star\|^2.
\end{aligned}
\end{equation}
Substituting the last display  into \eqref{eq:alpha_energy} with the awareness of $r \geq 3\alpha/2$ yields
\begin{multline}\nonumber
\dot{\mathcal{E}} \le -  \frac{(2\mu(2r-\alpha)t^2 - (\alpha-2)(2r-\alpha)^2)t^{\alpha-3}}{8}\|X-x^\star\|^2 + \frac{(\alpha-2)(2r-\alpha)t^{\alpha-2}}{8}\frac{\d\|X-x^\star\|^2}{\d t}.
\end{multline}
Hence, if $t \geq t_{\alpha} :=  \sqrt{(\alpha-2)(2r-\alpha)/(2\mu)}$, we obtain
\begin{equation}\nonumber
\dot{\mathcal{E}}(t) \le
\frac{(\alpha-2)(2r-\alpha)t^{\alpha-2}}{8}\frac{\d\|X-x^\star\|^2}{\d t}.
\end{equation}
Integrating the last inequality on the interval $(t_{\alpha}, t)$ gives
\begin{multline}\label{eq:bound_tilde_energy}
\mathcal{E}(t) \le \mathcal{E}(t_{\alpha}) + 
\frac{(\alpha-2)(2r-\alpha)t^{\alpha-2}}{8}\|X(t)-x^\star\|^2 - \frac{(\alpha-2)(2r-\alpha)t_{\alpha}^{\alpha-2}}{8}\|X(t_{\alpha})-x^\star\|^2 \\- \inverse{8}\int^t_{t_{\alpha}}(\alpha-2)^2(2r-\alpha)u^{\alpha-3}\|X(u)-x^\star\|^2\d u
\le \mathcal{E}(t_{\alpha}) + 
\frac{(\alpha-2)(2r-\alpha)t^{\alpha-2}}{8}\|X(t)-x^\star\|^2\\
\le \mathcal{E}(t_{\alpha}) + 
\frac{(\alpha-2)(2r-\alpha)t^{\alpha-2}}{4\mu}(f(X(t)) - f^\star).
\end{multline}

Making use of \eqref{eq:bound_tilde_energy}, we apply induction on
$\alpha$ to finish the proof. First, consider $2 < \alpha \le
4$. Applying Theorem \ref{thm:large_r}, from
\eqref{eq:bound_tilde_energy} we get that $\mathcal{E}(t)$ is upper
bounded by
\begin{equation}\label{eq:simple_bound_tilde_energy}
\mathcal{E}(t_{\alpha}) + 
\frac{(\alpha-2)(r-1)^2(2r-\alpha)\|x_0-x^\star\|^2}{8\mu t^{4-\alpha}} \le \mathcal{E}(t_{\alpha}) + 
\frac{(\alpha-2)(r-1)^2(2r-\alpha)\|x_0-x^\star\|^2}{8\mu t_{\alpha}^{4-\alpha}}.
\end{equation}
Then, we bound $\mathcal{E}(t_{\alpha})$ as follows.
\begin{multline}\label{eq:general_alpha_bound}
\mathcal{E}(t_{\alpha}) \le t_{\alpha}^{\alpha}(f(X(t_{\alpha})) - f^\star) + \frac{(2r-\alpha)^2t_{\alpha}^{\alpha-2}}{4}\Big{\|}\frac{2r-2}{2r-\alpha}X(t_{\alpha})+ \frac{2t_{\alpha}}{2r-\alpha}\dot X(t_{\alpha}) - \frac{2r-2}{2r-\alpha}x^\star\Big{\|}^2\\
 + \frac{(2r-\alpha)^2t_{\alpha}^{\alpha-2}}{4}\Big{\|}\frac{\alpha-2}{2r-\alpha}X(t_{\alpha})  - \frac{\alpha-2}{2r-\alpha}x^\star\Big{\|}^2 \\
\le (r-1)^2t_{\alpha}^{\alpha-2}\|x_0-x^\star\|^2 + \frac{(\alpha-2)^2(r-1)^2\|x_0-x^\star\|^2}{4\mu t_{\alpha}^{4-\alpha}},
\end{multline}
where in the second inequality we use the decreasing property of the energy functional defined for Theorem \ref{thm:large_r}. Combining \eqref{eq:simple_bound_tilde_energy} and \eqref{eq:general_alpha_bound}, we have
\begin{equation}\nonumber
\mathcal{E}(t) \le (r-1)^2t_{\alpha}^{\alpha-2}\|x_0-x^\star\|^2 + \frac{(\alpha-2)(r-1)^2(2r+\alpha-4)\|x_0-x^\star\|^2}{8\mu t_{\alpha}^{4-\alpha}} = O \Big{(} \frac{\|x_0-x^\star\|^2}{\mu^{\frac{\alpha-2}{2}}}  \Big{)}.
\end{equation}
For $t \ge t_{\alpha}$, it suffices to apply $f(X(t))-f^\star\le \mathcal{E}(t)/t^3$ to the last display. For $t < t_{\alpha}$, by Theorem \ref{thm:large_r}, $f(X(t)) - f^\star$ is upper bounded by
\begin{equation}\label{eq:alpha_general_alpha_small_t}
\begin{aligned}
\frac{(r-1)^2\|x_0 - x^\star\|^2}{2t^2} &\le \frac{(r-1)^2\mu^{\frac{\alpha-2}{2}}[(\alpha-2)(2r-\alpha)/(2\mu)]^{\frac{\alpha-2}{2}}}{2}\frac{\|x_0-x^\star\|^2}{\mu^{\frac{\alpha-2}{2}}t^{\alpha}} \\
&= O\Big{(}\frac{\|x_0-x^\star\|^2}{\mu^{\frac{\alpha-2}{2}}t^{\alpha}} \Big{)}.
\end{aligned}
\end{equation}

Next, suppose that the theorem is valid for some $\tilde{\alpha} > 2$. We show below that this theorem is still valid for $\alpha := \tilde{\alpha} + 1$ if still $r \geq 3\alpha/2$. By the assumption, \eqref{eq:bound_tilde_energy} further induces
\begin{equation}\nonumber
\mathcal{E}(t) \le \mathcal{E}(t_{\alpha}) +
\frac{(\alpha-2)(2r-\alpha)t^{\alpha-2}}{4\mu}\frac{\tilde{C}\|x_0 - x^\star\|^2}{\mu^{\frac{\tilde{\alpha}-2}{2}}t^{\tilde{\alpha}}} \le \mathcal{E}(t_{\alpha}) +
\frac{\tilde{C}(\alpha-2)(2r-\alpha)\|x_0 - x^\star\|^2}{4\mu^{\frac{\alpha-1}{2}}t_{\alpha}}
\end{equation}
for some constant $\tilde C$ only depending on $\tilde{\alpha}$ and $r$. This inequality with \eqref{eq:general_alpha_bound} implies
\begin{align*}\nonumber
\mathcal{E}(t)  &\le (r-1)^2t_{\alpha}^{\alpha-2}\|x_0-x^\star\|^2 + \frac{(\alpha-2)^2(r-1)^2\|x_0-x^\star\|^2}{4\mu t_{\alpha}^{4-\alpha}} + \frac{\tilde{C}(\alpha-2)(2r-\alpha)\|x_0 - x^\star\|^2}{4\mu^{\frac{\alpha-1}{2}}t_{\alpha}} \\
&= O\left( \|x_0-x^\star\|^2/\mu^{\frac{\alpha-2}{2}} \right),
\end{align*}
which verify the induction for $t \geq t_{\alpha}$. As for $t < t_{\alpha}$, the validity of the induction follows from Theorem \ref{thm:large_r}, similarly to \eqref{eq:alpha_general_alpha_small_t}. Thus, combining the base and induction steps, the proof is completed.
\end{proof}

It should be pointed out that the constant $C$ in the statement of Theorem \ref{thm:cube_rate} grows with the parameter $r$. Hence, simply increasing $r$ does not guarantee to give a better error bound. While it is desirable to expect a discrete analogy of Theorem \ref{thm:cube_rate}, i.e., $O(1/k^{\alpha})$ convergence rate for \eqref{eq:nesterov_general}, a complete proof can be notoriously complicated. That said, we mimic the proof of Theorem \ref{thm:cube_rate} for $\alpha = 3$ and succeed in obtaining a $O(1/k^3)$ convergence rate for the generalized Nesterov's schemes, as summarized in the theorem below.

\begin{theorem}\label{thm:dis_t3}
Suppose $f$ is written as $f = g + h$, where $g\in\mathcal{S}_{\mu,L}$ and $h$ is convex with possible extended value $\infty$. Then, the generalized Nesterov's scheme \eqref{eq:nesterov_general} with $r\geq 9/2$ and $s = 1/L$ satisfies
\begin{equation}\nonumber
f(x_k) - f^\star \le \frac{CL\|x_0-x^\star\|^2}{k^2}\frac{\sqrt{L/\mu}}{k},
\end{equation}
where $C$ only depends on $r$.
\end{theorem}

This theorem states that the discrete scheme \eqref{eq:nesterov_general} enjoys the error bound $O(1/k^3)$ without any knowledge of the condition number $L/\mu$. In particular, this bound is much better than that given in Theorem \ref{thm:large_r_dis} if $k \gg \sqrt{L/\mu}$. The strategy of the proof is fully inspired by that of Theorem \ref{thm:cube_rate}, though it is much more complicated and thus deferred to the Appendix. The relevant energy functional $\mathcal{E}(k)$ for this Theorem \ref{thm:dis_t3} is equal to
\begin{multline}\label{eq:dis_t3_energy}
\frac{s(2k+3r-5)(2k+2r-5)(4k+4r-9)}{16}(f(x_k)-f^\star)\\ + \frac{2k+3r-5}{16}\|2(k+r-1)y_k - (2k+1)x_k - (2r-3)x^\star\|^2.
\end{multline}

\subsection{Numerical Examples}
\label{sec:numerical-examples1}
We study six synthetic examples to compare \eqref{eq:nesterov_general} with the step sizes are fixed to be $1/L$, as illustrated in Figure \ref{fig:lasso_ps}. The error rates exhibits similar patterns for all $r$, namely, decreasing while suffering from local bumps. A smaller $r$ introduces less friction, thus allowing $x_k$ moves towards $x^\star$ faster in the beginning. However, when sufficiently close to $x^\star$, more friction is preferred in order to reduce overshoot. This point of view explains what we observe in these examples. That is, across these six examples, \eqref{eq:nesterov_general} with a smaller $r$ performs slightly better in the beginning, but a larger $r$ has advantage when $k$ is large. It is an interesting question how to choose a good $r$ for different problems in practice.

\begin{figure}[!h]
\centering
\begin{subfigure}[b]{0.48\textwidth}
\centering
\includegraphics[width=\textwidth,height=1.7in]{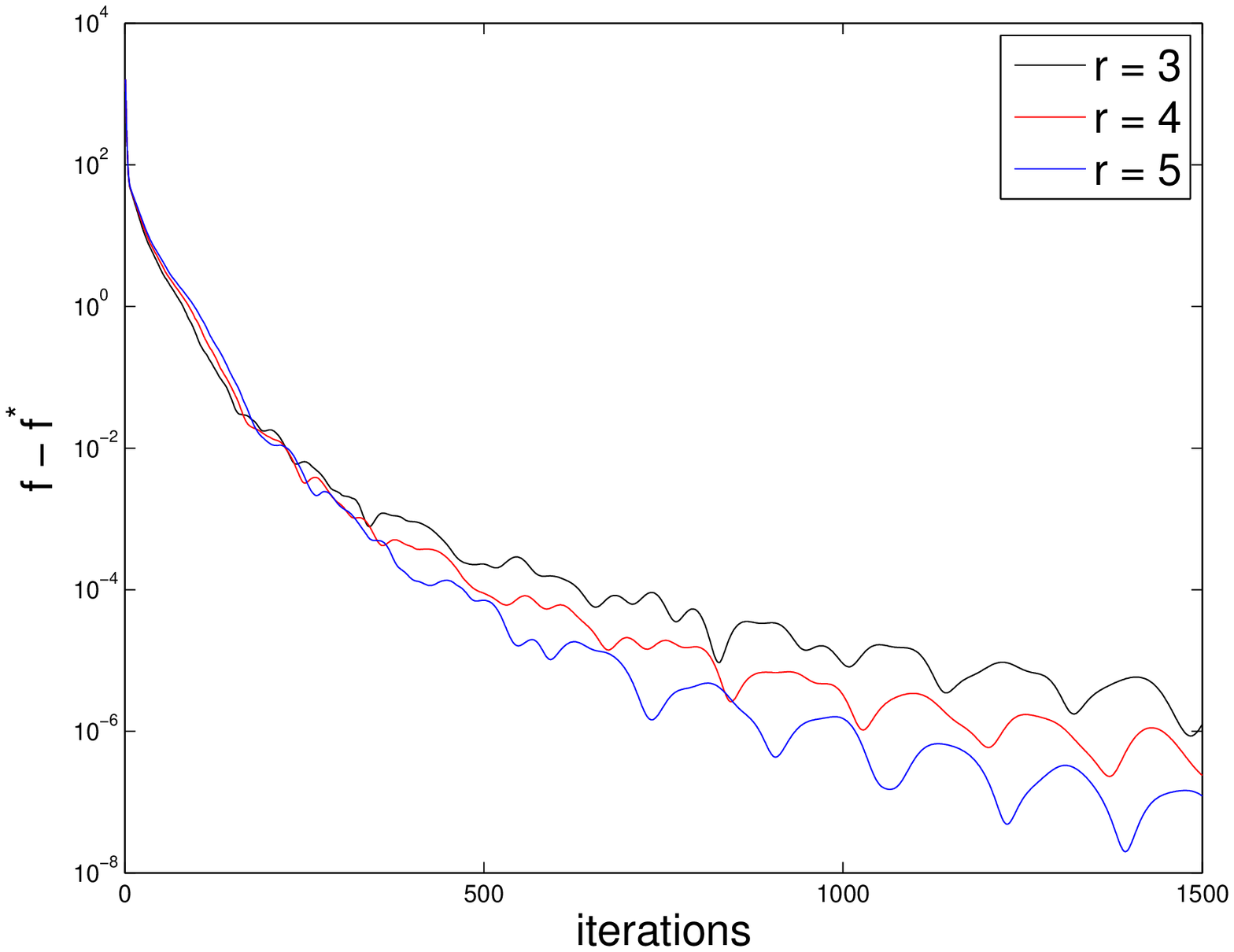}
\caption{Lasso with fat design.}
\label{fig:larger_r_lasso}
\end{subfigure}
\hfill
\begin{subfigure}[b]{0.48\textwidth}
\centering
\includegraphics[width=\textwidth,height=1.7in]{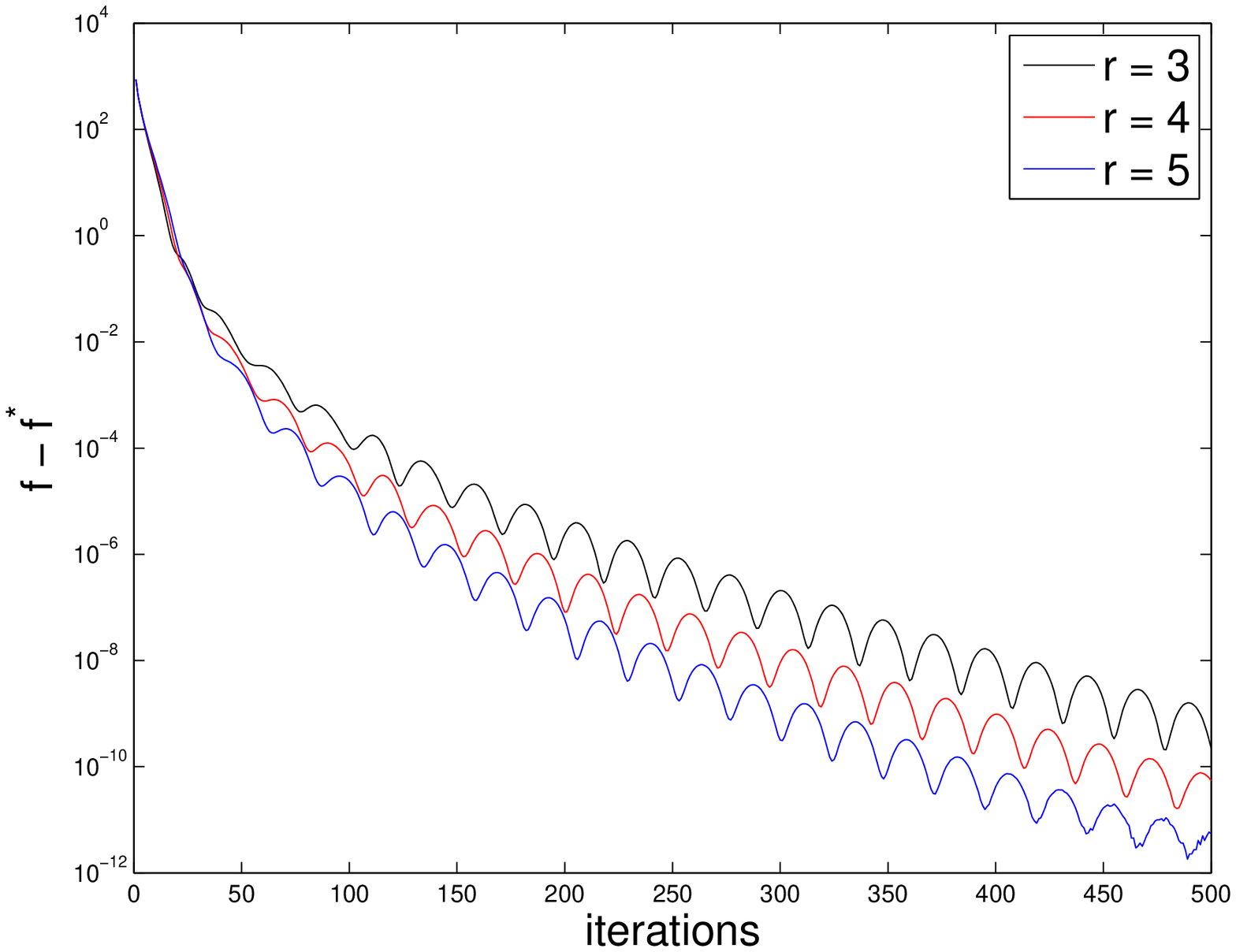}
\caption{Lasso with square design.}
\label{fig:larger_r_lasso2}
\end{subfigure}
\begin{subfigure}[b]{0.48\textwidth}
\centering
\includegraphics[width=\textwidth,height=1.7in]{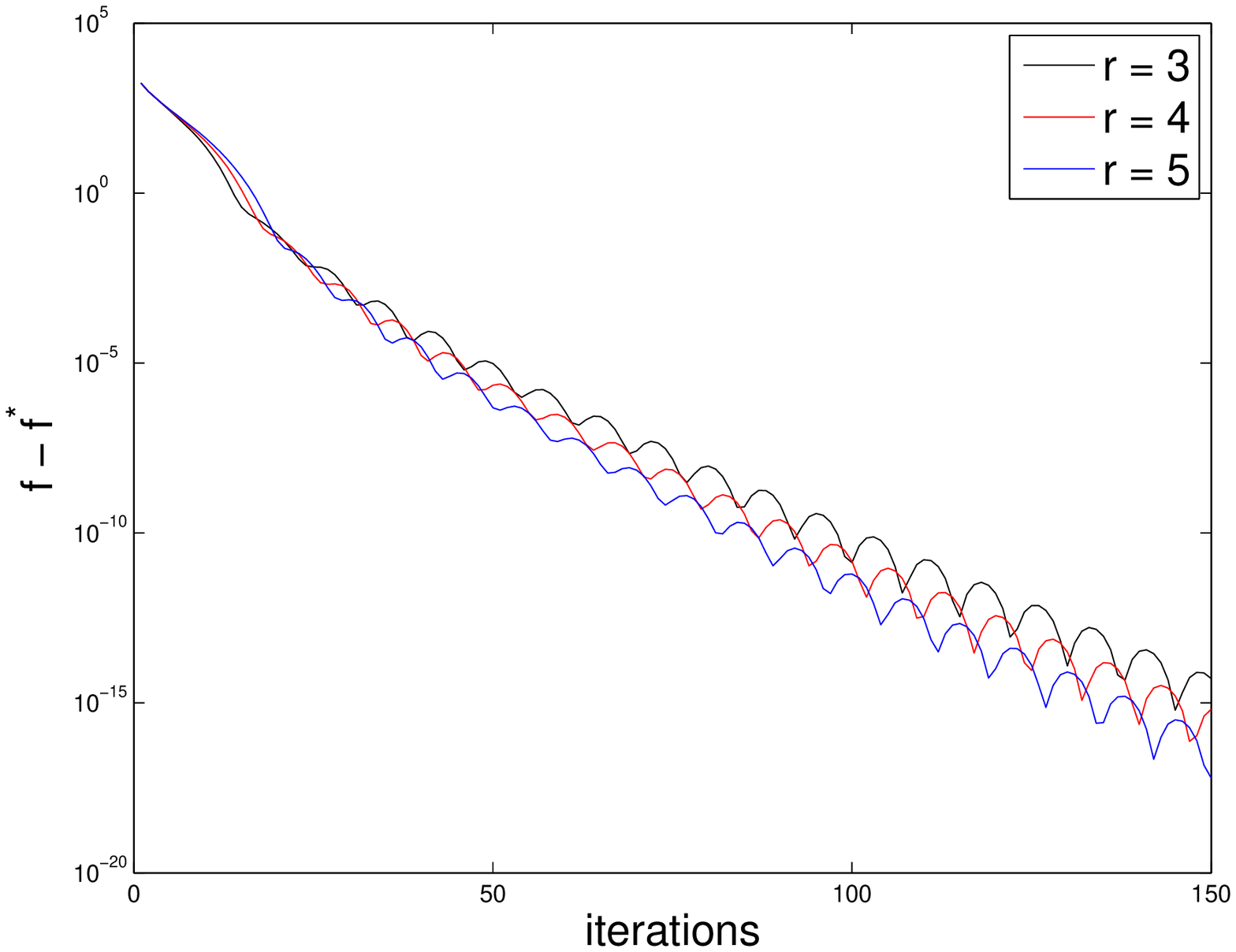}
\caption{NLS with fat design.}
\label{fig:larger_r_nonneg}
\end{subfigure}
\hfill
\begin{subfigure}[b]{0.48\textwidth}
\centering
\includegraphics[width=\textwidth,height=1.7in]{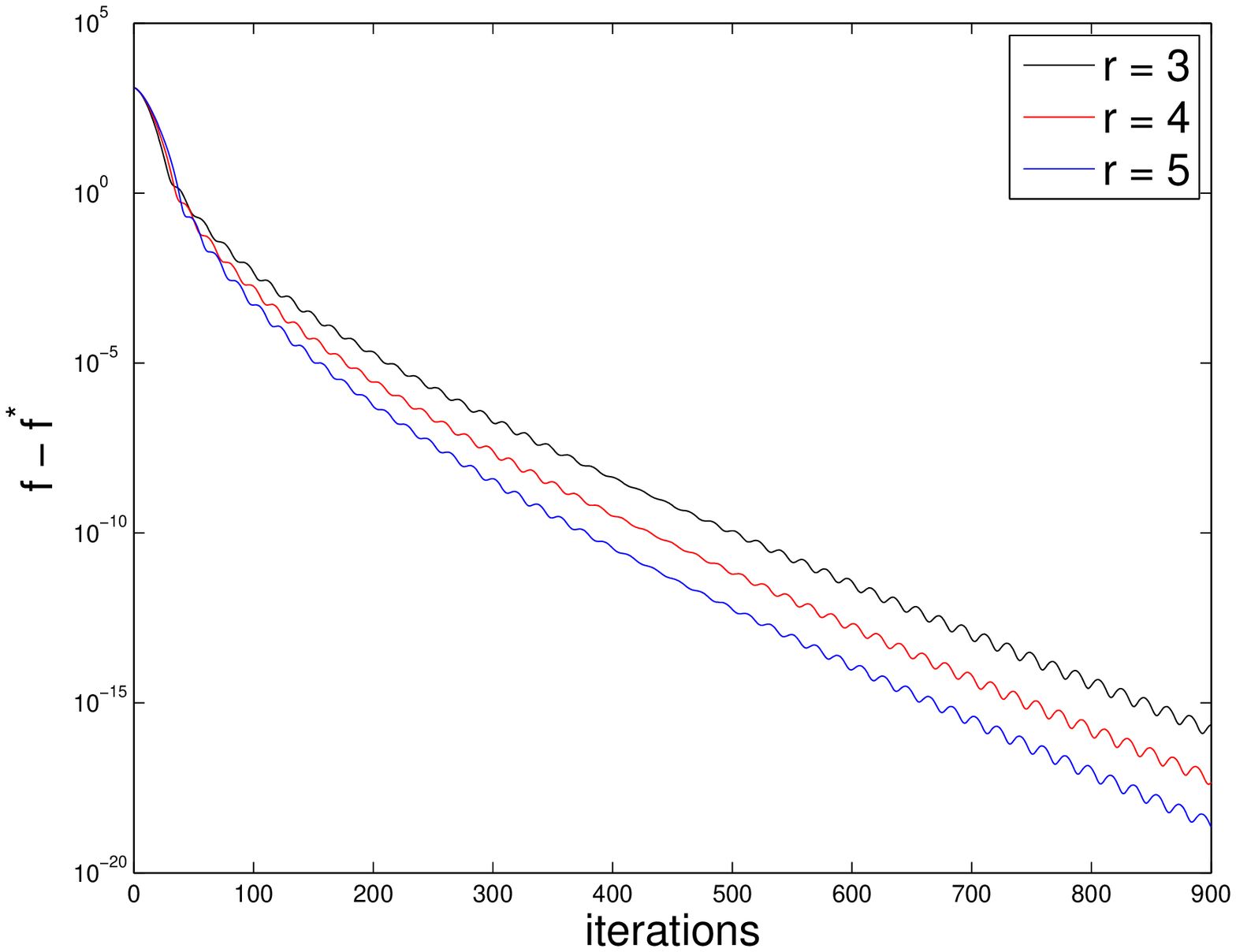}
\caption{NLS with square design.}
\label{fig:larger_r_nonneg2}
\end{subfigure}
\\\vspace{0.5em}
\begin{subfigure}[b]{0.48\textwidth}
\centering
\includegraphics[width=\textwidth,height=1.7in]{larger_r_logistic}
\caption{Logistic regression.}
\label{fig:larger_r_log}
\end{subfigure}
\hfill
\begin{subfigure}[b]{0.48\textwidth}
\centering
\includegraphics[width=\textwidth,height=1.7in]{larger_r_reglogistic}
\caption{$\ell_1$-regularized logistic regression.}
\label{fig:larger_r_reglog}
\end{subfigure}
\caption{Comparisons of generalized Nesterov's schemes with different $r$.}
\label{fig:lasso_ps}
\end{figure}

\noindent{\bf Lasso with fat design.}~ Minimize $f(x) = \half\|Ax-b\|^2 + \lambda\|x\|_1$, in which $A$ a $100\times 500$ random matrix with \iid standard Gaussian $\mathcal{N}(0, 1)$ entries, $b$ generated independently has \iid $\mathcal{N}(0, 25)$ entries, and the penalty $\lambda = 4$. The plot is Figure \ref{fig:larger_r_lasso}.

\vspace{0.4em}
\noindent{\bf Lasso with square design.}~ Minimize $f(x) = \half\|Ax-b\|^2 + \lambda\|x\|_1$, where $A$ a $500\times 500$ random matrix with \iid standard Gaussian entries, $b$ generated independently has \iid $\mathcal{N}(0, 9)$ entries, and the penalty $\lambda = 4$. The plot is Figure \ref{fig:larger_r_lasso2}.

\vspace{0.4em}
\noindent{\bf Nonnegative least squares (NLS) with fat design.}~ Minimize $f(x) = \|Ax-b\|^2$ subject to $x\succeq 0$, with the same design $A$ and $b$ as in Figure \ref{fig:larger_r_lasso}. The plot is Figure \ref{fig:larger_r_nonneg}.

\vspace{0.4em}
\noindent{\bf Nonnegative least squares with sparse design.}~ Minimize $f(x) = \|Ax-b\|^2$ subject to $x\succeq 0$, in which $A$ is a $1000\times 10000$ sparse matrix with nonzero probability $10\%$ for each entry and $b$ is given as $b = Ax^0 + \mathcal{N}(0, I_{1000})$. The nonzero entries of $A$ are independently Gaussian distributed before column normalization, and $x^0$ has 100 nonzero entries that are all equal to 4. The plot is Figure \ref{fig:larger_r_nonneg2}.

\vspace{0.4em}
\noindent{\bf Logistic regression.}~ Minimize $\sum_{i=1}^n - y_i a_i^T x + \log(1 + \mathrm{e}^{a_i^T x})$, in which $A = (a_1, \ldots, a_n)^T$ is a $500 \times 100$ matrix with \iid $\mathcal{N}(0,1)$ entries. The labels $y_i \in \{0, 1\}$ are generated by the logistic model: $\mathbb P(Y_i = 1) = 1/(1 + \mathrm{e}^{-a_i^T x^0})$, where $x^0$ is a realization of \iid $\mathcal{N}(0, 1/100)$. The plot is Figure \ref{fig:larger_r_log}.

\vspace{0.4em}
\noindent{\bf $\ell_1$-regularized logistic regression.}~ Minimize $\sum_{i=1}^n - y_i a_i^T x + \log(1 + \mathrm{e}^{a_i^T x}) + \lambda \|x\|_1$, in which $A = (a_1, \ldots, a_n)^T$ is a $200 \times 1000$ matrix with \iid $\mathcal{N}(0,1)$ entries and $\lambda = 5$. The labels $y_i$ are generated similarly as in the previous example, except for the ground truth $x^0$ here having 10 nonzero components given as \iid $\mathcal{N}(0, 225)$. The plot is Figure \ref{fig:larger_r_reglog}.


\section{Restarting}
\label{sec:accelerate}
The example discussed in Section \ref{sec:low-friction} demonstrates that Nesterov's scheme and its generalizations \eqref{eq:nesterov_general} are not capable of fully exploiting strong convexity. That is, this example suggests evidence that $O(1/\mathtt{poly}(k))$ is the best rate achievable under strong convexity. In contrast, the vanilla gradient method achieves linear convergence $O((1-\mu/L)^k)$. This drawback results from too much momentum introduced when the objective function is strongly convex. The derivative of a strongly convex function is generally more reliable than that of non-strongly convex functions. In the language of ODEs, at later stage a too small $3/t$ in \eqref{key} leads to a lack of friction, resulting in unnecessary overshoot along the trajectory. Incorporating the optimal momentum coefficient $\frac{\sqrt{L} - \sqrt{\mu}}{\sqrt{L} + \sqrt{\mu}}$ (This is less than $(k-1)/(k+2)$ when $k$ is large), Nesterov's scheme has convergence rate of $O((1-\sqrt{\mu/L})^k)$ \citep{nesterov-book}, which, however, requires knowledge of the condition number $\mu/L$. While it is relatively easy to bound the Lipschitz constant $L$ by the use of backtracking, estimating the strong convexity parameter $\mu$, if not impossible, is very challenging. 

Among many approaches to gain acceleration via adaptively estimating $\mu/L$ \citep[see][]{nesterov_compo}, \citet{restart} proposes a procedure termed as gradient restarting for Nesterov's scheme in which \eqref{eq:nesterov-scheme} is restarted with $x_0 = y_0 :=x_k$ whenever $f(x_{k+1}) > f(x_k)$.  In the language of ODEs, this restarting essentially keeps $\langle \nabla f,\dot{X} \rangle$ negative, and resets $3/t$ each time to prevent this coefficient from steadily decreasing along the trajectory. Although it has been empirically observed that this method significantly boosts convergence, there is no general theory characterizing the convergence rate.

In this section, we propose a new restarting scheme we call the speed restarting scheme. The underlying motivation is to maintain a
relatively high velocity $\dot X$ along the trajectory, similar in spirit to the gradient restarting. Specifically, our main result, Theorem \ref{exp}, ensures linear convergence of the continuous version of the speed restarting. More generally, our contribution here is merely to provide a framework for analyzing restarting schemes rather than competing with other schemes; it is beyond the scope of this paper to get optimal constants in these results. Throughout this section, we assume $f\in \mathcal{S}_{\mu,L}$ for some $0 < \mu \le L$. Recall that
function $f\in \mathcal{S}_{\mu,L}$ if $f\in \mathcal{F}_{L}$ and
$f(x) - \mu \norm{x}^2/2$ is convex.

\subsection{A New Restarting Scheme}
\label{sec:speed-restarting}
We first define the speed restarting time. For the ODE \eqref{key}, we call
\begin{equation}\nonumber
T =  T(x_0; f) = \sup \left\{t > 0: \forall u\in (0, t), ~ \frac{\d\norm{\dot X(u)}^2}{\d u} > 0 \right\}
\end{equation}
the speed restarting time. In words, $T$ is the first time the velocity $\|\dot X\|$ decreases. Back to the discrete scheme, it is the first time when we observe $\|x_{k+1} - x_k\| < \|x_k - x_{k-1}\|$. This definition itself does not directly imply that $0 < T < \infty$, which is proven later in Lemmas \ref{lm:small} and \ref{lm:stopping}. Indeed, $f(X(t))$ is a decreasing function before time $T$; for $t\le T$,
\begin{equation}\nonumber
\frac{\d f(X(t))}{\d t} =\langle \nabla f(X),\dot X\rangle = -\frac{3}{t}\norm{\dot X}^2 - \frac{1}{2}\frac{\d\norm{\dot X}^2}{\mathrm{d} t} \le 0.
\end{equation}
The speed restarted ODE is thus 
\begin{equation}\label{eq:ode_restart}
\ddot X(t) + \frac{3}{\tsr}\dot X(t) + \nabla f(X(t)) = 0,
\end{equation}
where $\tsr$ is set to zero whenever $\langle \dot X, \ddot X\rangle = 0$ and between two consecutive restarts, $\tsr$ grows just as $t$. That is, $\tsr = t - \tau$, where $\tau$ is the latest restart time. In particular, $\tsr = 0$ at $t = 0$. Letting $\xsr$ be the solution to \eqref{eq:ode_restart}, we have the following observations.
\begin{itemize}
\item 
$\xsr(t)$ is continuous for $t \geq 0$, with $\xsr(0) = x_0$;
\item
$\xsr(t)$ satisfies \eqref{key} for $0 < t < T_1 := T(x_0; f)$.
\item
Recursively define $T_{i+1} = T\left( \xsr\left(\sum_{j=1}^i T_j\right); f\right)$ for $i \ge 1$, and $\widetilde X(t) := \xsr\left(\sum_{j=1}^i T_j + t \right)$ satisfies  the ODE \eqref{key}, with $\widetilde X(0) = \xsr \left( \sum_{j=1}^i T_j \right)$, for $0 < t < T_{i+1}$.
\end{itemize}

The theorem below guarantees linear convergence of $\xsr$. This is a new result in the literature \citep{restart,restarting_gatech}. The proof of Theorem \ref{exp} is based on Lemmas \ref{lm:decayfast} and \ref{lm:small}, where the first guarantees the rate $f(\xsr) - f^\star$ decays by a constant factor for each restarting, and the second confirms that restartings are adequate. In these lemmas we all make a convention that the uninteresting case $x_0 = x^\star$ is excluded.

\begin{theorem}\label{exp}
There exist positive constants $c_1$ and $c_2$, which only depend on the condition number $L/\mu$, such that for any $f\in\mathcal{S}_{\mu, L}$, we have
\begin{equation}\nonumber
f(\xsr(t)) - f^\star  \le \frac{c_1L\norm{x_0 - x^\star}^2}{2}\e{-c_2t\sqrt{L}}.
\end{equation}
\end{theorem}

Before turning to the proof, we make a remark that this linear convergence of $\xsr$ remains to hold for the generalized ODE \eqref{eq:ode_r} with $r > 3$. Only minor modifications in the proof below are needed, such as replacing $u^3$ by $u^r$ in the definition of $I(t)$ in Lemma \ref{lm:stopping}.

\subsection{Proof of Linear Convergence}
\label{sec:proof-line-conv}
First, we collect some useful estimates. Denote by $M(t)$ the supremum of $\norm{\dot X(u)}/{u}$ over $u\in (0, t]$ and let
\[ I(t) := \int^t_0 u^3(\nabla f(X(u)) - \nabla f(x_0)) \d u. \]
It is guaranteed that $M$ defined above is finite, for example, see the proof of Lemma \ref{lm:arzela}. The definition of $M$ gives a bound on the gradient of $f$,
\begin{equation}\nonumber
\norm{\nabla f(X(t)) - \nabla f(x_0)} \le  L\Big{\|}\int_{0}^t \dot X(u) \d u \Big{\|}  \le L\int_{0}^t u\frac{\norm{\dot X(u)}}{u} \d u \le \frac{L M(t)t^2}{2}.
\end{equation}
Hence, it is easy to see that $I$ can also be bounded via $M$,
\begin{equation}\nonumber
\|I(t)\| \le \int_0^t u^3\|\nabla f(X(u))-\nabla f(x_0)\|\d u\le \int_0^t \frac{L M(u)u^5}{2}\d u\le \frac{LM(t)t^6}{12}.
\end{equation}
To fully facilitate these estimates, we need the following lemma that gives an upper bound of $M$, whose proof is deferred to the appendix.

\begin{lemma}\label{lm:m_upper}
For $t < \sqrt{12/L}$, we have
\begin{equation}\nonumber
M(t) \le \frac{\norm{\nabla f(x_0)}}{4(1 - Lt^2/12)}.
\end{equation}
\end{lemma}

Next we give a lemma which claims that the objective function decays by a constant through each speed restarting.
\begin{lemma}\label{lm:decayfast}
There is a universal constant $C > 0$ such that
\begin{equation}\nonumber
f(X(T)) - f^\star \le \left( 1 - \frac{C\mu}{L} \right) (f(x_0) - f^\star).
\end{equation}

\end{lemma}
\begin{proof}
By Lemma \ref{lm:m_upper}, for $t < \sqrt{12/L}$ we have
\begin{equation}\nonumber
\left\| \dot X(t) + \frac{t}{4}\nabla f(x_0) \right\| = \frac{1}{t^3}\|I(t)\|\le \frac{LM(t)t^3}{12} \le \frac{L\norm{\nabla f(x_0)}t^3}{48(1 - Lt^2/12)},
\end{equation}
which yields
\begin{equation}\label{eq:close}
0 \le \frac{t}{4}\norm{\nabla f(x_0)} -  \frac{L\norm{\nabla f(x_0)}t^3}{48(1 - Lt^2/12)} \le \|\dot X(t)\| \le \frac{t}{4}\norm{\nabla f(x_0)} +  \frac{L\norm{\nabla f(x_0)}t^3}{48(1 - Lt^2/12)}.
\end{equation}
Hence, for $0 < t < 4/(5\sqrt{L})$ we get
\begin{multline}\nonumber
\frac{\d f(X)}{\d t} = -\frac{3}{t}\norm{\dot X}^2 - \frac{1}{2}\frac{\d}{\d t}\norm{\dot X}^2 \le -\frac{3}{t}\norm{\dot X}^2   \\
\le -\frac{3}{t} \left(\frac{t}{4}\norm{\nabla f(x_0)} -  \frac{L\norm{\nabla f(x_0)}t^3}{48(1 - Lt^2/12)}\right)^2 \le - C_1t\norm{\nabla f(x_0)}^2,
\end{multline}
where $C_1 > 0$ is an absolute constant and the second inequality follows from Lemma \ref{lm:stopping} in the appendix. Consequently,
\[
f\left(X(4/(5\sqrt{L}))\right) - f(x_0) \le \int_{0}^{\frac{4}{5\sqrt{L}}}-C_1u\norm{\nabla f(x_0)}^2\d u \le -\frac{C\mu}{L}(f(x_0)-f^\star),
\]
where $C = 16C_1/25$ and in the last inequality we use the $\mu$-strong convexity of $f$. Thus we have
\[
f\left( X \left(\frac{4}{5\sqrt{L}}\right)\right) - f^\star \le \left( 1 -\frac{C\mu}{L}\right)(f(x_0)-f^\star).
\]
To complete the proof, note that $f(X(T)) \le f(X(4/(5\sqrt{L})))$ by Lemma \ref{lm:stopping}.

\end{proof}

With each restarting reducing the error $f - f^\star$ by a constant a factor, we still need the following lemma to ensure sufficiently many restartings.
\begin{lemma}\label{lm:small}
There is a universal constant $\tilde C$ such that
\[
T \le \frac{4\exp\left(\tilde C L/\mu\right)}{5\sqrt{L}}.
\]
\end{lemma}
\begin{proof}
For $4/(5\sqrt{L}) \le t \le T$, we have $\frac{\d f(X)}{\d t} \le -\frac{3}{t}\norm{\dot X(t)}^2 \le -\frac{3}{t}\norm{\dot X(4/(5\sqrt{L}))}^2$, which implies
\[
f(X(T)) - f(x_0) \le -\int_{\frac{4}{5\sqrt{L}}}^T\frac{3}{t}\norm{\dot X(4/(5\sqrt{L}))}^2 \d t
= -3\norm{\dot X(4/(5\sqrt{L}))}^2\log\frac{5T\sqrt{L}}{4}.
\]
Hence, we get an upper bound for $T$,
\begin{equation}\nonumber
T \le \frac{4}{5\sqrt{L}}\exp\Big{(}\frac{f(x_0) - f(X(T))}{3\norm{\dot X(4/(5\sqrt{L}))}^2} \Big{)}\le \frac{4}{5\sqrt{L}}\exp\Big{(}\frac{f(x_0) - f^\star}{3\norm{\dot X(4/(5\sqrt{L}))}^2} \Big{)}.\\
\end{equation}
Plugging $t = 4/(5\sqrt{L})$ into \eqref{eq:close} gives $\norm{\dot X(4/(5\sqrt{L}))} \geq \frac{C_1}{\sqrt{L}}\norm{\nabla f(x_0)}$ for some universal constant $C_1 > 0$. Hence, from the last display we get
\begin{equation*}
T \le\frac{4}{5\sqrt{L}}\exp\left( \frac{L(f(x_0)-f^\star)}{3C_1^2\norm{\nabla f(x_0)}^2} \right) \le \frac{4}{5\sqrt{L}}\exp\frac{L}{6C_1^2\mu}.
\end{equation*}
\end{proof}


Now, we are ready to prove Theorem \ref{exp} by applying Lemmas \ref{lm:decayfast} and \ref{lm:small}.
\begin{proof}
Note that Lemma \ref{lm:small} asserts, by time $t$ at least $m := \lfloor 5t\sqrt{L}\e{-\tilde C L/\mu}/4\rfloor$
restartings have occurred for $\xsr$. Hence, recursively applying Lemma \ref{lm:decayfast}, we have
\begin{align*}
f(\xsr(t)) - f^\star &\le f \left( \xsr(T_1 + \cdots + T_m) \right) - f^\star \\
&\le (1 - C\mu/L) \left(f \left( \xsr(T_1 + \cdots + T_{m-1}) \right) - f^\star \right)\\
& \le \cdots \le \cdots \\
&\le (1 - C\mu/L)^m (f(x_0) - f^\star)  \le \e{- C\mu m/L} (f(x_0) - f^\star) \\ 
&\le c_1\e{-c_2t\sqrt{L}}(f(x_0) - f^\star) \le \frac{c_1L\|x_0 - x^\star\|^2}{2}\e{-c_2t\sqrt{L}},
\end{align*}
where $c_1 = \exp(C\mu/L)$ and $c_2 = 5C\mu e^{-\tilde C \mu/L}/(4L)$.
\end{proof}

In closing, we remark that we believe that estimate in Lemma \ref{lm:decayfast} is tight, while not for Lemma \ref{lm:small}. Thus we conjecture that for a large class of $f\in\mathcal{S}_{\mu,L}$, if not all, $T = O(\sqrt{L}/\mu)$. If this is true, the exponent constant $c_2$ in Theorem \ref{exp} can be significantly improved.

\subsection{Numerical Examples}\label{sec:numerical-examples}
Below we present a discrete analog to the restarted scheme. There,
$k_{\min}$ is introduced to avoid having consecutive restarts that are
too close. To compare the performance of the restarted scheme with the
original \eqref{eq:nesterov-scheme}, we conduct four simulation
studies, including both smooth and non-smooth objective
functions. Note that the computational costs of the restarted and
non-restarted schemes are the same.

\begin{algorithm}[H]
\caption{Speed Restarting Nesterov's Scheme}
\label{alg:restart}
\begin{algorithmic}
\STATE \textbf{input:} $x_0\in\R^n, y_0 = x_0, x_{-1} = x_0, 0 < s \le 1/L, k_{\max}\in\mathbb{N}^+$ and $k_{\min}\in\mathbb{N}^+$
\STATE $j \gets 1$
\FOR{$k = 1$ to $k_{\max}$}
\STATE $x_k \gets \mbox{argmin}_x(\frac{1}{2s}\|x - y_{k-1} + s\nabla g(y_{k-1})\|^2 + h(x))$
\STATE $y_k \gets x_k + \frac{j-1}{j+2}(x_k - x_{k-1})$
\IF{ $\|x_k - x_{k-1}\| < \|x_{k-1} - x_{k-2}\|$ {\bf{and}} $j \ge k_{\min}$}
\STATE $j \gets 1$
\ELSE
\STATE $j \gets j + 1$
\ENDIF
\ENDFOR
\end{algorithmic}
\end{algorithm}

\noindent{\bf Quadratic.}~ $f(x) = \half x^TAx + b^Tx$ is a strongly
convex function, in which $A$ is a $500\times 500$ random positive
definite matrix and $b$ a random vector. The eigenvalues of $A$ are
between $0.001$ and $1$. The vector $b$ is generated as \iid Gaussian
random variables with mean 0 and variance 25.

\vspace{0.4em}
\noindent{\bf Log-sum-exp.}~
\begin{equation}\nonumber
f(x) = \rho\log\Big{[}\sum_{i=1}^m\exp((a_i^Tx - b_i)/\rho)\Big{]},
\end{equation}
where $n = 50, m = 200, \rho = 20$. The matrix $A = (a_{ij})$ is a
random matrix with \iid standard Gaussian entries, and $b = (b_i)$ has \iid Gaussian entries with mean $0$ and variance $2$. This function
is not strongly convex.

\vspace{0.4em}
\noindent{\bf Matrix completion.}~ $f(X) = \half \|X_{\mathrm{obs}} -
M_{\mathrm{obs}}\|_F^2 + \lambda\|X\|_*$, in which the ground truth
$M$ is a rank-5 random matrix of size $300\times 300$. The
regularization parameter is set to $\lambda = 0.05$. The 5 singular
values of $M$ are $1, \ldots, 5$. The observed set is independently
sampled among the $300\times 300$ entries so that 10\% of the entries
are actually observed.

\begin{figure}[!htp]
\centering
\begin{subfigure}[b]{0.48\textwidth}
\includegraphics[width=\textwidth, height=1.7in]{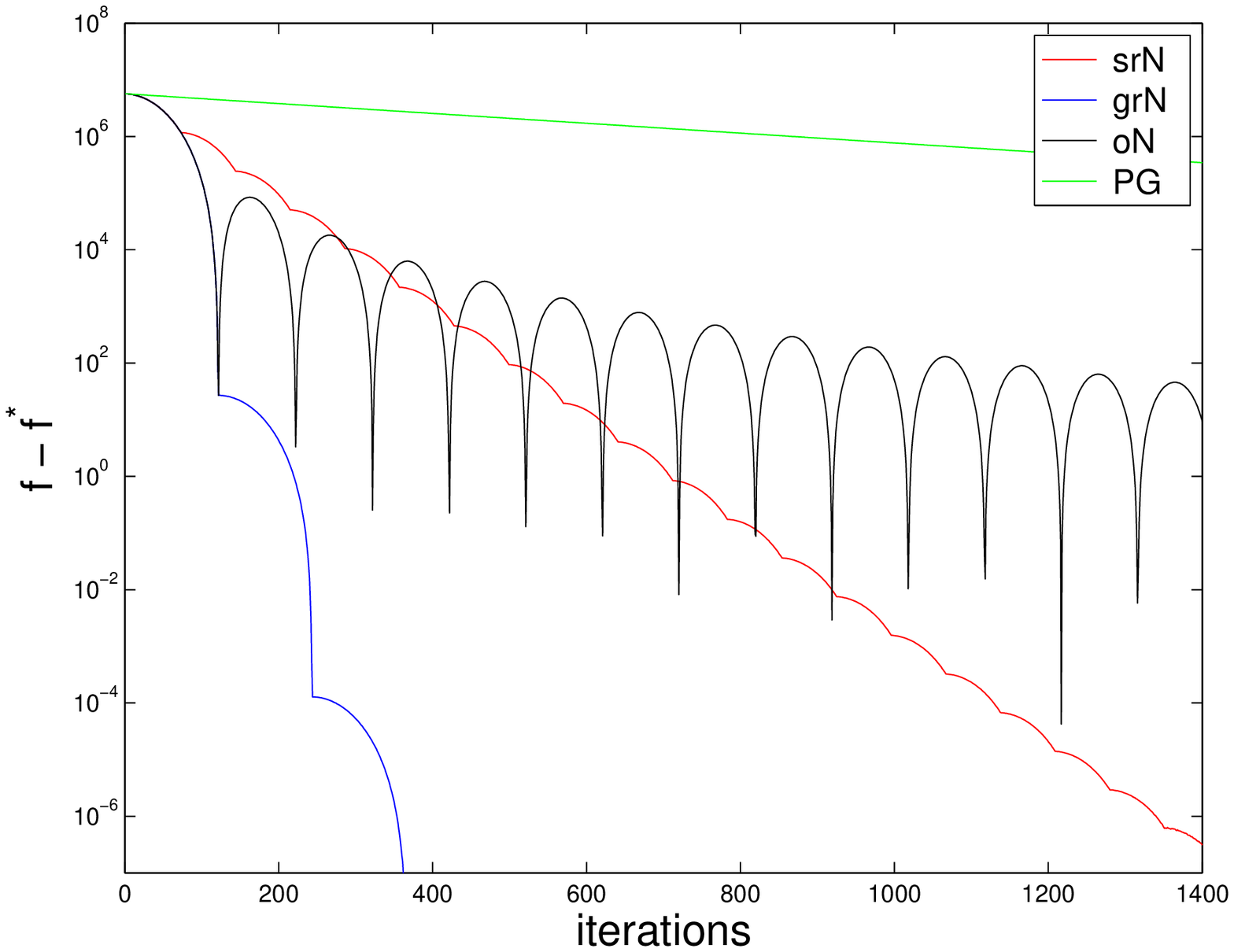}
\caption{$\mbox{min}~\half x^TAx + bx$.}
\label{fig:compare_restarting_quad}
\end{subfigure}
\hfill
\begin{subfigure}[b]{0.48\textwidth}
\includegraphics[width=\textwidth, height=1.7in]{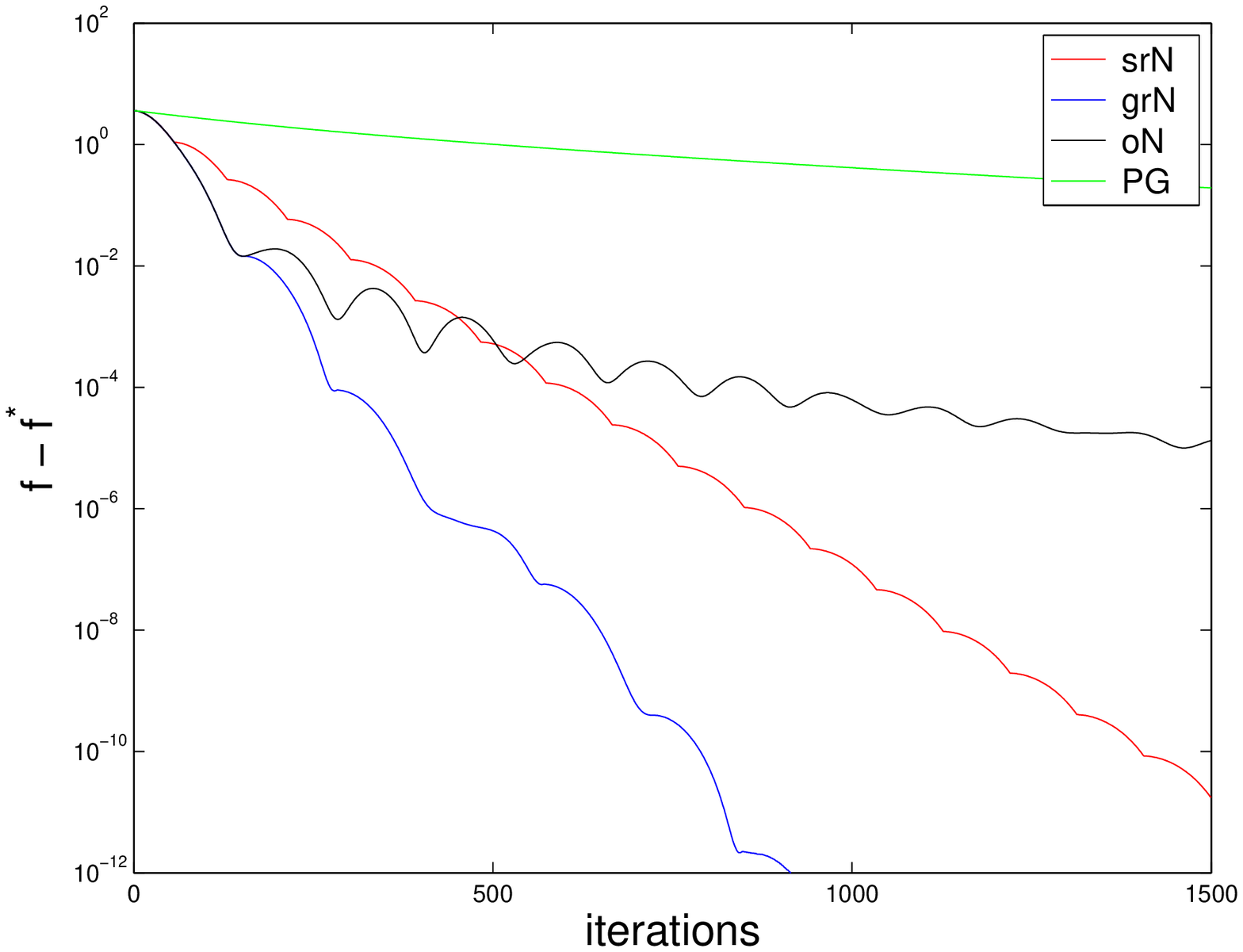}
\caption{$\mbox{min}~\rho\log(\sum_{i=1}^m\exp((a_i^Tx - b_i)/\rho))$.}
\label{fig:compare_restarting_expo}
\end{subfigure}
\begin{subfigure}[b]{0.48\textwidth}
\includegraphics[width=\textwidth, height=1.7in]{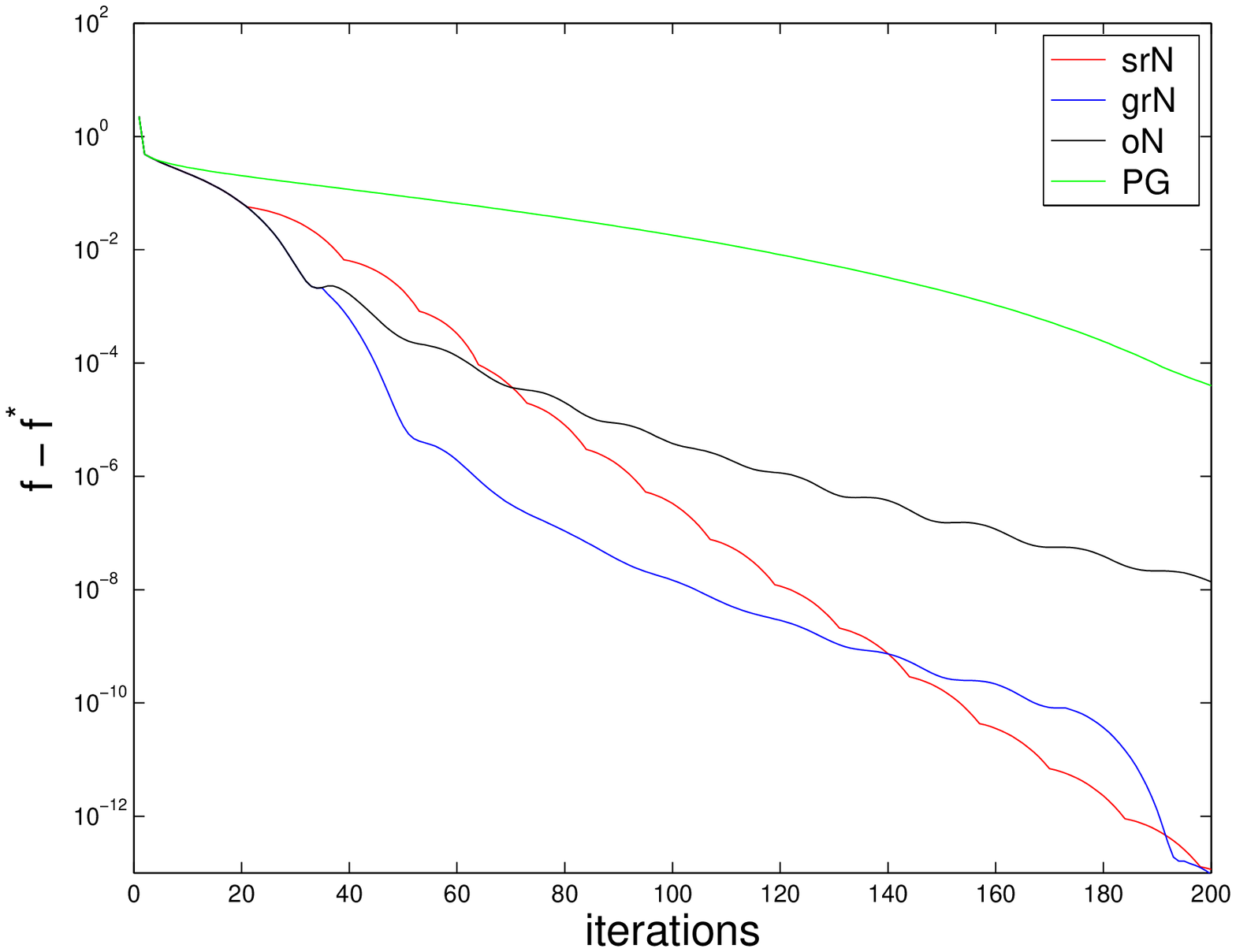}
\caption{$\mbox{min}~\half \|X_{\mathrm{obs}} - M_{\mathrm{obs}}\|_{\mathrm{F}}^2 + \lambda\|X\|_*$.}
\label{fig:compare_restarting_matrix}
\end{subfigure}
\hfill
\begin{subfigure}[b]{0.48\textwidth}
\includegraphics[width=\textwidth, height=1.7in]{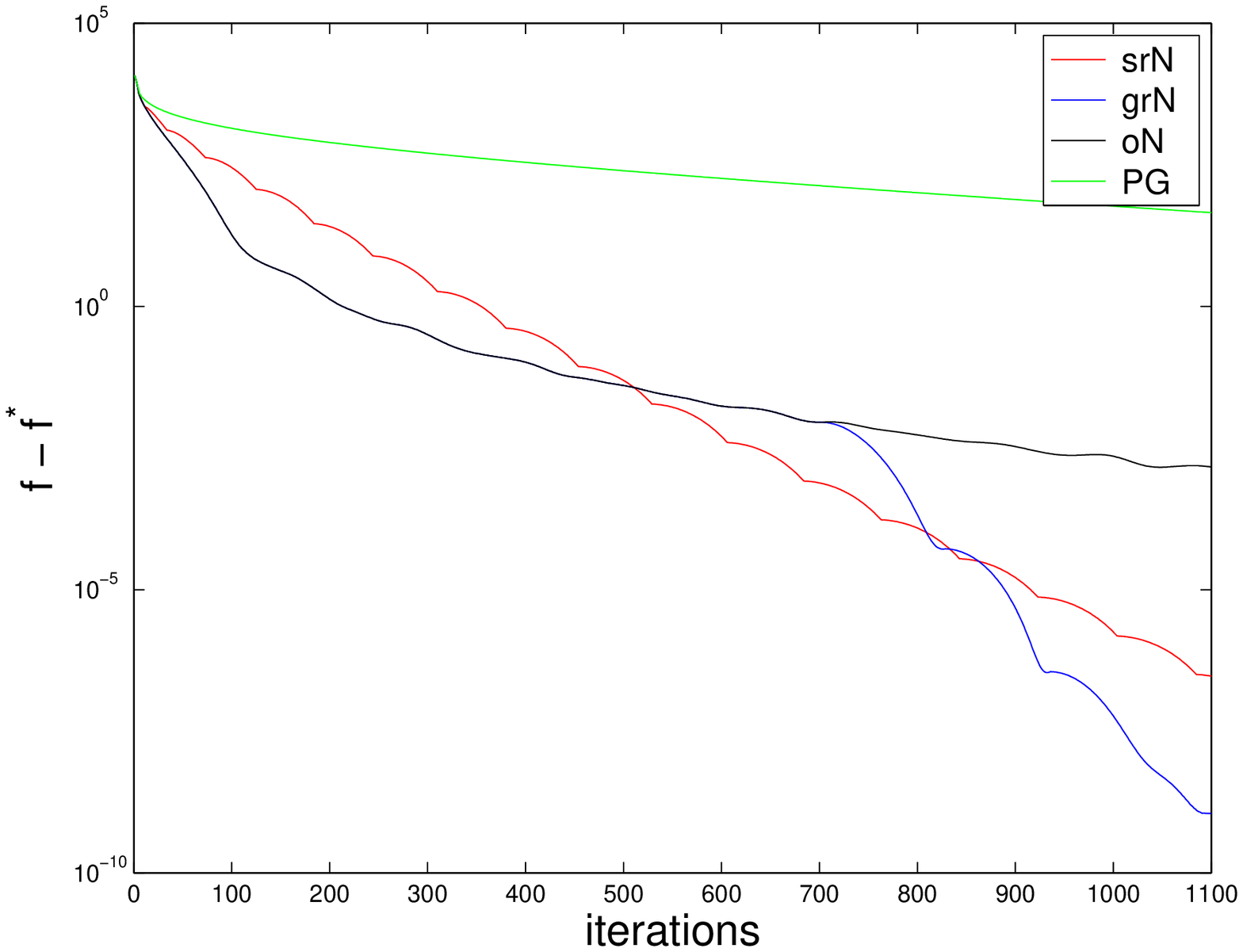}
\caption{$\mbox{min}~\half\|Ax-b\|^2\quad \mbox{s.t.}~ \|x\|_1\le C$.}
\label{fig:compare_restarting_lasso}
\end{subfigure}
\\\vspace{0.6em}
\begin{subfigure}[b]{0.48\textwidth}
\includegraphics[width=\textwidth, height=1.7in]{original_speed_slope.eps}
\caption{$\mbox{min}~\half \|Ax - b\|^2 + \sum_{i=1}^p \lambda_i |x|_{(i)}$.}
\label{fig:compare_restarting_slope}
\end{subfigure}
\hfill
\begin{subfigure}[b]{0.48\textwidth}
\includegraphics[width=\textwidth, height=1.7in]{original_speed_duallasso.eps}
\caption{$\mbox{min}~\half \|Ax - b\|^2 + \lambda \|x\|_1$.}
\label{fig:compare_restarting_duallasso}
\end{subfigure}
\\\vspace{0.6em}
\begin{subfigure}[b]{0.48\textwidth}
\includegraphics[width=\textwidth, height=1.7in]{original_speed_reglogistic.eps}
\caption{$\mbox{min}~ \sum_{i=1}^n - y_i a_i^T x + \log(1 + \mathrm{e}^{a_i^T x}) + \lambda \|x\|_1$.}
\label{fig:compare_restarting_reglogistic}
\end{subfigure}
\hfill
\begin{subfigure}[b]{0.48\textwidth}
\includegraphics[width=\textwidth, height=1.7in]{original_speed_logistic.eps}
\caption{$\mbox{min}~ \sum_{i=1}^n - y_i a_i^T x + \log(1 + \mathrm{e}^{a_i^T x})$.}
\label{fig:compare_restarting_logistic}
\end{subfigure}
\caption{Numerical performance of speed restarting (srN), gradient restarting (grN), the original Nesterov's scheme (oN) and the proximal gradient (PG).}
\label{fig:compare_restarting}
\end{figure}

\noindent{\bf Lasso in $\ell_1$--constrained form with large sparse
  design.}~ $f(x) = \half\|Ax-b\|^2\quad \mbox{s.t.}~ \|x\|_1\le \delta$,
where $A$ is a $5000\times 50000$ random sparse matrix with nonzero
probability $0.5\%$ for each entry and $b$ is generated as $b = Ax^0 +
z$. The nonzero entries of $A$ independently follow the Gaussian
distribution with mean 0 and variance $0.04$. The signal $x^0$ is a
vector with 250 nonzeros and $z$ is \iid standard Gaussian noise. The
parameter $\delta$ is set to $\|x^0\|_1$.

\vspace{0.4em}
\noindent{\bf Sorted $\ell_1$ penalized estimation.}~ $f(x) = \half \|Ax - b\|^2 + \sum_{i=1}^p \lambda_i |x|_{(i)}$, where $|x|_{(1)} \ge \cdots \ge |x|_{(p)}$ are the order statistics of $|x|$. This is a recently introduced testing and estimation procedure \citep{slope}. The design $A$ is a $1000 \times 10000$ Gaussian random matrix, and $b$ is generated as $b = Ax^0 +
z$ for 20-sparse $x^0$ and Gaussian noise $z$. The penalty sequence is set to $\lambda_i = 1.1\Phi^{-1}(1-0.05i/(2p))$.

\vspace{0.4em}
\noindent{\bf Lasso.}~ $f(x) = \half\|Ax-b\|^2 + \lambda \|x\|_1$, where $A$ is a $1000\times 500$ random matrix and $b$ is given as $b = Ax^0 +
z$ for 20-sparse $x^0$ and Gaussian noise $z$. We set $\lambda = 1.5\sqrt{2\log p}$.

\vspace{0.4em}
\noindent{\bf $\ell_1$-regularized logistic regression.}~ $f(x) = \sum_{i=1}^n - y_i a_i^T x + \log(1 + \mathrm{e}^{a_i^T x}) + \lambda \|x\|_1$, where the setting is the same as in Figure \ref{fig:larger_r_reglog}. The results are presented in Figure \ref{fig:compare_restarting_reglogistic}.

\vspace{0.4em}
\noindent{\bf Logistic regression with large sparse design.}~ $f(x) = \sum_{i=1}^n - y_i a_i^T x + \log(1 + \mathrm{e}^{a_i^T x})$, in which $A = (a_1, \ldots, a_n)^T$ is a $10^7 \times 20000$ sparse random matrix with nonzero probability $0.1\%$ for each entry, so there are roughly $2 \times 10^8$ nonzero entries in total. To generate the labels $y$, we set $x^0$ to be \iid $\mathcal{N}(0, 1/4)$. The plot is Figure \ref{fig:compare_restarting_logistic}.

In these examples, $k_{\min}$ is set to be 10 and the step sizes are fixed to be $1/L$. If the objective is in composite form,  the Lipschitz bound applies to the smooth part. Figure \ref{fig:compare_restarting} presents the performance of the speed restarting scheme, the gradient restarting scheme, the original Nesterov's scheme and the proximal gradient method. The objective functions include strongly convex, non-strongly convex and non-smooth functions, violating the assumptions in Theorem \ref{exp}. Among all the examples, it is interesting to note that both speed restarting scheme empirically exhibit linear convergence by significantly reducing bumps in the objective values. This leaves us an open problem of whether there exists provable linear convergence rate for the gradient restarting scheme as in Theorem \ref{exp}. It is also worth pointing out that compared with gradient restarting, the speed restarting scheme empirically exhibits more stable linear convergence rate.


\section{Discussion}
This paper introduces a second-order ODE and accompanying tools for characterizing Nesterov's accelerated gradient method. This ODE is applied to study variants of Nesterov's scheme and is capable of interpreting some empirically observed phenomena, such as oscillations along the trajectories. Our approach suggests (1) a large family of generalized Nesterov's schemes that are all guaranteed to converge at the rate $O(1/k^2)$, and (2) a restarting scheme provably achieving a linear convergence rate whenever $f$ is strongly convex.

In this paper, we often utilize ideas from continuous-time ODEs, and then apply these ideas to discrete schemes. The translation, however, involves parameter tuning and tedious calculations. This is the reason why a general theory mapping properties of ODEs into corresponding properties for discrete updates would be a welcome advance. Indeed, this would allow researchers to only study the simpler and more user-friendly ODEs. 

As evidenced by many examples, the viewpoint of regarding the ODE as a surrogate for Nesterov's scheme would allow a new perspective for studying accelerated methods in optimization. The discrete scheme and the ODE are closely connected by the exact mapping between the coefficients of momentum (e.g.~$(k-1)/(k+2)$) and velocity (e.g.~$3/t$). The derivations of generalized Nesterov's schemes and the speed restarting scheme are both motivated by trying a different velocity coefficient, in which the surprising phase transition at 3 is observed. Clearly, such alternatives are endless, and we expect this will lead to findings of many discrete accelerated schemes. In a different direction, a better understanding of the trajectory of the ODEs, such as curvature, has the potential to be helpful in deriving appropriate stopping criteria for termination, and choosing step size by backtracking.


\acks{W.~S.~was partially supported by a General Wang Yaowu Stanford Graduate Fellowship. S.~B.~was partially supported by DARPA XDATA. E.~C.~was partially supported by AFOSR under grant FA9550-09-1-0643, by NSF under grant CCF-0963835, and by the Math + X Award from the Simons Foundation. We would like to thank Carlos Sing-Long, Zhou Fan, and Xi Chen for helpful discussions about parts of this paper. We would also like to thank the associate editor and two reviewers for many constructive comments that improved the presentation of the paper.
}



\appendix
\section*{Appendix A. Proof of Theorem \ref{thm:regularity}}
\label{app:a}
The proof is divided into two parts, namely, existence and uniqueness.
\begin{lemma}\label{lm:exist}
For any $f\in\mathcal{F}_{\infty}$ and any $x_0\in\R^n$, the ODE \eqref{key} has at least one solution $X$ in $C^{2}(0,\infty)\cap C^{1}[0,\infty)$. 
\end{lemma}
Below, some preparatory lemmas are given before turning to the proof of this lemma. To begin with, for any $\delta > 0$ consider the smoothed ODE
\begin{equation}\label{eq:approx}
\dder{X} + \frac{3}{\max(\delta, t)}\der{X} + \nabla f(X) = 0
\end{equation}
with $X(0) = x_0, \der{X(0)} = 0$. Denoting by $Z = \der{X}$, then \eqref{eq:approx} is equivalent to
\begin{equation}\nonumber
\frac{\d}{\d t}
\begin{pmatrix}
X\\
Z
\end{pmatrix}
=
\begin{pmatrix}
Z\\
-\frac{3}{\max(\delta, t)}Z - \nabla f(X) 
\end{pmatrix}
\end{equation}
with $X(0) = x_0, Z(0) = 0$. As functions of $(X, Z)$, both $Z$ and $-3Z/\max(\delta, t) - \nabla f(X))$ are $\max(1, L) + 3/\delta$-Lipschitz continuous. Hence by standard ODE theory, \eqref{eq:approx} has a unique global solution in $C^{2}[0,\infty)$, denoted by $X_{\delta}$. Note that $\dder{X_{\delta}}$ is also well defined at $t=0$. Next, introduce $M_{\delta}(t)$ to be the supremum of $\norm{\der{X}_{\delta}(u)}/u$ over $u\in (0, t]$. It is easy to see that $M_{\delta}(t)$ is finite because $\norm{\der{X}_{\delta}(u)}/u = (\norm{\der{X}_{\delta}(u) - \der{X}_{\delta}(0)})/u = \norm{\dder{X_{\delta}}(0)} + o(1)$ for small $u$. We give an upper bound for $M_{\delta}(t)$ in the following lemma.
\begin{lemma}\label{lm:m-estimate}
For $\delta < \sqrt{6/L}$, we have
\begin{equation}\nonumber
M_{\delta}(\delta) \le \frac{\norm{\nabla f(x_0)}}{1 - L\delta^2/6}.
\end{equation}
\end{lemma}
The proof of Lemma \ref{lm:m-estimate} relies on a simple lemma.
\begin{lemma}\label{lm:lipsch-gra}
For any $u > 0$, the following inequality holds
\begin{equation}\nonumber
\norm{\nabla f(X_{\delta}(u)) - \nabla f(x_0)} \le \half LM_{\delta}(u)u^2.
\end{equation}
\end{lemma}
\begin{proof}
By Lipschitz continuity,
\begin{equation*}
\norm{\nabla f(X_{\delta}(u)) - \nabla f(x_0)} \le L\norm{X_{\delta}(u) - x_0} = \Big{\|}\int_0^u\der{X_{\delta}}(v)\d v\Big{\|}\le \int_0^uv\frac{\norm{\der{X_{\delta}}(v)}}{v}\d v \le \half LM_{\delta}(u)u^2.
\end{equation*}
\end{proof}

Next, we prove Lemma \ref{lm:m-estimate}.
\begin{proof}
For $0 < t \le \delta$, the smoothed ODE takes the form
$$\dder{X_{\delta}} + \frac{3}{\delta}\der{X_{\delta}} + \nabla f(X_{\delta}) = 0,$$
which yields
\begin{equation}\nonumber
\der{X_{\delta}}\e{3t/\delta} = -\int^t_0\nabla f(X_{\delta}(u))\e{3u/\delta}\d u = -\nabla f(x_0)\int^t_0\e{3u/\delta}\d u - \int^t_0(\nabla f(X_{\delta}(u)) - \nabla f(x_0))\e{3u/\delta}\d u.
\end{equation}
Hence, by Lemma \ref{lm:lipsch-gra}
\begin{equation}\nonumber
\begin{aligned}
\frac{\norm{\der{X_{\delta}}(t)}}{t} &\le \inverse{t}\e{-3t/\delta}\norm{\nabla f(x_0)}\int^t_0\e{3u/\delta}\d u + \inverse{t}\e{-3t/\delta}\int^t_0\half LM_{\delta}(u)u^2\e{3u/\delta}\d u\\
&\le \norm{\nabla f(x_0)} + \frac{LM_{\delta}(\delta)\delta^2}{6}.
\end{aligned}  
\end{equation}
Taking the supremum of $\norm{\der{X_{\delta}}(t)}/t$ over $0 < t \le \delta$ and rearranging the inequality give the desired result.
\end{proof}
Next, we give an upper bound for $M_{\delta}(t)$ when $t > \delta$.
\begin{lemma}\label{lm:m-estimate-s}
For $\delta < \sqrt{6/L}$ and $\delta < t < \sqrt{12/L}$, we have
\begin{equation}\nonumber
M_{\delta}(t) \le \frac{(5-L\delta^2/6)\norm{\nabla f(x_0)}}{4(1-L\delta^2/6)(1-Lt^2/12)}.
\end{equation}
\end{lemma}
\begin{proof}
For $t > \delta$, the smoothed ODE takes the form
\[
\dder{X_{\delta}} + \frac{3}{t}\der{X_{\delta}} + \nabla f(X_{\delta}) = 0,
\]
which is equivalent to
$$\frac{\d t^3\der{X_{\delta}}(t)}{\d t} = -t^3\nabla f(X_{\delta}(t)).$$
Hence, by integration, $t^3\der{X_{\delta}}(t)$ is equal to
\begin{multline*}
 -\int^t_{\delta}u^3\nabla f(X_{\delta}(u))\d u + \delta^3\der{X_{\delta}}(\delta) = -\int^t_{\delta}u^3\nabla f(x_0)\d u -\int^t_{\delta}u^3(\nabla f(X_{\delta}(u)) - \nabla f(x_0))\d u + \delta^3\der{X_{\delta}}(\delta).
\end{multline*}
Therefore by Lemmas \ref{lm:lipsch-gra} and \ref{lm:m-estimate}, we get
\begin{align*}
\frac{\norm{\der{X_{\delta}}(t)}}{t} &\le \frac{t^4-\delta^4}{4t^4}\norm{\nabla f(x_0)} + \inverse{t^4}\int^t_{\delta}\half LM_{\delta}(u)u^5\d u + \frac{\delta^4}{t^4}\frac{\norm{\der{X_{\delta}}(\delta)}}{\delta}\\
&\le \inverse{4}\norm{\nabla f(x_0)} + \inverse{12}LM_{\delta}(t)t^2 + \frac{\norm{\nabla f(X_0)}}{1 - L\delta^2/6},
\end{align*}
where the last expression is an increasing function of $t$. So for any $\delta < t' < t$, it follows that
\begin{equation}\nonumber
\frac{\norm{\der{X_{\delta}}(t')}}{t'} \le \inverse{4}\norm{\nabla f(x_0)} + \inverse{12}LM_{\delta}(t)t^2 + \frac{\norm{\nabla f(x_0)}}{1 - L\delta^2/6},
\end{equation}
which also holds for $t'\le\delta$. Taking the supremum over $t'\in (0, t)$ gives
\begin{equation}\nonumber
M_{\delta}(t) \le \inverse{4}\norm{\nabla f(x_0)} + \inverse{12}LM_{\delta}(t)t^2 + \frac{\norm{\nabla f(X_0)}}{1 - L\delta^2/6}.
\end{equation}
The desired result follows from rearranging the inequality.
\end{proof}

\begin{lemma}\label{lm:arzela}
The function class $\mathcal{F} = \{X_{\delta}: \left[ 0, \sqrt{6/L}\right] \rightarrow \R^n \big{|} \delta = \sqrt{3/L}/2^m, m = 0, 1, \ldots\}$ is uniformly bounded and equicontinuous.
\end{lemma}
\begin{proof}
By Lemmas \ref{lm:m-estimate} and \ref{lm:m-estimate-s}, for any $t\in [0, \sqrt{6/L}], \delta\in (0, \sqrt{3/L})$ the gradient is uniformly bounded as
\begin{align*}
\norm{\der{X_{\delta}}(t)} \le \sqrt{6/L}M_{\delta}(\sqrt{6/L}) \le \sqrt{6/L}\max\Big{\{}\frac{\norm{\nabla f(x_0)}}{1 - \half}, \frac{5\norm{\nabla f(x_0)}}{4(1-\half)(1-\half)}\Big{\}} = 5\sqrt{6/L}\norm{\nabla f(x_0)}.
\end{align*}
Thus it immediately implies that $\mathcal{F}$ is equicontinuous. To establish the uniform boundedness, note that
\begin{equation}\nonumber
\norm{X_{\delta}(t)} \le \norm{X_{\delta}(0)} + \int_0^t\norm{\der{X_{\delta}}(u)}\d u \le \norm{x_0} + 30\norm{\nabla f(x_0)}/L.
\end{equation}
\end{proof}

We are now ready for the proof of Lemma \ref{lm:exist}.
\begin{proof}
By the Arzel\'a-Ascoli theorem and Lemma \ref{lm:arzela}, $\mathcal{F}$ contains a subsequence converging uniformly on $[0, \sqrt{6/L}]$. Denote by $\{X_{\delta_{m_i}}\}_{i\in\mathbb{N}}$ the convergent subsequence and $\breve X$ the limit. Above, $\delta_{m_i} = \sqrt{3/L}/2^{m_i}$ decreases as $i$ increases. We will prove that $\breve X$ satisfies \eqref{key} and the initial conditions $\breve X(0) = x_0, \dot{\breve{X}}(0) = 0$.

Fix an arbitrary $t_0\in (0, \sqrt{6/L})$. Since
$\norm{\der{X}_{\delta_{m_i}}(t_0)}$ is bounded, we can pick a
subsequence of $\der{X}_{\delta_{m_i}}(t_0)$ which converges to a
limit, denoted by $X^D_{t_0}$. Without loss of generality, assume the
subsequence is the original sequence. Denote by $\tilde X$ the local
solution to \eqref{key} with $X(t_0) = \breve X(t_0)$ and
$\der{X}(t_0) = X^D_{t_0}$. Now recall that ${X}_{\delta_{m_i}}$ is the
solution to \eqref{key} with $X(t_0) = X_{\delta_{m_i}}(t_0)$ and
$\der{X}(t_0) = \der{X}_{\delta_{m_i}}(t_0)$ when $\delta_{m_i} <
t_0$. Since both $X_{\delta_{m_i}}(t_0)$ and
$\der{X}_{\delta_{m_i}}(t_0)$ approach $\breve X(t_0)$ and
$X^D_{t_0}$, respectively, there exists $\epsilon_0>0$ such that
\begin{equation}\nonumber
\sup_{t_0-\epsilon_0 < t < t_0+\epsilon_0}\norm{X_{\delta_{m_i}}(t) - \tilde X(t)}\rightarrow 0
\end{equation}
as $i\rightarrow\infty$. However, by definition we have  
\begin{equation}\nonumber
\sup_{t_0-\epsilon_0 < t <t_0+\epsilon_0}\norm{X_{\delta_{m_i}}(t) - \breve X(t)}\rightarrow 0.
\end{equation}
Therefore $\breve X$ and $\tilde X$ have to be identical on $(t_0-\epsilon_0, t_0+\epsilon_0)$. So $\breve X$ satisfies \eqref{key} at $t_0$. Since $t_0$ is arbitrary, we conclude that $\breve X$ is a solution to \eqref{key} on $(0, \sqrt{6/L})$. By extension, $\breve X$ can be a global solution to \eqref{key} on $(0, \infty)$. It only leaves to verify the initial conditions to complete the proof.

The first condition $\breve X(0)= x_0$ is a direct consequence of $X_{\delta_{m_i}}(0) = x_0$. To check the second, pick a small $t>0$ and note that
\begin{multline}\nonumber
\frac{\norm{\breve X(t) - \breve X(0)}}{t} = \lim_{i\rightarrow\infty}\frac{\norm{X_{\delta_{m_i}}(t) - X_{\delta_{m_i}}(0)}}{t} = \lim_{i\rightarrow\infty}\norm{\der{X}_{\delta_{m_i}}(\xi_i)}\\
 \le \limsup_{i\rightarrow\infty}tM_{\delta_{m_i}}(t) \le 5t\sqrt{6/L}\norm{\nabla f(x_0)},
\end{multline}
where $\xi_i\in(0, t)$ is given by the mean value theorem. The desired result follows from taking $t\rightarrow 0$.
\end{proof}

Next, we aim to prove the uniqueness of the solution to \eqref{key}.
\begin{lemma}\label{lm:unique}
For any $f \in \mathcal{F}_{\infty}$, the ODE \eqref{key} has at most one local solution in a neighborhood of $t=0$.
\end{lemma}
Suppose on the contrary that there are two solutions, namely, $X$ and $Y$, both defined on $(0, \alpha)$ for some $\alpha > 0$. Define $\tilde M(t)$ to be the supremum of $\norm{\der{X}(u) - \der{Y}(u)}$ over $u\in [0, t)$. To proceed, we need a simple auxiliary lemma.
\begin{lemma}\label{lm:x-y-gra}
For any $t\in (0, \alpha)$, we have
\begin{equation}\nonumber
\norm{\nabla f(X(t)) - \nabla f(Y(t))} \le Lt\tilde M(t).
\end{equation}
\end{lemma}
\begin{proof}
By Lipschitz continuity of the gradient, one has
\begin{multline}\nonumber
\norm{\nabla f(X(t)) - \nabla f(Y(t))} \le L\norm{X(t) - Y(t)} = L\Big{\|}\int_{0}^t\der{X}(u) - \der{Y}(u)\d u + X(0)-Y(0)\Big{\|}\\\le L\int_{0}^t\norm{\der{X}(u) - \der{Y}(u)}\d u\le Lt\tilde M(t).
\end{multline}
\end{proof}

Now we prove Lemma \ref{lm:unique}.
\begin{proof}
Similar to the proof of Lemma \ref{lm:m-estimate-s}, we get
\begin{equation}\nonumber
t^3(\der{X}(t) - \der{Y}(t)) = -\int^t_{0}u^3(\nabla f(X(u)) - \nabla f(Y(u)))\d u.
\end{equation}
Applying Lemma \ref{lm:x-y-gra} gives
\begin{equation}\nonumber
t^3{\norm{\der{X}(t) - \der{Y}(t)}} \le\int^t_0Lu^4\tilde M(u)\d u \le \inverse{5}Lt^5\tilde M(t),
\end{equation}
which can be simplified as ${\norm{\der{X}(t) - \der{Y}(t)}}\le Lt^2\tilde M(t)/5$. Thus, for any $t'\le t$ it is true that ${\norm{\der{X}(t') - \der{Y}(t')}}\le Lt^2\tilde M(t)/5$. Taking the supremum of ${\norm{\der{X}(t') - \der{Y}(t')}}$ over $t'\in (0, t)$ gives $\tilde M(t)\le Lt^2\tilde M(t)/5$. Therefore $\tilde M(t)=0$ for $t < \min(\alpha, \sqrt{5/L})$, which is equivalent to saying $\der{X} = \der{Y}$ on $[0, \min(\alpha, \sqrt{5/L}))$. With the same initial value $X(0) = Y(0) = x_0$ and the same gradient, we conclude that $X$ and $Y$ are identical on $(0, \min(\alpha, \sqrt{5/L}))$, a contradiction.
\end{proof}

Given all of the aforementioned lemmas, the proof of Theorem \ref{thm:regularity} is simply combining \ref{lm:exist} and \ref{lm:unique}.

\section*{Appendix B. Proof of Theorem \ref{thm:nesterov_limit}}\label{app:bb}

Identifying $\sqrt{s} = \Delta t$, the comparison between \eqref{eq:nesterovscheme_one} and \eqref{eq:like_nesterov} reveals that Nesterov's scheme is a discrete scheme for numerically integrating the ODE \eqref{key}. However, its singularity of the damping coefficient  at $t = 0$ leads to the nonexistence of off-the-shelf ODE theory for proving Theorem \ref{thm:nesterov_limit}. To address this difficulty, we use the smoothed ODE \eqref{eq:approx} to approximate the original one; then bound the difference between Nesterov's scheme and the forward Euler scheme of \eqref{eq:approx}, which may take the following form:
\begin{equation}\label{eq:euler_x_z}
\begin{aligned}
X_{k+1}^{\delta} &= X_k^{\delta} + \Delta t Z_k^{\delta}\\
Z_{k+1}^{\delta} & = \Big( 1 - \frac{3\Delta t}{\max\{\delta, k\Delta t\}} \Big) Z_k^{\delta} - \Delta t \nabla f(X_k^{\delta})
\end{aligned}
\end{equation}
with $X_0^{\delta} = x_0$ and $Z_0^{\delta} = 0$.

\begin{lemma}\label{lm:euler_nesterov}
With step size $\Delta t = \sqrt{s}$, for any $T > 0$ we have
\[
\max_{1 \le k \le \frac{T}{\sqrt{s}}}\|X_k^{\delta} - x_k\| \le C\delta^2 + o_s(1)
\]
for some constant $C$.
\end{lemma}

\begin{proof}
Let $z_k = (x_{k+1} - x_k)/\sqrt{s}$. Then Nesterov's scheme is equivalent to
\begin{equation}\label{eq:nesterov_x_z}
\begin{aligned}
x_{k+1} &= x_k + \sqrt{s}z_k\\
z_{k+1} & = \Big( 1 - \frac3{k+3} \Big) z_k - \sqrt{s}\nabla f\Big( x_k + \frac{2k+3}{k+3}\sqrt{s}z_k\Big).
\end{aligned}
\end{equation}
Denote by $a_k = \|X_k^{\delta} - x_k\|, \quad b_k = \|Z_k^{\delta} - z_k\|$, whose initial values are $a_0 = 0$ and $b_0 = \|\nabla f(x_0)\|\sqrt{s}$. The idea of this proof is to bound $a_k$ via simultaneously estimating $a_k$ and $b_k$. By comparing \eqref{eq:euler_x_z} and \eqref{eq:nesterov_x_z}, we get the iterative relationship for $a_k$: $a_{k+1} \le a_k + \sqrt{s}b_k$. Denoting by $S_k = b_0 + b_1 + \cdots + b_k$, this yields
\begin{equation}\label{eq:diff_x_z}
a_k \le \sqrt{s}S_{k-1}.
\end{equation}
Similarly, for sufficiently small $s$ we get
\[
\begin{aligned}
b_{k+1} &\le \Big| 1 - \frac{3}{\max\{\delta/\sqrt{s}, k \}} \Big| b_k + L\sqrt{s}a_k + \Big(\Big|\frac3{k+3} - \frac3{\max\{\delta/\sqrt{s}, k\}} \Big| + 2Ls  \Big)\|z_k\|\\
& \le b_k + L\sqrt{s}a_k + \Big(\Big|\frac3{k+3} - \frac3{\max\{\delta/\sqrt{s}, k\}} \Big| + 2Ls  \Big)\|z_k\|.
\end{aligned}
\]
To upper bound $\|z_k\|$, denoting by $C_1$ the supremum of $\sqrt{2L(f(y_k) - f^\star)}$ over all $k$ and $s$, we have
\[
\|z_k\| \le \frac{k - 1}{k + 2}\|z_{k-1}\|  + \sqrt{s}\|\nabla f(y_k)\| \le \|z_{k-1}\|  + C_1\sqrt{s},
\]
which gives $\|z_k\| \le C_1(k + 1)\sqrt{s}$. Hence,
\begin{equation}\nonumber
\Big(\Big|\frac3{k+3} - \frac3{\max\{\delta/\sqrt{s}, k\}} \Big| + 2Ls
\Big)\|z_k\| \le
\begin{cases}
C_2\sqrt{s}, \quad k \le \frac{\delta}{\sqrt{s}}\\
\frac{C_2\sqrt{s}}{k} < \frac{C_2s}{\delta}, \quad k > \frac{\delta}{\sqrt{s}}.\\
\end{cases}
\end{equation}
Making use of \eqref{eq:diff_x_z} gives 
\begin{equation}\label{eq:bk_iterative}
b_{k +1} \le 
\begin{cases}
b_k + LsS_{k-1} + C_2\sqrt{s}, \quad k \le \delta/\sqrt{s}\\
b_k + LsS_{k-1} + \frac{C_2s}{\delta}, \quad k > \delta/\sqrt{s}.\\
\end{cases}
\end{equation}
By induction on $k$, for $k \le \delta/\sqrt{s}$ it holds that
\[
b_k \le \frac{C_1Ls + C_2 + (C_1+C_2)\sqrt{Ls}}{2\sqrt{L}}(1 + \sqrt{Ls})^{k-1} - \frac{C_1Ls + C_2 - (C_1+C_2)\sqrt{Ls}}{2\sqrt{L}}(1 - \sqrt{Ls})^{k-1}.
\]
Hence,
\[
S_k \le \frac{C_1Ls + C_2 + (C_1+C_2)\sqrt{Ls}}{2L\sqrt{s}}(1 + \sqrt{Ls})^{k} + \frac{C_1Ls + C_2 - (C_1+C_2)\sqrt{Ls}}{2L\sqrt{s}}(1 - \sqrt{Ls})^k - \frac{C_2}{L\sqrt{s}}.
\]
Letting $k^\star = \lfloor \delta/\sqrt{s} \rfloor$, we get
\[
 \limsup_{s \goto 0} \sqrt{s}S_{k^\star-1} \le \frac{C_2\e{\delta\sqrt{L}} + C_2 \e{-\delta\sqrt{L}} - 2C_2}{2L} = O(\delta^2),
\]
which allows us to conclude that
\begin{equation}\label{eq:ak_bound_delta}
a_k \le \sqrt{s}S_{k-1} = O(\delta^2) + o_s(1)
\end{equation}
for all $k \le \delta/\sqrt{s}$.

Next, we bound $b_k$ for $k > k^\star = \lfloor \delta/\sqrt{s} \rfloor$. To this end, we consider the worst case of \eqref{eq:bk_iterative}, that is,
\[
b_{k+1} = b_k + LsS_{k-1} + \frac{C_2s}{\delta}
\]
for $k > k^\star$ and $S_{k^\star} = S_{k^\star+1} = C_3\delta^2/\sqrt{s} + o_s(1/\sqrt{s})$ for some sufficiently large $C_3$. In this case, $C_2s/\delta < sS_{k-1}$ for sufficiently small $s$. Hence, the last display gives
\[
b_{k+1} \le b_k + (L+1)sS_{k-1}.
\]
By induction, we get
\[
S_k \le \frac{C_3\delta^2/\sqrt{s} + o_s(1/\sqrt{s})}{2}\left( (1+\sqrt{(L+1)s})^{k-k^\star} +  (1-\sqrt{(L+1)s})^{k-k^\star} \right).
\]
Letting $k^\diamond = \lfloor T/\sqrt{s} \rfloor$, we further get
\[
\limsup_{s\goto 0} \sqrt{s}S_{k^\diamond} \le \frac{C_3\delta^2(\e{(T-\delta)\sqrt{L+1}} + \e{-(T-\delta)\sqrt{L+1}})}{2} = O(\delta^2),
\]
which yields
\begin{equation}\nonumber
a_k \le \sqrt{s}S_{k-1} = O(\delta^2) + o_s(1)
\end{equation}
for $k^\star < k \le k^\diamond$. Last, combining \eqref{eq:ak_bound_delta} and the last display, we get the desired result.

\end{proof}

Now we turn to the proof of Theorem \ref{thm:nesterov_limit}.
\begin{proof}
Note the triangular inequality
\[
\|x_k - X(k\sqrt{s})\| \le \|x_k - X^\delta_k\| + \|X^\delta_k - X_\delta(k\sqrt{s})\| + \|X_\delta(k\sqrt{s}) - X(k\sqrt{s})\|,
\]
where $X_\delta(\cdot)$ is the solution to the smoothed ODE \eqref{eq:approx}. The proof of Lemma \ref{lm:exist} implies that, we can choose a sequence $\delta_m \goto 0$ such that  
\[
\sup_{0 \le t \le T}\|X_{\delta_m}(t) - X(t)\| \goto 0.
\]
The second term $\|X^{\delta_m}_k - X_{\delta_m}(k\sqrt{s})\|$ will uniformly vanish as $s \goto 0$ and so does the first term $\|x_k - X^{\delta_m}_k\|$ if first $s \goto 0$ and then $\delta_m \goto 0$. This completes the proof.
\end{proof}

\section*{Appendix C. ODE for Composite Optimization}\label{app:c}
In analogy to \eqref{key} for smooth $f$ in Section \ref{sec:derivation}, we develop an ODE for composite optimization,
\begin{equation}\label{eq:composite}
\mbox{minimize} \quad f(x) = g(x) + h(x),
\end{equation}
where $g\in \mathcal{F}_{L}$ and $h$ is a general convex function
possibly taking on the value $+ \infty$. Provided it is easy to
evaluate the proximal of $h$, \citet{BeckTeboulle} propose a proximal
gradient version of Nesterov's scheme for solving
\eqref{eq:composite}. It is to repeat the following recursion for $k
\ge 1$,
\begin{equation}\nonumber
\begin{aligned}
&x_k = y_{k-1} -  s G_t(y_{k-1})\\
&y_k = x_k + \frac{k-1}{k+2}(x_k -x_{k-1}),
\end{aligned}
\end{equation}
where the proximal subgradient $G_s$ has been defined in Section \ref{sec:high-friction}. If the constant step size $s\le 1/L$, it is guaranteed that \citep{BeckTeboulle}
\[
f(x_k) - f^\star \le \frac{2\|x_0-x^\star\|^2}{s(k+1)^2},
\]
which in fact is a special case of Theorem \ref{thm:large_r_dis}.

Compared to the smooth case, it is not as clear to define the driving force as $\nabla f$ in \eqref{key}. At first, it might be a good try to define
\[
G(x) = \underset{s\rightarrow 0}{\lim} G_s(x) = \underset{s\rightarrow 0}{\lim}\frac{x - \text{argmin}_z\left( \norm{z - (x - s\nabla g(x))}^2/(2s) + h(z) \right)}{s},
\]
if it exists. However, as implied in the proof of Theorem \ref{thm:nonsmooth} stated below, this definition fails to capture the \textit{directional} aspect of the subgradient. To this end, we define the subgradients through the following lemma.
\begin{lemma}\citep{rockafellar}\label{rock}
For any convex function $f$ and any $x, p \in \R^n$, the directional derivative 
$\lim_{t\rightarrow 0+}(f(x + sp) - f(x))/s $ exists, and can be evaluated as
$$\lim_{s\rightarrow 0+}\frac{f(x + sp) - f(x)}{s} = \sup_{\xi\in \partial f(x)}\langle \xi, p \rangle.$$
\end{lemma}
Note that the directional derivative is semilinear in $p$ because 
\[
\sup_{\xi\in \partial f(x)}\langle \xi, cp \rangle =  c\sup_{\xi\in \partial f(x)}\langle \xi, p \rangle
\]
for any $c>0$. 

\begin{definition}
A Borel measurable function $G(x, p;f)$ defined on $\mathbb{R}^n\times \mathbb{R}^{n}$ is said to be a directional subgradient of $f$ if 
\begin{align*}
&G(x, p) \in \partial f(x),\\
& \langle G(x, p), p\rangle = \sup_{\xi\in \partial f(x)}\langle \xi, p \rangle
\end{align*}
for all $x, p$.
\end{definition}

Convex functions are naturally locally Lipschitz, so $\partial f(x)$ is compact for any $x$. Consequently there exists $\xi \in \partial f(x)$ which maximizes $\langle \xi, p \rangle$. So Lemma \ref{rock} guarantees the existence of a directional subgradient. The function $G$ is essentially a function defined on $\mathbb{R}^n\times \mathbb{S}^{n-1}$ in that we can define
\[
G(x, p) = G(x, p/\norm{p}),
\]
and $G(x, 0)$ to be any element in $\partial f(x)$. Now we give the main theorem. However, note that we do not guarantee the existence of solution to \eqref{nonsmooth-key}.
\begin{theorem}\label{thm:nonsmooth}
Given a convex function $f(x)$ with directional subgradient $G(x, p;f)$, assume that the second order ODE
\begin{equation}\label{nonsmooth-key}
\ddot{X} + \frac{3}{t}\dot{X} + G(X, \dot X) = 0, ~ X(0) = x_0, \dot X(0) = 0
\end{equation}
admits a solution $X(t)$ on $[0, \alpha)$ for some $\alpha > 0$. Then for any $0 < t < \alpha$, we have
$$f(X(t)) - f^\star \le \frac{2\norm{x_0-x^\star}_2^2}{t^2}.$$
\end{theorem}

\begin{proof}
  It suffices to establish that $\mathcal{E}$, first defined in the
  proof of Theorem \ref{thm:ode_rate}, is monotonically
  decreasing. The difficulty comes from that $\mathcal E$ may not be
  differentiable in this setting. Instead, we study $(\mathcal{E}(t +
  \D t) - \mathcal{E}(t))/\D t$ for small $\D t > 0$. In
  $\mathcal{E}$, the second term $2\norm{X + t\dot{X}/2 - x^\star}^2$
  is differentiable, with derivative $4\langle X + \frac{t}{2}\dot X-
  x^\star, \frac{3}{2}\dot X + \frac{t}{2}\ddot X \rangle$. Hence,
\begin{equation}\label{eq:energy_diff_first}
\begin{aligned}
& 2\norm{X(t + \D t) + \frac{t}{2}\dot{X}(t + \D t) - x^\star}^2 - 2\norm{X(t) + \frac{t}{2}\dot{X}(t) - x^\star}^2\\
& = 4\langle X + \frac{t}{2}\dot X- x^\star, \frac{3}{2}\dot X + \frac{t}{2}\ddot X \rangle\D t + o(\D t)\\
& = -t^2\langle \dot X, G(X, \dot X)\rangle\D t - 2t\langle X -x^\star, G(X, \dot X)\rangle\D t+ o(\D t).
\end{aligned}
\end{equation}
For the first term, note that
\begin{multline}\nonumber
(t + \D t)^2(f(X(t + \D t)) -f^\star) - t^2(f(X(t)) -f^\star) = 2t(f(X(t+\D t)) - f^\star)\D t + \\
t^2 (f(X(t + \D t)) - f(X(t)))+ o(\D t).
\end{multline}
Since $f$ is locally Lipschitz, $o(\D t)$ term does not affect the function in the limit,
\begin{equation}\label{eq:local_lip}
f(X(t+\D t)) = f(X + \D t \dot X + o(\D t)) = f(X + \D t \dot X) + o(\D t).
\end{equation}
By Lemma \ref{rock}, we have the approximation
\begin{equation}\label{eq:approx_grow}
f(X + \D t \dot X) = f(X) +  \langle \dot X,G(X, \dot X)\rangle\D t + o(\D t).
\end{equation}
Combining all of \eqref{eq:energy_diff_first}, \eqref{eq:local_lip} and \eqref{eq:approx_grow}, we obtain
\begin{gather*}
\mathcal{E}(t + \D t) - \mathcal{E}(t) = 2t(f(X(t+\D t)) - f^\star)\D t + t^2\langle\dot X,G(X,\dot X) \rangle\D t -t^2\langle \dot X, G(X, \dot X)\rangle\D t \\
- 2t\langle X -x^\star, G(X, \dot X)\rangle\D t+ o(\D t)\\
=2t(f(X) - f^\star)\D t - 2t\langle X -x^\star, G(X, \dot X)\rangle\D t+ o(\D t)\le o(\D t),
\end{gather*}
where the last inequality follows from the convexity of $f$. Thus, 
$$\limsup_{\D t\rightarrow 0+}\frac{\mathcal{E}(t+\D t)-\mathcal{E}(t)}{\D t} \le 0,$$
which along with the continuity of $\mathcal{E}$, concludes that $\mathcal{E}(t)$ is a non-increasing function of $t$. 

\end{proof}

We give a simple example as follows. Consider the Lasso problem
\[
\mbox{minimize}\quad \frac{1}{2}\norm{y - Ax}^2 + \lambda \norm{x}_1.
\]
Any directional subgradients admits the form $G(x, p) = -A^T(y - Ax) + \lambda\mathop{\mathrm{sgn}}(x, p)$, where 
\[
\mathop{\mathrm{sgn}}(x, p)_i=
\begin{cases}
\mathop{\mathrm{sgn}}(x_i),~ &x_i \ne 0\\
\mathop{\mathrm{sgn}}(p_i), ~ &x_i = 0, p_i \ne 0\\
\in [-1, 1],~ &x_i = 0, p_i = 0.
\end{cases}
\]
To encourage sparsity, for any index $i$ with $x_i = 0, p_i = 0$, we let
\[
G(x, p)_i = \mathop{\mathrm{sgn}}\left( A_i^T(Ax-y) \right) \left( |A_i^T(Ax-y)|-\lambda \right)_+.
\]

\section*{Appendix D. Proof of Theorem \ref{thm:dis_t3}}\label{app:d}
\begin{proof}
Let $g$ be $\mu$--strongly convex and $h$ be convex. For $f = g + h$, we show that \eqref{eq:prox_ineq} can be strengthened to
\begin{equation}\label{eq:str_prox}
f(y-sG_s(y)) \le f(x) + G_s(y)^T(y-x) - \frac{s}{2}\|G_s(y)\|^2 - \frac{\mu}{2}\|y-x\|^2.
\end{equation}
Summing $(4k-3)\times\eqref{eq:str_prox}$ with $x = x_{k-1}, y=y_{k-1}$ and $(4r-6)\times\eqref{eq:str_prox}$ with $x = x^\star, y=y_{k-1}$ yields
\begin{multline}\label{eq:bound_for_g}
(4k+4r-9)f(x_k)\le (4k-3)f(x_{k-1}) + (4r-6)f^\star \\
+ G_s(y_{k-1})^T[(4k+4r-9)y_{k-1} - (4k-3)x_{k-1}-(4r-6)x^\star]\\ - \frac{s(4k+4r-9)}{2}\|G_s(y_{k-1})\|^2 - \frac{\mu(4k-3)}{2}\|y_{k-1}-x_{k-1}\|^2 - \mu(2r-3)\|y_{k-1}-x^\star\|^2\\
\le (4k-3)f(x_{k-1}) + (4r-6)f^\star - \mu(2r-3)\|y_{k-1}-x^\star\|^2\\
+ G_s(y_{k-1})^T\left[ (4k+4r-9)(y_{k-1}-x^\star) 
- (4k-3)(x_{k-1}-x^\star) \right],
\end{multline}
which gives a lower bound on $G_s(y_{k-1})^T \left[ (4k+4r-9)y_{k-1} - (4k-3)x_{k-1}-(4r-6)x^\star \right]$. Denote by $\Delta_k$ the second term of $\tilde{\mathcal{E}}(k)$ in \eqref{eq:dis_t3_energy}, namely, 
\[
\Delta_k \triangleq \frac{k+d}{8}\|(2k+2r-2)(y_k-x^\star) - (2k+1)(x_k-x^\star)\|^2,
\]
where $d := 3r/2-5/2$. Then by \eqref{eq:bound_for_g}, we get
\begin{multline}\nonumber
\Delta_k - \Delta_{k-1} = -\frac{k+d}{8}\Big\langle s(2r+2k-5)G_s(y_{k-1})+\frac{k-2}{k+r-2}(x_{k-1}-x_{k-2}), (4k+4r-9)(y_{k-1}-x^\star) \\
- (4k-3)(x_{k-1}-x^\star) \Big\rangle + \frac{1}{8}\|(2k+2r-4)(y_{k-1}-x^\star) - (2k-1)(x_{k-1}-x^\star)\|^2\\
\le -\frac{s(k+d)(2k+2r-5)}{8} \big[(4k+4r-9)(f(x_k)-f^\star)\\
-(4k-3)(f(x_{k-1})-f^\star) + \mu(2r-3)\|y_{k-1}-x^\star\|^2 \big]\\
-\frac{(k+d)(k-2)}{8(k+r-2)}\left\langle x_{k-1}-x_{k-2}, (4k+4r-9)(y_{k-1}-x^\star)-(4k-3)(x_{k-1}-x^\star) \right\rangle \\+ \frac{1}{8}\|2(k+r-2)(y_{k-1}-x^\star) - (2k-1)(x_{k-1}-x^\star)\|^2.\\
\end{multline}
Hence,
\begin{multline}\label{eq:prox_main_t3}
\Delta_k + \frac{s(k+d)(2k+2r-5)(4k+4r-9)}{8}(f(x_k)-f^\star)\\
\le \Delta_{k-1} + \frac{s(k+d)(2k+2r-5)(4k-3)}{8}(f(x_{k-1})-f^\star)\\
-\frac{s\mu(2r-3)(k+d)(2k+2r-5)}{8}\|y_{k-1}-x^\star\|^2 + \Pi_1 + \Pi_2,
\end{multline}
where
\begin{equation}\nonumber
\Pi_1 \triangleq -\frac{(k+d)(k-2)}{8(k+r-2)}\langle x_{k-1}-x_{k-2}, (4k+4r-9)(y_{k-1}-x^\star)-(4k-3)(x_{k-1}-x^\star)\rangle,
\end{equation}
\begin{equation}\nonumber
\Pi_2 \triangleq \frac{1}{8}\|2(k+r-2)(y_{k-1}-x^\star) - (2k-1)(x_{k-1}-x^\star)\|^2.
\end{equation}
By the iterations defined in \eqref{eq:nesterov_general}, one can show that
\begin{multline}\nonumber
\Pi_1 = -\frac{(2r-3)(k+d)(k-2)}{8(k+r-2)}(\|x_{k-1} - x^\star\|^2 - \|x_{k-2} - x^\star\|^2) \\- \frac{(k-2)^2(4k+4r-9)(k+d) + (2r-3)(k-2)(k+r-2)(k+d)}{8(k+r-2)^2}\|x_{k-1}-x_{k-2}\|^2,
\end{multline}
\begin{multline}\nonumber
\Pi_2 = \frac{(2r-3)^2}{8}\|y_{k-1}-x^\star\|^2 + \frac{(2r-3)(2k-1)(k-2)}{8(k+r-2)}(\|x_{k-1} - x^\star\|^2 - \|x_{k-2} - x^\star\|^2) \\+ \frac{(k-2)^2(2k-1)(2k+4r-7) + (2r-3)(2k-1)(k-2)(k+r-2)}{8(k+r-2)^2}\|x_{k-1}-x_{k-2}\|^2.
\end{multline}
Although this is a little tedious, it is straightforward to check that
$(k-2)^2(4k+4r-9)(k+d) + (2r-3)(k-2)(k+r-2)(k+d) \geq
(k-2)^2(2k-1)(2k+4r-7) + (2r-3)(2k-1)(k-2)(k+r-2)$ for any
$k$. Therefore, $\Pi_1 + \Pi_2$ is bounded as
\begin{equation}\nonumber
\Pi_1 + \Pi_2 \le \frac{(2r-3)^2}{8}\|y_{k-1}-x^\star\|^2 + \frac{(2r-3)(k-d-1)(k-2)}{8(k+r-2)}(\|x_{k-1} - x^\star\|^2 - \|x_{k-2} - x^\star\|^2),
\end{equation}
which, together with the fact that $s\mu(2r-3)(k+d)(2k+2r-5) \geq (2r-3)^2$ when $k\geq \sqrt{(2r-3)/(2s\mu)}$, reduces \eqref{eq:prox_main_t3} to
\begin{multline}\nonumber
\Delta_k + \frac{s(k+d)(2k+2r-5)(4k+4r-9)}{8}(f(x_k)-f^\star)\\
\le \Delta_{k-1} + \frac{s(k+d)(2k+2r-5)(4k-3)}{8}(f(x_{k-1})-f^\star) \\+ \frac{(2r-3)(k-d-1)(k-2)}{8(k+r-2)}(\|x_{k-1} - x^\star\|^2 - \|x_{k-2} - x^\star\|^2).
\end{multline}
This can be further simplified as
\begin{equation}\label{eq:prox_main_t3_2}
\tilde{\mathcal{E}}(k) + A_k(f(x_{k-1})-f^\star) \le \tilde{\mathcal{E}}(k-1) + B_k(\|x_{k-1} - x^\star\|^2 - \|x_{k-2} - x^\star\|^2)
\end{equation}
for $k\geq \sqrt{(2r-3)/(2s\mu)}$, where $A_k = (8r-36)k^2+(20r^2-126r+200)k+12r^3-100r^2+288r-281 > 0$ since $r\geq 9/2$ and $B_k = (2r-3)(k-d-1)(k-2)/(8(k+r-2))$. Denote by $k^\star = \lceil\max\{\sqrt{(2r-3)/(2s\mu)}, 3r/2-3/2\}\rceil \asymp 1/\sqrt{s\mu}$. Then $B_k$ is a positive increasing sequence if $k > k^\star$. Summing \eqref{eq:prox_main_t3_2} from $k$ to $k^\star + 1$, we obtain
\begin{equation}\nonumber
\begin{aligned}
&\mathcal{E}(k) + \sum_{i=k^\star+1}^k A_i(f(x_{i-1}) - f^\star) \le \mathcal{E}(k^\star) + \sum_{i=k^\star+1}^k B_i(\|x_{i-1} - x^\star\|^2 - \|x_{i-2} - x^\star\|^2)\\
& = \mathcal{E}(k^\star) + B_k\|x_{k-1} - x^\star\|^2 - B_{k^\star+1}\|x_{k^\star-1} - x^\star\|^2 + \sum_{i=k^\star+1}^{k-1} (B_{j} - B_{j+1})\|x_{j-1} - x^\star\|^2\\
&\le \mathcal{E}(k^\star) + B_k\|x_{k-1} - x^\star\|^2.
\end{aligned}
\end{equation}
Similarly, as in the proof of Theorem \ref{thm:cube_rate}, we can bound $\mathcal{E}(k^\star)$ via another energy functional defined from Theorem \ref{thm:large_r},
\begin{multline}\label{eq:ener_dis_bound1}
\mathcal{E}(k^\star) \le \frac{s(2k^\star+3r-5)(k^\star+r-2)^2}{2}(f(x_{k^\star}) - f^\star) \\+ \frac{2k^\star+3r-5}{16}\|2(k^\star+r-1)y_{k^\star} - 2k^\star x_{k^\star} - 2(r-1)x^\star -(x_{k^\star} - x^\star)\|^2\\
\le \frac{s(2k^\star+3r-5)(k^\star+r-2)^2}{2}(f(x_{k^\star}) - f^\star) \\
+ \frac{2k^\star+3r-5}{8}\|2(k^\star+r-1)y_{k^\star} - 2k^\star x_{k^\star} - 2(r-1)x^\star\|^2\\ + \frac{2k^\star+3r-5}{8}\|x_{k^\star} - x^\star\|^2 \le \frac{(r-1)^2(2k^\star+3r-5)}{2}\|x_0 - x^\star\|^2 \\+ \frac{(r-1)^2(2k^\star+3r-5)}{8s\mu(k^\star+r-2)^2}\|x_0 - x^\star\|^2 \lesssim \frac{\|x_0-x^\star\|^2}{\sqrt{s\mu}}.
\end{multline}
For the second term, it follows from Theorem \ref{thm:large_r_dis} that
\begin{equation}\label{eq:ener_dis_bound2}
\begin{aligned}
B_k\|x_{k-1} - x^\star\|^2 &\le  \frac{(2r-3)(2k-3r+3)(k-2)}{8\mu(k+r-2)}(f(x_{k-1}) - x^\star)\\
&\le \frac{(2r-3)(2k-3r+3)(k-2)}{8\mu(k+r-2)}\frac{(r-1)^2\|x_0-x^\star\|^2}{2s(k+r-3)^2} \\
&\le \frac{(2r-3)(r-1)^2(2k^\star-3r+3)(k^\star-2)}{16s\mu(k^\star+r-2)(k^\star+r-3)^2} \|x_0-x^\star\|^2 \lesssim  \frac{\|x_0-x^\star\|^2}{\sqrt{s\mu}}.
\end{aligned}
\end{equation}
For $k > k^\star$, \eqref{eq:ener_dis_bound1} together with \eqref{eq:ener_dis_bound2} this gives
\begin{equation}\nonumber
\begin{aligned}
f(x_k)-f^\star &\le \frac{16\mathcal{E}(k)}{s(2k+3r-5)(2k+2r-5)(4k+4r-9)}\\
&\le \frac{16(\mathcal{E}(k^\star) + B_k\|x_{k-1}-x^\star\|^2)}{s(2k+3r-5)(2k+2r-5)(4k+4r-9)} \lesssim \frac{\|x_0 - x^\star\|^2}{s^{\frac{3}{2}}\mu^{\half}k^3}.
\end{aligned}
\end{equation}
To conclusion, note that by Theorem \ref{thm:large_r_dis} the gap
$f(x_k) - f^\star$ for $k \le k^\star$ is bounded by
\begin{equation}\nonumber
\frac{(r-1)^2\|x_0-x^\star\|^2}{2s(k+r-2)^2} = \frac{(r-1)^2\sqrt{s\mu}k^3}{2(k+r-2)^2}\frac{\|x_0 - x^\star\|^2}{s^{\frac{3}{2}}\mu^{\half}k^3} \lesssim \sqrt{s\mu}k^\star\frac{\|x_0 - x^\star\|^2}{s^{\frac{3}{2}}\mu^{\half}k^3} \lesssim \frac{\|x_0 - x^\star\|^2}{s^{\frac{3}{2}}\mu^{\half}k^3}.
\end{equation}

\end{proof}
\section*{Appendix E. Proof of Lemmas in Section \ref{sec:accelerate}}\label{app:e}
First, we prove Lemma \ref{lm:m_upper}.
\begin{proof}
  To begin with, note that the ODE \eqref{key} is equivalent to
  $\d(t^3\dot X(t))/\d t = -t^3\nabla f(X(t))$, which by integration
  leads to
\begin{equation}\label{eq:equiv_ode}
t^3\dot X(t)  = -\frac{t^4}{4}\nabla f(x_0) - \int^t_0 u^3(\nabla f(X(u)) - \nabla f(x_0)) \d u = -\frac{t^4}{4}\nabla f(x_0) - I(t).
\end{equation}
Dividing \eqref{eq:equiv_ode} by $t^4$ and applying the bound on $I(t)$, we obtain
\begin{equation*}
\frac{\norm{\dot X(t)}}{t} \le \frac{\norm{\nabla f(x_0)}}{4} + \frac{\|I(t)\|}{t^4} \le \frac{\norm{\nabla f(x_0)}}{4} + \frac{LM(t)t^2}{12}.
\end{equation*}
Note that the right-hand side of the last display is monotonically increasing in $t$. Hence, by taking the supremum of the left-hand side over $(0, t]$, we get
\begin{equation}\nonumber
M(t) \le \frac{\norm{\nabla f(x_0)}}{4} + \frac{LM(t)t^2}{12},
\end{equation}
which completes the proof by rearrangement.
  
\end{proof}

Next, we prove the lemma used in the proof of Lemma \ref{lm:decayfast}.
\begin{lemma}\label{lm:stopping}
The speed restarting time $T$ satisfies
\[T(x_0, f) \geq \frac{4}{5\sqrt{L}}.\]
\end{lemma}

\begin{proof}
The proof is based on studying $\langle \dot X(t), \ddot X(t) \rangle$. Dividing \eqref{eq:equiv_ode} by $t^3$, we get an expression for $\dot X$,
\begin{equation}\label{eq:xdot}
\dot X(t) = -\frac{t}{4}\nabla f(x_0) - \frac{1}{t^3}\int^t_0 u^3(\nabla f(X(u)) - \nabla f(x_0)) \d u.
\end{equation}
Differentiating the above, we also obtain an expression for $\ddot X$:
\begin{equation}\label{eq:xddot}
\ddot X(t) = -\nabla f(X(t)) + \frac{3}{4}\nabla f(x_0) + \frac{3}{t^4}\int^t_0 u^3(\nabla f(X(u)) - \nabla f(x_0)) \d u.
\end{equation}
Using the two equations we can show that $\d \norm{\dot X}^2/{\d t} =
2\langle \dot X(t), \ddot X(t) \rangle > 0$ for $0< t <
4/(5\sqrt{L})$. Continue by observing that \eqref{eq:xdot} and
\eqref{eq:xddot} yield
\begin{equation}\nonumber
\begin{aligned}
&\langle \dot X(t), \ddot X(t) \rangle = \Big{\langle} -\frac{t}{4}\nabla f(x_0) - \frac{1}{t^3}I(t),~ -\nabla f(X(t)) + \frac{3}{4}\nabla f(x_0) + \frac{3}{t^4}I(t) \Big{\rangle}\\
&\geq \frac{t}{4}\langle \nabla f(x_0), \nabla f(X(t)) \rangle - \frac{3t}{16}\norm{\nabla f(x_0)}^2 - \frac{1}{t^3}\norm{I(t)}\Big{(}\norm{\nabla f(X(t))} + \frac{3}{2}\norm{\nabla f(x_0)}\Big{)} - \frac{3}{t^7}\norm{I(t)}^2\\
&\geq \frac{t}{4}\norm{\nabla f(x_0)}^2 - \frac{t}{4}\norm{\nabla f(x_0)}\norm{\nabla f(X(t)) - \nabla f(x_0)} - \frac{3t}{16}\norm{\nabla f(x_0)}^2 \\
&\quad\quad - \frac{LM(t)t^3}{12}\Big{(}\norm{\nabla f(X(t)) - \nabla f(x_0)} + \frac{5}{2}\norm{\nabla f(x_0)}\Big{)} - \frac{L^2M(t)^2t^{5}}{48}\\
&\geq 
\frac{t}{16}\norm{\nabla f(x_0)}^2 - \frac{LM(t)t^3\|\nabla f(x_0)\|}{8} - \frac{LM(t)t^3}{12}\Big{(}\frac{LM(t)t^2}{2} + \frac{5}{2}\norm{\nabla f(x_0)}\Big{)} - \frac{L^2M(t)^2t^{5}}{48}\\
&= \frac{t}{16}\norm{\nabla f(x_0)}^2 - \frac{LM(t)t^3}{3}\norm{\nabla f(x_0)}   - \frac{L^2M(t)^2t^5}{16}.
\end{aligned}
\end{equation}
To complete the proof, applying Lemma \ref{lm:m_upper}, the last inequality yields
\begin{equation}\nonumber
\langle\dot X(t), \ddot X(t)\rangle
\geq \Big{(}\frac{1}{16} - \frac{Lt^2}{12(1 - Lt^2/12)} - \frac{L^2t^4}{256(1 - Lt^2/12)^2}\Big{)}\norm{\nabla f(x_0)}^2t \geq 0
\end{equation}
for $t < \min\{\sqrt{12/L}, 4/(5\sqrt{L})\}=4/(5\sqrt{L})$, where the
positivity follows from 
$$\frac{1}{16} - \frac{Lt^2}{12(1 - Lt^2/12)} - \frac{L^2t^4}{256(1 - Lt^2/12)^2} > 0,$$
which is valid for $0 < t \le 4/(5\sqrt{L})$.
\end{proof}




\vskip 0.2in
\bibliography{ref}

\begin{thebibliography}{31}
\providecommand{\natexlab}[1]{#1}
\providecommand{\url}[1]{\texttt{#1}}
\expandafter\ifx\csname urlstyle\endcsname\relax
  \providecommand{\doi}[1]{doi: #1}\else
  \providecommand{\doi}{doi: \begingroup \urlstyle{rm}\Url}\fi

\bibitem[Beck(2014)]{beck2014introduction}
A.~Beck.
\newblock \emph{Introduction to Nonlinear Optimization: Theory, Algorithms, and
  Applications with {MATLAB}}.
\newblock SIAM, 2014.

\bibitem[Beck and Teboulle(2009)]{BeckTeboulle}
A.~Beck and M.~Teboulle.
\newblock A fast iterative shrinkage-thresholding algorithm for linear inverse
  problems.
\newblock \emph{SIAM Journal on Imaging Sciences}, 2\penalty0 (1):\penalty0
  183--202, 2009.

\bibitem[Becker et~al.(2011)Becker, Bobin, and Cand{\`e}s]{becker}
S.~Becker, J.~Bobin, and E.~J. Cand{\`e}s.
\newblock {NESTA}: {A} fast and accurate first-order method for sparse
  recovery.
\newblock \emph{SIAM Journal on Imaging Sciences}, 4\penalty0 (1):\penalty0
  1--39, 2011.

\bibitem[Bogdan et~al.(2015)Bogdan, Berg, Sabatti, Su, and Cand\`es]{slope}
M.~Bogdan, E.~v.~d. Berg, C.~Sabatti, W.~Su, and E.~J. Cand\`es.
\newblock {SLOPE}--adaptive variable selection via convex optimization.
\newblock \emph{The Annals of Applied Statistics}, 9\penalty0 (3):\penalty0
  1103--1140, 2015.

\bibitem[Boyd and Vandenberghe(2004)]{boydconvex}
S.~Boyd and L.~Vandenberghe.
\newblock \emph{Convex Optimization}.
\newblock Cambridge University Press, 2004.

\bibitem[Boyd et~al.(2011)Boyd, Parikh, Chu, Peleato, and Eckstein]{ADMM}
S.~Boyd, N.~Parikh, E.~Chu, B.~Peleato, and J.~Eckstein.
\newblock Distributed optimization and statistical learning via the alternating
  direction method of multipliers.
\newblock \emph{Foundations and Trends in Machine Learning}, 3\penalty0
  (1):\penalty0 1--122, 2011.

\bibitem[D{\"u}rr and Ebenbauer(2012)]{durr2012a}
H.-B. D{\"u}rr and C.~Ebenbauer.
\newblock On a class of smooth optimization algorithms with applications in
  control.
\newblock \emph{Nonlinear Model Predictive Control}, 4\penalty0 (1):\penalty0
  291--298, 2012.

\bibitem[D{\"u}rr et~al.(2012)D{\"u}rr, Saka, and Ebenbauer]{durr2012b}
H.-B. D{\"u}rr, E.~Saka, and C.~Ebenbauer.
\newblock A smooth vector field for quadratic programming.
\newblock In \emph{51st IEEE Conference on Decision and Control}, pages
  2515--2520, 2012.

\bibitem[Fiori(2005)]{fiori2005}
S.~Fiori.
\newblock Quasi-geodesic neural learning algorithms over the orthogonal group:
  A tutorial.
\newblock \emph{Journal of Machine Learning Research}, 6:\penalty0 743--781,
  2005.

\bibitem[Helmke and Moore(1996)]{dynamic}
U.~Helmke and J.~Moore.
\newblock Optimization and dynamical systems.
\newblock \emph{Proceedings of the IEEE}, 84\penalty0 (6):\penalty0 907, 1996.

\bibitem[Hinton(2005)]{sturm}
D.~Hinton.
\newblock Sturm's 1836 oscillation results evolution of the theory.
\newblock In \emph{Sturm-{L}iouville theory}, pages 1--27. Birkh\"auser, Basel,
  2005.

\bibitem[Leader(2004)]{leadernumerical}
J.~J. Leader.
\newblock \emph{Numerical Analysis and Scientific Computation}.
\newblock Pearson Addison Wesley, 2004.

\bibitem[Lessard et~al.(2014)Lessard, Recht, and Packard]{recht}
L.~Lessard, B.~Recht, and A.~Packard.
\newblock Analysis and design of optimization algorithms via integral quadratic
  constraints.
\newblock \emph{arXiv preprint arXiv:1408.3595}, 2014.

\bibitem[Monteiro et~al.(2012)Monteiro, Ortiz, and Svaiter]{restarting_gatech}
R.~Monteiro, C.~Ortiz, and B.~Svaiter.
\newblock An adaptive accelerated first-order method for convex optimization.
\newblock Technical report, ISyE, Gatech, 2012.

\bibitem[Nesterov(1983)]{nesterov}
Y.~Nesterov.
\newblock A method of solving a convex programming problem with convergence
  rate ${O}(1/k^2)$.
\newblock \emph{Soviet Mathematics Doklady}, 27\penalty0 (2):\penalty0
  372--376, 1983.

\bibitem[Nesterov(2004)]{nesterov-book}
Y.~Nesterov.
\newblock \emph{Introductory Lectures on Convex Pptimization: {A} Basic
  Course}, volume~87 of \emph{Applied Optimization}.
\newblock Kluwer Academic Publishers, Boston, MA, 2004.

\bibitem[Nesterov(2005)]{nesterovsmooth}
Y.~Nesterov.
\newblock Smooth minimization of non-smooth functions.
\newblock \emph{Mathematical Programming}, 103\penalty0 (1):\penalty0 127--152,
  2005.

\bibitem[Nesterov(2013)]{nesterov_compo}
Y.~Nesterov.
\newblock Gradient methods for minimizing composite functions.
\newblock \emph{Mathematical Programming}, 140\penalty0 (1):\penalty0 125--161,
  2013.

\bibitem[Nocedal and Wright(2006)]{nocedal2006numerical}
J.~Nocedal and S.~Wright.
\newblock \emph{Numerical Optimization}.
\newblock Springer Science \& Business Media, 2006.

\bibitem[O'Donoghue and Cand{\`e}s(2013)]{restart}
B.~O'Donoghue and E.~J. Cand{\`e}s.
\newblock Adaptive restart for accelerated gradient schemes.
\newblock \emph{Found. Comput. Math.}, 2013.

\bibitem[Osher et~al.(2014)Osher, Ruan, Xiong, Yao, and Yin]{osher2014sparse}
S.~Osher, F.~Ruan, J.~Xiong, Y.~Yao, and W.~Yin.
\newblock Sparse recovery via differential inclusions.
\newblock \emph{arXiv preprint arXiv:1406.7728}, 2014.

\bibitem[Polyak(1987)]{polyak1987}
B.~T. Polyak.
\newblock \emph{Introduction to optimization}.
\newblock Optimization Software New York, 1987.

\bibitem[Qin and Goldfarb(2012)]{qin2012}
Z.~Qin and D.~Goldfarb.
\newblock Structured sparsity via alternating direction methods.
\newblock \emph{Journal of Machine Learning Research}, 13\penalty0
  (1):\penalty0 1435--1468, 2012.

\bibitem[Rockafellar(1997)]{rockafellar}
R.~T. Rockafellar.
\newblock \emph{Convex Analysis}.
\newblock Princeton Landmarks in Mathematics. Princeton University Press, 1997.
\newblock Reprint of the 1970 original.

\bibitem[Ruszczy{\'n}ski(2006)]{ruszczynski2006}
A.~P. Ruszczy{\'n}ski.
\newblock \emph{Nonlinear Optimization}.
\newblock Princeton University Press, 2006.

\bibitem[Schropp and Singer(2000)]{schropp}
J.~Schropp and I.~Singer.
\newblock A dynamical systems approach to constrained minimization.
\newblock \emph{Numerical functional analysis and optimization}, 21\penalty0
  (3-4):\penalty0 537--551, 2000.

\bibitem[Shor(2012)]{shor2012}
N.~Z. Shor.
\newblock \emph{Minimization Methods for Non-Differentiable Functions}.
\newblock Springer Science \& Business Media, 2012.

\bibitem[Sutskever et~al.(2013)Sutskever, Martens, Dahl, and
  Hinton]{sutskever2013}
I.~Sutskever, J.~Martens, G.~Dahl, and G.~Hinton.
\newblock On the importance of initialization and momentum in deep learning.
\newblock In \emph{Proceedings of the 30th International Conference on Machine
  Learning}, pages 1139--1147, 2013.

\bibitem[Tseng(2008)]{tseng}
P.~Tseng.
\newblock On accelerated proximal gradient methods for convex-concave
  optimization.
\newblock \url{http://pages.cs.wisc.edu/~brecht/cs726docs/Tseng.APG.pdf}, 2008.

\bibitem[Tseng(2010)]{tsengapprox}
P.~Tseng.
\newblock Approximation accuracy, gradient methods, and error bound for
  structured convex optimization.
\newblock \emph{Mathematical Programming}, 125\penalty0 (2):\penalty0 263--295,
  2010.

\bibitem[Watson(1995)]{besselbook}
G.~N. Watson.
\newblock \emph{A Treatise on the Theory of Bessel Functions}.
\newblock Cambridge Mathematical Library. Cambridge University Press, 1995.
\newblock Reprint of the second (1944) edition.

\end{thebibliography}

\end{document}